\documentclass[twoside,11pt]{article}

\usepackage{blindtext}

\usepackage{jmlr2e}
\usepackage{mathrsfs}
\usepackage{amsmath}
\usepackage{amssymb}
\usepackage{mathtools}
\usepackage{tikz}
\usepackage{graphicx}
\usetikzlibrary{calc}           

\DeclareMathOperator*{\argmin}{\arg\!\min}

\DeclareMathOperator{\Tr}{Tr}

\usepackage[capitalize,noabbrev]{cleveref}

\usepackage{lastpage}

\ShortHeadings{Learning High-dimensional Simplices with Noise}{Learning High-dimensional Simplices with Noise}
\firstpageno{1}

\allowdisplaybreaks

\begin{document}

\title{Fundamental Limits of Learning\\
High-dimensional Simplices in Noisy Regimes}

\author{\name Seyed Amir Hossein Saberi \email saberi.sah@ee.sharif.edu \\
       \addr Department of Electerical Engineering\\
       Sharif university of Technology,
       Tehran, Iran
       \AND
       \name Amir Najafi \email amir.najafi@sharif.edu \\
       \addr Department of Computer Engineering\\
       Sharif university of Technology,
       Tehran, Iran
       \AND
       \name Abolfazl Motahari \email motahari@sharif.edu \\
       \addr Department of Computer Engineering\\
       Sharif university of Technology,
       Tehran, Iran
       \AND
       \name Babak khalaj \email khalaj@sharif.edu \\
       \addr Department of Electerical Engineering\\
       Sharif university of Technology,
       Tehran, Iran}

\editor{TBD}

\maketitle

\begin{abstract}
In this paper, we establish sample complexity bounds for learning high-dimensional simplices in $\mathbb{R}^K$ from noisy data. Specifically, we consider $n$ i.i.d. samples uniformly drawn from an unknown simplex in $\mathbb{R}^K$, each corrupted by additive Gaussian noise of unknown variance. We prove an algorithm exists that, with high probability, outputs a simplex within $\ell_2$ or total variation (TV) distance at most $\varepsilon$ from the true simplex, provided $n \ge (K^2/\varepsilon^2) e^{\mathcal{O}(K/\mathrm{SNR}^2)}$, where $\mathrm{SNR}$ is the signal-to-noise ratio. Extending our prior work~\citep{saberi2023sample}, we derive new information-theoretic lower bounds, showing that simplex estimation within TV distance $\varepsilon$ requires at least $n \ge \Omega(K^3 \sigma^2/\varepsilon^2 + K/\varepsilon)$ samples, where $\sigma^2$ denotes the noise variance. In the noiseless scenario, our lower bound $n \ge \Omega(K/\varepsilon)$ matches known upper bounds up to constant factors. We resolve an open question by demonstrating that when $\mathrm{SNR} \ge \Omega(K^{1/2})$, noisy-case complexity aligns with the noiseless case. Our analysis leverages sample compression techniques~\citep{ashtiani2018nearly} and introduces a novel Fourier-based method for recovering distributions from noisy observations, potentially applicable beyond simplex learning.
\end{abstract}

\begin{keywords}
  Sample Complexity Bound, Noisy Regime, High-dimensional Simplices, Sample Compression, Total Variation Distance
\end{keywords}
\section{Introduction}
\label{sec:intro}

Many practical problems in machine learning and data science can be naturally formulated as the task of learning high-dimensional geometric structures from unlabeled data. In particular, the problem of learning simplices from randomly scattered points arises in various domains, including bioinformatics and remote sensing~\citep{chan2009convex, schwartz2010applying, satas2017tumor}. A simplex is defined as the set of all convex combinations of $K+1$ points in $\mathbb{R}^K$, for any $K \in \mathbb{N}$. Formally, the problem we consider is as follows: assume we are given $n$ i.i.d. samples drawn from the uniform distribution over an unknown simplex in $\mathbb{R}^K$. Each sample is further corrupted by additive multivariate Gaussian noise with covariance $\sigma^2 \mathbf{I}_{K \times K}$, where $\sigma > 0$ is unknown. The central question we aim to address is how large must $n$ be, as a function of $K$, $\sigma$ (or alternatively, the signal-to-noise ratio, SNR), and other parameters, to ensure consistent approximation of the underlying simplex with arbitrarily high probability?

Learning simplices has been extensively studied, with numerous heuristic algorithms proposed for various settings, including noisy and sparse regimes. In this paper, however, we focus on the \emph{theoretical} aspects of the problem, which still involve several fundamental open questions. While at least one efficient (i.e., polynomial-time) algorithm with theoretical guarantees exists (see Section \ref{sec:intro:LR}), it still requires prohibitively large sample sizes—it demands $n \geq \Omega(K^{22})$ samples, rendering it impractical in high dimensions. In parallel, researchers have investigated \emph{information-theoretic} limits on sample complexity, independent of computational constraints. Recent work has provided nearly optimal information-theoretic bounds for the \emph{noiseless} case, using the Maximum Likelihood Estimator (MLE), which returns the minimum-volume simplex that contains all the observed samples. Notably, the MLE requires solving an exponentially hard problem in general \footnote{Designing a polynomial-time algorithm with provable guarantees and reasonable dependence on $K$ still remains an open problem.} \citet{najafi2021statistical} showed that the MLE acts as a PAC-learning algorithm for estimating a $K$-simplex up to a vanishing total variation (TV) distance from the true one. Specifically, they proved that at most $\widetilde{\mathcal{O}}(K^2 / \varepsilon)$ samples (ignoring logarithmic factors) suffice to recover the true simplex within TV distance $\varepsilon > 0$. However, the noiseless case is largely unrealistic—clean data is rarely available in practical settings. Moreover, the minimum-volume simplex estimator proposed by \citet{najafi2021statistical} can become highly inaccurate in the presence of noise. In high dimensions (i.e., when $K \gg 1$), the corrupted samples are likely to fall outside the true simplex, leading to significant estimation errors.


In this work, we aim to determine the information-theoretic sample complexity of learning simplices in noisy regimes. Mathematically, we assume that the observed samples are generated according to the following model:
\begin{equation}
\label{eq:equation1:defOfY_i}
\boldsymbol{y}_i = \boldsymbol{V} \boldsymbol{\phi}_i + \boldsymbol{z}_i \quad,\quad i = 1, \ldots, n,
\end{equation}
where $\boldsymbol{V} \in \mathbb{R}^{K \times (K+1)}$ contains the vertices of the true simplex as its columns. Each coefficient vector $\boldsymbol{\phi}_i \in \mathbb{R}^K$ is drawn from a uniform Dirichlet distribution, and the noise vectors $\boldsymbol{z}_i \in \mathbb{R}^K$ are sampled independently from a multivariate Gaussian distribution $\mathcal{N}(0, \sigma^2 \mathbf{I})$.

\textbf{Our Results}: We address the task of estimating the true simplex using an exponential-time estimator $\widehat{\mathcal{S}} = \widehat{\mathcal{S}}(\boldsymbol{y}_{1:n})$. Specifically, we prove that it suffices to have at most $\widetilde{\mathcal{O}}\left(K^2/\varepsilon^2\right) \cdot e^{\mathcal{O}\left(K/\mathrm{SNR}^2\right)}$ noisy samples in order to estimate a $K$-dimensional simplex up to $\ell_2$ or TV distance of $\varepsilon>0$. Here, $\mathrm{SNR}$ denotes the ratio of the standard deviation of the uniform distribution over $\mathcal{S}_T$ to the standard deviation of the additive noise. 

Furthermore, extending our earlier work in \citet{saberi2023sample}, we derive new information-theoretic lower bounds. We show that any algorithm estimating the underlying simplex within total variation distance $\varepsilon$ from noisy samples must use at least $\Omega\left(K^3 \sigma^2 / \varepsilon^2\right)$ samples. In the noiseless case, we establish a lower bound of $n \ge \Omega(K/\varepsilon)$ samples, which captures the correct asymptotic scaling up to constant factors and provides deeper insight into the fundamental statistical limits of simplex learning. These results resolve a previously open problem in the literature and also shed light on an intriguing empirical observation: most heuristic algorithms perform nearly as well as in the noiseless case when $\mathrm{SNR}$ exceeds a certain sublinear function of $K$, but exhibit a sharp deterioration in performance once $\mathrm{SNR}$ falls below this threshold \citep{najafi2021statistical}.

\subsection{Related Work}
\label{sec:intro:LR}

We categorize the existing literature into two main groups based on their focus: (i) works emphasizing the computational efficiency of proposed algorithms, and (ii) those investigating the fundamental sample complexity of the problem, independent of computational constraints. The former primarily aim to develop heuristic methods with good empirical performance, while the latter seek to characterize the inherent, information-theoretic limits of the problem. In the following, we briefly review representative works from both categories.

\citet{anderson2013efficient} showed that with $\mathcal{O}\left(K^{22}\right)$ noiseless samples, one can recover the true simplex using a polynomial-time algorithm. Their approach leverages third-order moment tensors and local search techniques inspired by Independent Component Analysis (ICA). To the best of our knowledge, designing polynomial-time algorithms with improved dependence on $K$ remains an open problem. In parallel, \citet{najafi2021statistical} analyzed the Maximum Likelihood Estimator (MLE) in the noiseless case, proving a sample complexity lower bound of $\widetilde{\mathcal{O}}\left(K^2/\varepsilon\right)$, where $\varepsilon$ denotes the target total variation (TV) distance from the true distribution. Since the MLE requires solving a minimum-volume enclosing simplex problem, which is computationally intractable in high dimensions, they also proposed a practical heuristic surrogate—albeit without strong theoretical guarantees.

In the noisy setting, theoretical investigations are relatively sparse. A notable exception is the work of \citet{bhattacharyya2020near}, which builds upon the sample compression technique introduced by \citet{ashtiani2018nearly} for learning high-dimensional Gaussian mixture models. \citet{bhattacharyya2020near} proved that a noisy simplex can, in principle, be recovered using $\Omega(K^2)$ samples, under the assumption that at least one sample lies near each vertex of the simplex. However, this requirement is extremely strong: a simple probabilistic argument reveals that to ensure such vertex coverage, one would need approximately $\widetilde{\Omega}\left(\varepsilon^{-K}\right)$ samples, making the assumption impractical in high dimensions.

There is also a parallel line of research whose primary goal is not density estimation but rather identifying simplex vertices, a problem known as \emph{vertex hunting}. This task arises frequently in areas like hyperspectral unmixing, archetypal analysis, network membership inference, and topic modeling~\citep{bioucas2012hyperspectral,cutler1994archetypal, rubin2022statistical, ke2024using}. A widely-used method for vertex hunting is the Successive Projection Algorithm (SPA), originally proposed by~\citet{araujo2001successive}, which iteratively finds simplex vertices by projecting data onto orthogonal complements of previously identified vertices. \citet{gillis2013fast} significantly advanced SPA theory, deriving non-asymptotic bounds on the vertex estimation error that depend explicitly on the inverse squared smallest singular value of the vertex matrix. However, standard SPA can perform poorly under heavy noise and outliers, motivating recent developments like the pseudo-point SPA (pp-SPA) by~\citet{jin2024improved}. They introduced a pre-processing step involving hyperplane projections and neighborhood-based denoising before applying SPA, leading to improved robustness and tighter theoretical error bounds based on geometric analysis and extreme value theory, consequently outperforming earlier SPA variants such as robust and smoothed SPA~\citep{bhattacharyya2020near,gillis2019successive}.

From a more applied perspective, numerous heuristic algorithms have been developed for real-world problems in domains such as bioinformatics and hyperspectral imaging \citep{piper2004object, bioucas2012hyperspectral, lin2013endmember}. In hyperspectral imaging, the goal is to estimate the distribution of constituent materials in a region based on remotely sensed hyperspectral images. Each pixel in such images is modeled as a random convex combination of a fixed set of spectral signatures corresponding to pure materials. Hence, the task naturally reduces to estimating an unknown simplex from samples presumed to be uniformly distributed \citep{ambikapathi2011chance, agathos2014robust, zhang2017robust}. In bioinformatics, simplex learning often arises in the context of analyzing complex tissues, which are mixtures of multiple cell types—collections of cells with similar functions, such as blood, brain, or tumor cells \citep{tolliver2010robust, zuckerman2013self}. Bulk measurements from such tissues, like gene expression profiles, can be modeled as convex combinations of the constituent cell types. Accordingly, many methods in this area aim to infer tissue composition by estimating the underlying high-dimensional simplex \citep{shoval2012evolutionary, korem2015geometry}.


The rest of the paper is organized as follows: In Section \ref{sec:notation}, we present the preliminaries and formally define the problem. Our main theoretical results as well as the algorithm that achieves our bounds are discussed in Section \ref{sec:main}. Section \ref{sec:lower-bound} presents our lower-bound techniques. Finally, Section \ref{sec:conc} concludes the paper.


\section{Preliminaries and Summary of Method}
\label{sec:notation}

We use the same notations as \citet{najafi2021statistical}. Throughout the paper, we use light letters to show scalars, bold lowercase letters to show vectors and bold uppercase letters to show matrices. A $K$-simplex $\mathcal{S}$ is defined as the set of all convex combination of $K+1$ 
points in $\mathbb{R}^{K}$. Let $\boldsymbol{V}= \left[\boldsymbol{v}_0\vert\boldsymbol{v}_1\vert\cdots\vert \boldsymbol{v}_{k}\right]\in\mathbb{R}^{K\times(K+1)}$ be a matrix whose columns represent vertices of the simplex, then
\begin{equation*}
\mathcal{S} = 
\mathcal{S}\left(\boldsymbol{v}_0,\ldots,\boldsymbol{v}_k\right)
\triangleq
\left\{ \boldsymbol{V}\boldsymbol{\phi} \bigg\vert~  \boldsymbol{\phi} \in \mathbb{R}^{K+1} ,~ \boldsymbol{\phi} \succeq \boldsymbol{0},~ \boldsymbol{\phi}^{T}\boldsymbol{1}  = 1  \right\}.
\end{equation*}
Also, let $\mathbb{S}_{K}$ denote the set of all possible $K$-simplices in $\mathbb{R}^K$. We denote the uniform probability measure over a simplex $\mathcal{S}$ by $\mathbb{P}_{\mathcal{S}}$, and its probability density function by $f_\mathcal{S}\left(x\right)$: 
\begin{equation*}
f_\mathcal{S}\left(x\right) = \frac{\boldsymbol{1}\left(\boldsymbol{x} \in \mathcal{S}\right)}{\mathrm{Vol}\left(\mathcal{S}\right)},
\end{equation*}
where $\mathrm{Vol}\left(\mathcal{S}\right)$ denotes the Lebesgue measure (or volume) of $\mathcal{S}$. 

The noisy simplex family, i.e., the class of distributions formed by {\it {convolving}} uniform probability density function over simplices in $\mathbb{S}_K$ by $G_{\sigma}\triangleq\mathcal{N}\left(\boldsymbol{0},\sigma^2\boldsymbol{I}\right)$, is denoted by $\mathbb{G}_{K,\sigma}$. Mathematically speaking,
\begin{equation}
\mathbb{G}_{K,\sigma}
\triangleq
\left\{
f_{\mathcal{S}}*G_{\sigma}
\vert~
\mathcal{S}\in\mathbb{S}_K
\right\},
\end{equation}
where $*$ denotes the convolution operator. Obviously, the distribution underlying the input data points $\boldsymbol{y}_i$ in \eqref{eq:equation1:defOfY_i} belongs to $\mathbb{G}_{K,\sigma}$. With a little abuse of notation, we use the term ``class of simplices" to refer to both $\mathbb{S}_K$ and $\left\{f_{\mathcal{S}}\vert~\mathcal{S}\in\mathbb{S}_K\right\}$ whenever the difference is clear from the context. In a similar fashion, we refer to $\mathbb{G}_{K,\sigma}$ as the class of ``noisy simplices".

In order to measure the difference between two distributions, we use both $\ell_2$ and total variation distance. Consider two probability measures $\mathbb{P}_1$ and $\mathbb{P}_2$, with respective density functions $f_1$ and $f_2$, which are defined over $\mathbb{R}^K$. Then, TV distance between $\mathbb{P}_1$ and $\mathbb{P}_2$ can be defined as
\begin{align*}
\operatorname{TV}(\mathbb{P}_1, \mathbb{P}_2) \triangleq & \sup_{\mathrm{A} \in \mathcal{B}}{|\mathbb{P}_1\left(\mathrm{A}\right) - \mathbb{P}_2\left(\mathrm{A}\right)|} 
=  \frac{1}{2}\|f_1 - f_2\| _1,
\end{align*}
where $\mathcal{B}$ is the the Borel $\sigma$-algebra in $\mathbb{R}^K$.

\begin{definition}[PAC-Learnability in realizable setting]
\label{definition:PAC}
A class of distributions $\mathcal{F}$ is PAC learnable in realizable setting, if there exists a learning method which for any distribution $g \in \mathcal{F}$ and any $\epsilon,\delta>0$, outputs an estimator $\widehat{g}$ using $n \geq \mathrm{poly}\left(1/\epsilon, 1/\delta \right)$ i.i.d. samples from $g$, such that with probability at least $1- \delta$ satisfies $\|\hat{g} - g\|_{\mathrm{TV}} \leq \epsilon$.
\end{definition}
We also need to define a series of geometric restrictions for the simplex in noisy cases. In fact, those simplices that are significantly stretched toward on particular direction can be shown to be more prune to noise than those with some minimum levels of geometric regularity. We discuss the necessity of such definitions in later stages. Similar to \citet{najafi2021statistical}, let us denote the $(K-1)$-dimensional volume of the largest \emph{facet} of a $K$-simplex by $\mathcal{A}_{\max}\left(\mathcal{S}\right)$, and the length of the largest line segment inside the simplex (its diameter) by $\mathcal{L}_{\max}\left(\mathcal{S}\right)$. In this regard, we define the \emph{isoperimetricity} of a $K$-simplex as follow:
\begin{definition}[$\left(\underline{\theta},\bar{\theta}\right)$-isoperimetricity of simplices]
\label{def:isoperimetric}
A $K$-simplex $\mathcal{S}\in\mathbb{S}_K$ is defined to be $\left(\underline{\theta},\bar{\theta}\right)$-isoperimetric if the following inequalities hold:
\begin{align*}
\mathcal{A}_{\max}\left(\mathcal{S}\right)
~\leq~
\bar{\theta} \mathrm{Vol}\left(\mathcal{S}\right)^{\frac{K-1}{K}},
\quad
\mathcal{L}_{\max}\left(\mathcal{S}\right)
~\leq~
\underline{\theta}K \mathrm{Vol}\left(\mathcal{S}\right)^{\frac{1}{K}}.
\end{align*}
\end{definition}
The overall concept of isoperimetricity in Definition \ref{def:isoperimetric} reflects the fact that for a simplex to be (even partially) recoverable from noisy data, it should not be stretched too much in any direction or having highly acute angles. In other words, unlike the noiseless case, sample complexity in noisy regimes is also affected by the geometric shape of the underlying simplex (see Theorem \ref{thm:mainNoisySimplexBound} in Section \ref{sec:lower-bound}). for the sake of simplicity in notation, for any fixed $\bar{\theta},\underline{\theta}>0$, whenever we say the class of simplices we actually mean the class of $\left(\underline{\theta},\bar{\theta}\right)$-isoperimetric simplices.

\begin{definition}[$\epsilon$-representative set]
For any $\epsilon>0$, we say that a finite set of distributions $\mathcal{G}$ is an $\epsilon$-representative set for a distribution class $\mathcal{F}$, if for any distribution $f\in \mathcal{F}$, there exists at least one $g \in \mathcal{G}$ that satisfies
$\|f-g\|_{\mathrm{TV}} \leq \epsilon$.
\end{definition}

\subsection{Formal Problem Definition}
\label{sec:formalProblemDef}

Let $\mathcal{S}_T$ denote the unknown $\left(\underline{\theta}, \bar{\theta}\right)$-isoperimetric simplex, referred to as the true simplex. Suppose we observe i.i.d. noisy samples $\boldsymbol{y}_1, \ldots, \boldsymbol{y}_n$ generated from $\mathcal{S}_T$ according to the model described in \eqref{eq:equation1:defOfY_i} with an unknown $\sigma$. Our objective is to construct an estimator $\widehat{\mathcal{S}} \triangleq \widehat{\mathcal{S}}(\boldsymbol{y}_1, \ldots, \boldsymbol{y}_n)$ that approximates $\mathcal{S}_T$ within $\varepsilon$ accuracy in either $\ell_2$ or total variation (TV) distance, with confidence at least $1 - \delta$. That is, for any $\varepsilon, \delta > 0$, the estimator should satisfy this guarantee provided that
$n \geq \mathrm{poly}\left(\frac{1}{\varepsilon}, \log\frac{1}{\delta}\right)$. The sample complexity can also depend on $\bar{\theta}$ and $\underline{\theta}$.

\subsection{Summary of Our Method}

\textbf{Error Upper Bound.} Our error upper bound analysis proceeds in two main stages. In the first stage, we establish the PAC-learnability of the family of noisy simplices $\mathbb{G}_{K,\sigma}$. Specifically, we construct an algorithm which, given any $f_{\mathcal{S}} * G_{\sigma} \in \mathbb{G}_{K,\sigma}$, accuracy parameters $\varepsilon,\delta > 0$, and a set $\mathcal{D} = \{\boldsymbol{y}_1, \ldots, \boldsymbol{y}_n\}$ of i.i.d. samples from $f_{\mathcal{S}} * G_{\sigma}$, outputs a noisy simplex-based density $f_{\widehat{\mathcal{S}}} * G_{\sigma}$ such that
\begin{equation}
\mathbb{P}\left(
\left\|
\left(f_{\widehat{\mathcal{S}}} - f_{\mathcal{S}}\right)*G_{\sigma}
\right\|_{\mathrm{TV}} 
\geq \varepsilon\right) \leq \delta,
\end{equation}
provided $n \geq \mathrm{poly}(1/\varepsilon, 1/\delta)$.

To achieve this, we first identify a high-probability region (a ball in $\mathbb{R}^K$) that contains the true simplex $\mathcal{S}$. We then construct a finite $\varepsilon$-cover of this ball; each subset of $K+1$ cover points defines a candidate simplex, inducing a corresponding density. We show that, with a sufficiently fine covering, at least one candidate density lies within $\varepsilon$ total variation distance of the true distribution. Next, we convolve each candidate density with the Gaussian kernel $G_\sigma$, producing a set of noise-corrupted densities. Using tools from \emph{sample compression} technique of \citet{ashtiani2018nearly}, we prove that, with high probability, one of these convolved densities is close to the observed distribution and can be identified reliably. This leads to a sample complexity bound of the form:
$$
n \geq \widetilde{\mathcal{O}}\left(\frac{K^2}{\varepsilon^2}\right) \log \frac{1}{\delta},
$$
as shown in Theorem \ref{Theorem2}.

In the second stage, we develop a Fourier-analytic technique to show that if two convolved distributions in $\mathbb{G}_{K,\sigma}$ are close in TV distance, then their corresponding simplices are close in $\ell_2$ (and hence in TV) distance. This allows us to lift distributional closeness back to geometric proximity. Specifically, Theorem \ref{thm:mainNoisySimplexBound} establishes that if $\left\|
\left(f_{\widehat{\mathcal{S}}} - f_{\mathcal{S}}\right)*G_{\sigma}
\right\|_{\mathrm{TV}} \leq \varepsilon$, then
$$
\left\Vert 
f_{\widehat{\mathcal{S}}} - f_{\mathcal{S}}
\right\Vert_2 \leq
\varepsilon e^{\mathcal{O}\left(\frac{K}{\mathrm{SNR}^2}\right)},
$$
where $\mathrm{SNR} \triangleq \mathcal{L}_{\max}(\mathcal{S})/(K\sigma)$ is the signal-to-noise ratio. It should be noted that $\mathcal{L}_{\max}/K$ is proportional to the maximal component-wise standard deviation of the Dirichlet distribution $f_{\mathcal{S}}$.

\textbf{Error Lower Bound.} For the lower bounds, we employ classical information-theoretic techniques, including Assouad’s lemma \citep{devroye2012combinatorial} and the local Fano method, to derive tight minimax lower bounds on the sample complexity. This analysis significantly extends our previous work \citep{saberi2023sample} by covering both noisy and noiseless settings. Together, these upper and lower bounds provide a sharp characterization of the statistical limits of simplex learning and resolve several previously open problems in the field.


\section{Statistical Learning of Noisy Simplices}
\label{sec:main}

This section provides a detailed characterization of our achievability scheme, i.e., the error upper bound. As outlined earlier, we follow a structured proof sketch to establish the PAC-learnability of noisy simplices.

(\textit{Bounding the Candidate Set}): We begin by splitting the data in half. The first half is used to restrict the set of all $K$-simplices in $\mathbb{S}_{K}$ to a bounded subset $\mathbb{S}^{\mathcal{D}}_{K}$, consisting of simplices entirely contained within a finite-radius ball in $\mathbb{R}^K$. This allows us to discard distant candidates and focus on those located near the data. We prove that for a sufficiently large radius, we have $\mathcal{S}_T\in \mathbb{S}^{\mathcal{D}}_{K}$ with high probability.

(\textit{Quantization}): Next, we quantize this bounded set to obtain a finite $\varepsilon$-cover, denoted by $\widehat{\mathbb{S}}^{\mathcal{D}}_{K} = \{\mathcal{S}_1, \mathcal{S}_2, \ldots, \mathcal{S}_M\}$ for some $M \in \mathbb{N}$. For every simplex $\mathcal{S} \in \mathbb{S}^{\mathcal{D}}_{K}$, there exists an index $i \in \{1, \ldots, M\}$ such that the total variation distance satisfies $\|\mathbb{P}_{\mathcal{S}_i} - \mathbb{P}_{\mathcal{S}}\|_{\mathrm{TV}} \leq \varepsilon$.

(\textit{Density Selection}): We then use the second half of the data to select a ``good'' simplex $\widehat{\mathcal{S}}$ from $\widehat{\mathbb{S}}^{\mathcal{D}}_{K}$—that is, one whose convolved density $f_{\widehat{\mathcal{S}}} * G_{\sigma}$ is close in total variation distance to the true convolved density $f_{\mathcal{S}} * G_{\sigma}$. We show that this closeness is guaranteed provided the sample size meets the bound established in Section~\ref{sec:intro}.

(\textit{Recoverability from Noise}): Finally, we demonstrate that an accurate estimate of the noisy distribution implies a consistent estimate of the underlying noiseless simplex. This completes the proof.



\subsection{Bounding the Candidate Set}

We show how the first half of the dataset can be used to restrict the set of candidate simplices to those contained within a ball of finite radius in $\mathbb{R}^K$. This step is crucial for the later stages of the proof. The following lemma establishes that, given a sufficient number of samples from a noisy simplex $f_{\mathcal{S}} * G_{\sigma}$, one can identify a hypersphere in $\mathbb{R}^K$ that, with high probability, contains the true simplex $\mathcal{S}$.

\begin{lemma}[Constructing $\mathbb{S}^{\mathcal{D}}_K$]
\label{lemma2}
Let $\mathcal{D} = \left\{\boldsymbol{y}_1, \ldots, \boldsymbol{y}_{2m}\right\}$ be a set of i.i.d. samples drawn from $f_{\mathcal{S}} * G_{\sigma}$, for some $m \in \mathbb{N}$. For any $\delta > 0$, if
$m \geq 1000(K+1)(K+2)\log{\frac{6}{\delta}}$,
then with probability at least $1 - \delta$, the following two claims hold: i) The true simplex $\mathcal{S}$ lies within a $K$-dimensional hypersphere in $\mathbb{R}^K$ centered at $\boldsymbol{p}$ with radius $R$ defined as follows:
$$
\boldsymbol{p} = \frac{1}{2m} \sum_{i=1}^{2m} \boldsymbol{y}_i,
~~
R = 8\sqrt{(K+1)(K+2)D},
\quad\mathrm{where}\quad
D \triangleq \frac{1}{2m} \sum_{i=1}^{m} \|\boldsymbol{y}_{2i} - \boldsymbol{y}_{2i-1}\|_2^2.
$$
ii) The noise variance satisfies the upper bound
$\sigma^2 \leq R_n\triangleq \frac{D}{K-2}$.
\end{lemma}

The proof is provided in Appendix \ref{sec:proof of lemmas}. In the next section, we describe how to quantize $\mathbb{S}^{\mathcal{D}}_K$ in order to select a suitable candidate $\widehat{\mathcal{S}}$ for the underlying simplex.

\subsection{Quantization}

Let $\mathrm{C}^{K}(\boldsymbol{p}, R)$ denote the $K$-dimensional hypersphere described in Lemma \ref{lemma2}, which, with high probability, contains the true simplex $\mathcal{S}_T$. Our next goal is to construct a finite set of points $\mathrm{T}_\epsilon(\mathrm{C}^{K}(\boldsymbol{p}, R)) = \{\boldsymbol{p}_1, \ldots, \boldsymbol{p}_l\} \subset \mathbb{R}^K$ of size $l \in \mathbb{N}$ such that for every point $\boldsymbol{x} \in \mathrm{C}^{K}(\boldsymbol{p}, R)$, there exists some $i \in \{1, \ldots, l\}$ satisfying
$
\|\boldsymbol{x} - \boldsymbol{p}_i\|_2 \leq \epsilon.
$
We refer to $\mathrm{T}_\epsilon(\mathrm{C}^{K}(\boldsymbol{p}, R))$ as an $\epsilon$-covering set of the hypersphere. One way to construct such a covering set is to sample points uniformly at random from the sphere. We show that if the number of such sampled points exceeds $\mathcal{O}\left(\left(1 + {2R}/{\epsilon}\right)^{2K}\right)$, then with high probability, they form a valid $\epsilon$-covering set for $\mathrm{C}^{K}(\boldsymbol{p}, R)$.

Now, suppose we have constructed such a covering set. Using every combination of $K+1$ distinct points from $\mathrm{T}_\epsilon(\mathrm{C}^{K}(\boldsymbol{p}, R))$, we can form a $K$-simplex. Let us define the set of all such simplices as
$$
\widehat{\mathbb{S}}(\mathrm{C}^{K}(\boldsymbol{p}, R)) =
\left\{ \mathcal{S}(\boldsymbol{x}_1, \ldots, \boldsymbol{x}_{K+1}) ~\middle|~ \boldsymbol{x}_i \in \mathrm{T}_\epsilon(\mathrm{C}^{K}(\boldsymbol{p}, R)),~ i = 1, \ldots, K+1 \right\}.
$$
The total number of simplices in this set is at most $\binom{|\mathrm{T}_\epsilon(\mathrm{C}^{K}(\boldsymbol{p}, R))|}{K+1}$. The following lemma establishes that this collection forms an $\epsilon$-representative set for all $(\underline{\theta}, \bar{\theta})$-isoperimetric $K$-simplices inside the hypersphere.

\begin{lemma}[Quantization of $\mathbb{S}^{\mathcal{D}}_K$]
\label{quantization lemma}
Let $\epsilon \in (0, 1)$, and suppose $\widehat{\mathbb{S}}(\mathrm{C}^{K}(\boldsymbol{p}, R))$ is constructed using the covering set $\mathrm{T}_{\alpha\epsilon}(\mathrm{C}^{K}(\boldsymbol{p}, R))$ for some $\alpha>0$. Then, $\widehat{\mathbb{S}}(\mathrm{C}^{K}(\boldsymbol{p}, R))$ forms an $\epsilon$-representative set for all $(\underline{\theta}, \bar{\theta})$-isoperimetric $K$-simplices contained in $\mathrm{C}^{K}(\boldsymbol{p}, R)$, provided
$$
\alpha \leq \frac{\mathrm{Vol}(\mathcal{S})^{1/K}}{5(K+1)\bar{\theta}}.
$$
\end{lemma}
The proof is given in Appendix~\ref{sec:proof of lemmas}. This completes the construction of the finite candidate set $\mathbb{S}^{\mathcal{D}}_K$, allowing us to proceed to the next stage of the algorithm: identifying a ``good" candidate for the true simplex.

\subsection{Density Selection}

Thus far, we have constructed a finite set of representative simplices $\widehat{\mathbb{S}}^{\mathcal{D}}_{K}$ and aim to identify a ``good" simplex from this set. To accomplish this, we invoke a fundamental result from \citet{devroye2012combinatorial}, which plays a central role in our analysis:

\begin{theorem}[Theorem 6.3 of \citet{devroye2012combinatorial}]
\label{combinatoeic algorithm}
Let $\mathcal{F} = \{f_1, \ldots, f_M\}$ be a finite collection of distinct distributions. Given $\epsilon, \delta > 0$, suppose we observe $n \geq \frac{\log(3M^2/\delta)}{2\epsilon^2}$ i.i.d. samples from an arbitrary distribution $g$. Then, there exists a deterministic algorithm $\mathscr{A}$ that receives the samples as input and outputs an index $j \in \{1, \ldots, M\}$ such that, with probability at least $1 - \delta$,
$$
\|f_j - g\|_{\mathrm{TV}} \leq 3 \cdot \min_{i \in \{1, \ldots, M\}} \|f_i - g\|_{\mathrm{TV}} + 4\epsilon.
$$
\end{theorem}
The proof is available in the reference. By combining Theorem \ref{combinatoeic algorithm} with Lemmas \ref{lemma2} and \ref{quantization lemma}, we arrive at one of our main results: the class of noisy simplices in $\mathbb{R}^K$ with bounded geometric properties, confined within a sphere of radius $R$, is PAC-learnable.

\begin{theorem}[PAC Learnability of Noisy Simplices in $\mathrm{C}^{K}(\boldsymbol{p}, R)$]
\label{thm:PAC:mainFinalPAC}
Let $\boldsymbol{p} \in \mathbb{R}^K$, and let $R, R_n > 0$. Then, the class of $(\underline{\theta}, \bar{\theta})$-isoperimetric $K$-simplices contained within the hypersphere $\mathrm{C}^{K}(\boldsymbol{p}, R)$ and convolved with an isotropic Gaussian noise $\mathcal{N}(\boldsymbol{0}, \sigma^2 \boldsymbol{\mathrm{I}})$ for $\sigma \leq R_n$ is PAC-learnable. Specifically, for any $\epsilon, \delta > 0$, assume we observe at least $n$ i.i.d. samples from a distribution $f_{\mathcal{S}} * G_{\sigma}$, where $\sigma \leq R_n$ and the vertices of $\mathcal{S}$ lie within $\mathrm{C}^{K}(\boldsymbol{p}, R)$. If the sample size satisfies
$$
n \geq 50 \cdot \frac{\log\left(\frac{30 R_n \sqrt{K}}{\delta \epsilon}\right) + 2(K+1)^2 \log\left(1 + \frac{100 R \bar{\theta}(K+1)}{\epsilon \mathrm{Vol}(\mathcal{S})^{1/K}}\right)}{\epsilon^2},
$$
then there exists an algorithm $\mathscr{A}$ that outputs $\mathcal{S}_{\mathscr{A}}$ and $\sigma_{\mathscr{A}}$ such that, with probability at least $1 - \delta$,
$$
\left\|f_{\mathcal{S}} * G_{\sigma} - f_{\mathcal{S}_{\mathscr{A}}} * G_{\sigma_{\mathscr{A}}} \right\|_{\mathrm{TV}} \leq \epsilon.
$$
\end{theorem}
The proof is provided in Appendix~\ref{sec:proof of theorems}. Let us now restate the result of Theorem \ref{thm:PAC:mainFinalPAC} in simpler terms. Define $\mathbb{G}_{K,\sigma}(\boldsymbol{p}, R)$ as the family of distributions
$
\left\{ f_{\mathcal{S}} * G_{\sigma} \;\middle|\; \mathcal{S} \subseteq \mathrm{C}^K(\boldsymbol{p}, R) \right\}.
$
We have shown that, in an agnostic setting, there exists an algorithm $\mathscr{A}$ such that for any $\epsilon, R, R_n > 0$ and $\boldsymbol{p} \in \mathbb{R}^K$, with
$
n \geq \widetilde{\mathcal{O}}\left(\frac{K^2 \log R}{\epsilon^2}\right)
$
samples drawn from an arbitrary distribution $g$ in the set
$$
\bigcup_{\sigma \leq R_n} \mathbb{G}_{K,\sigma}(\boldsymbol{p}, R),
$$
the algorithm outputs a simplex $\mathcal{S}_{\mathscr{A}}$ and a noise level $\sigma_{\mathscr{A}}$ such that, with high probability,
$$
\left\| f_{\mathcal{S}_{\mathscr{A}}} * G_{\sigma_{\mathscr{A}}} - f_{\mathcal{S}} * G_{\sigma} \right\|_{\mathrm{TV}} 
\leq 
4 \min_{\substack{\mathcal{S}^* \subseteq \mathrm{C}^K(\boldsymbol{p}, R) \\ \sigma^* \leq R_n}} \left\| f_{\mathcal{S}^*} * G_{\sigma^*} - f_{\mathcal{S}} * G_{\sigma} \right\|_{\mathrm{TV}} 
+ \epsilon.
$$
In the realizable case, where the true distribution lies exactly in the hypothesis class, the first term on the right-hand side vanishes. Also note that the sample complexity depends on $\bar{\theta}$, which controls the regularity of the simplex. Learning highly stretched simplices from noisy data is inherently difficult: even a small amount of noise may push nearly all samples outside the simplex, making recovery statistically infeasible.

A major limitation of Theorem \ref{thm:PAC:mainFinalPAC} is the assumption that simplices are contained within a bounded hypersphere of radius $R$. However, Lemma \ref{lemma2} shows that this restriction is not fundamental. The following result removes this assumption and establishes the full PAC-learnability of the entire class of noisy simplices:

\begin{theorem}[PAC Learnability of Noisy Simplices]
\label{Theorem2}
The class of $(\underline{\theta}, \bar{\theta})$- isoperimetric $K$-simplices convolved with isotropic Gaussian noise, i.e.,
$$
\bigcup_{\sigma > 0} \mathbb{G}_{K,\sigma},
$$
is PAC-learnable. Specifically, for any simplex $\mathcal{S} \in \mathbb{S}_K$ and noise level $\sigma > 0$, suppose we observe at least $n \geq \widetilde{\mathcal{O}}\left(K^2 / \varepsilon^2\right)$ i.i.d. samples from the distribution $f_{\mathcal{S}} * G_{\sigma}$. Then, there exists an algorithm $\mathscr{A}$ that outputs a noisy simplex $f_{\mathcal{S}_{\mathscr{A}}} * G_{\sigma_{\mathscr{A}}}$ such that, with high probability,
$$
\left\| f_{\mathcal{S}_{\mathscr{A}}} * G_{\sigma_{\mathscr{A}}} - f_{\mathcal{S}} * G_{\sigma} \right\|_{\mathrm{TV}} \leq \varepsilon.
$$
\end{theorem}
The proof of Theorem \ref{Theorem2} is given in Appendix \ref{sec:proof of theorems}. Notably, the logarithmic dependence on $R$ in the sample complexity of Theorem \ref{thm:PAC:mainFinalPAC} is eliminated in this more general result.


\subsection{Recoverability from Noise}

Theorem \ref{Theorem2} demonstrates that for a simplex $\mathcal{S} \in \mathbb{S}_K$ with adequate geometric regularity, the noise-smoothed distribution $f_{\mathcal{S}} * G_{\sigma}$ can be learned to within arbitrarily small error, provided the sample size $n$ exceeds a suitable bound. It remains to show that a non-asymptotically close estimate of $f_{\mathcal{S}} * G_{\sigma}$ also leads to a close approximation of the underlying distribution $f_{\mathcal{S}}$ itself. Our analysis is based on three key observations:
\begin{itemize}
\item For any two geometrically regular simplices $\mathcal{S}_1, \mathcal{S}_2 \in \mathbb{S}_K$, the difference $f_{\mathcal{S}_1} - f_{\mathcal{S}_2}$ is a low-frequency function, meaning that its Fourier transform is concentrated near the origin.

\item Corruption by additive Gaussian noise corresponds to convolution with the kernel $G_{\sigma}$, which in the Fourier domain turns into a corresponding pointwise multiplication.

\item The Fourier transform of a Gaussian decays rapidly away from the origin, so convolution with $G_\sigma$ attenuates high-frequency components while preserving low-frequency structure. Consequently, differences between low-frequency functions are not significantly obscured by this smoothing operation.
\end{itemize}

Formally, suppose $\mathcal{S}_1, \mathcal{S}_2 \in \mathbb{S}_K$ are distinct simplices with controlled geometric complexity—e.g., bounded aspect ratios or isoperimetric constants. If their smoothed densities satisfy
$
\|f_{\mathcal{S}_1} * G_{\sigma} - f_{\mathcal{S}_2} * G_{\sigma}\|_{\mathrm{TV}} \geq \varepsilon,
$
then their unsmoothed counterparts also differ in $\ell_2$ or total variation distance, up to constants depending on $\varepsilon$, $\sigma$, and the geometric regularity of the simplices. This principle is captured in the following general result.

\begin{theorem}[Recovery of Low-Frequency Densities from Additive Noise]
\label{thm:generalResultTheorem}
Let $\mathscr{F}$ be a family of densities supported on $\mathbb{R}^K$. Suppose that for sufficiently large $\alpha > 0$, the following inequality holds for all $f, g \in \mathscr{F}$:
$$
\frac{1}{(2\pi)^K}
\int_{\|\boldsymbol{\omega}\|_\infty \geq \alpha}
\left|\mathcal{F}\{f - g\}(\boldsymbol{\omega})\right|^2
\leq \zeta(\alpha^{-1}) \int_{\mathbb{R}^K} |f - g|^2,
$$
where $\zeta$ is an increasing function with $\zeta(0) = 0$ and continuous at $0$, and $\mathcal{F}\{\cdot\}$ denotes the Fourier transform. Further, assume that a smoothing kernel $Q$ satisfies
$$
\inf_{\|\boldsymbol{\omega}\|_\infty \leq \alpha} |\mathcal{F}\{Q\}(\boldsymbol{\omega})| \geq \eta(\alpha),
$$
for some nonnegative decreasing function $\eta(\cdot)$. Then, for any $\varepsilon > 0$ and $f, g \in \mathscr{F}$ with $\|f - g\|_2 \geq \varepsilon$, it holds that
$$
\|(f - g) * Q\|_2 \geq \frac{\varepsilon}{(2\pi)^{K/2}} \cdot \sup_{\alpha \geq 0} \left[ \eta(\alpha) \sqrt{1 - \zeta(\alpha^{-1})} \right].
$$
\end{theorem}
The proof is provided in Appendix \ref{sec:app:noise}, along with explanatory remarks. This result provides a general technique to establish the recoverability of latent signals (or distributions) in the presence of additive noise, provided both the signal class and the noise kernel are concentrated in low-frequency regions of the Fourier domain. For instance, Gaussian noise predominantly affects high-frequency components, leaving low-frequency signals mostly intact. This leads to the following corollary:
\begin{corollary}[Recoverability from Gaussian Noise $\mathcal{N}(\boldsymbol{0}, \sigma^2 \boldsymbol{I})$]
\label{corl:GaussinNoiseMain}
Under the setup of Theorem~\ref{thm:generalResultTheorem}, suppose the additive noise kernel $Q$ is the standard Gaussian $\mathcal{N}(\boldsymbol{0}, \sigma^2 \boldsymbol{I})$. Then, for any $f, g \in \mathscr{F}$ with $\|f - g\|_2 \geq \varepsilon$, we have
$$
\|(f - g) * Q\|_2 \geq \frac{\varepsilon}{(2\pi)^{K/2}} \cdot \sup_{\alpha > C} \left[ \sqrt{1 - \zeta(1/\alpha)} \cdot e^{-K (\sigma \alpha)^2 / 2} \right],
$$
where $C$ is an absolute constant depending on the regularity of the family $\mathscr{F}$.
\end{corollary}
The proof appears in Appendix~\ref{sec:app:noise}. We now specialize this result to the simplex family $\mathbb{S}_K$.
\begin{theorem}[Recoverability of Simplices from Additive Noise]
\label{thm:mainNoisySimplexBound}
Let $\mathcal{S}_1, \mathcal{S}_2 \in \mathbb{S}_K$ be $(\bar{\theta}, \underline{\theta})$-isoperimetric simplices. Suppose
$
\|f_{\mathcal{S}_1} * G_{\sigma} - f_{\mathcal{S}_2} * G_{\sigma}\|_{\mathrm{TV}} \leq \varepsilon,
$
for some $\varepsilon \geq 0$. Then
$$
\|f_{\mathcal{S}_1} - f_{\mathcal{S}_2}\|_2 \leq \varepsilon \cdot e^{\mathcal{O}(K / \mathrm{SNR}^2)},
$$
where $\mathrm{SNR}$ denotes the signal-to-noise ratio defined earlier in the paper. This result also holds for TV distance, albeit with a different constant.
\end{theorem}

The proof, provided in Appendix~\ref{sec:app:noise}, hinges on showing that geometrically regular simplices are predominantly low-frequency objects. Specifically, Lemma C.3 establishes that for such a simplex $\mathcal{S} \in \mathbb{S}_K$,
$$
\frac{1}{(2\pi)^K} \int_{\|\boldsymbol{\omega}\|_\infty \geq \alpha} \left|\mathcal{F}\{f_{\mathcal{S}}\}(\boldsymbol{\omega})\right|^2 \leq \frac{1}{\mathrm{Vol}(\mathcal{S})} \cdot \mathcal{O}\left(\frac{K}{\alpha}\right),
$$
for sufficiently large $\alpha > 0$, where constants in the $\mathcal{O}(\cdot)$ bound depend only on regularity parameters of $\mathcal{S}$. This bound is trivial for $K = 1$ (since one-dimensional simplices are indicator functions with sinc-like Fourier transforms), but generalizing to higher dimensions involves significantly more effort, as carried out in the proof of Theorem~\ref{thm:mainNoisySimplexBound}. A similar decay property also holds for the difference between two simplices, enabling us to invoke Corollary~\ref{corl:GaussinNoiseMain} to complete the proof of the main recoverability claim.


\section{Non-asymptotic Impossibility Results}
\label{sec:lower-bound}

In this section, we derive information-theoretic lower bounds for the number of samples required by any algorithm $\mathscr{A}$ to estimate a simplex from noisy observations, given a target TV distance of at most $\epsilon$.

\begin{theorem}
\label{thm:noisySimplexFirstLowerBound}
Any algorithm that estimates a simplex $\mathcal{S}_T\in\mathbb{S}_K$ within a TV distance of at most $\epsilon>0$ from samples contaminated by Gaussian noise $\mathcal{N}\left(\boldsymbol{0},\sigma^2\mathbf{I}\right)$ requires a minimum sample complexity of $n \geq \Omega\left(\frac{\sigma^2\bar{\theta}^2}{\epsilon^2} + \frac{K^3\sigma^2}{\epsilon^2}\right)$.
\end{theorem}
The proof is provided in Appendix \ref{sec:app:lowerbound}.An important implication of this theorem is that if the simplex does not satisfy the isoperimetricity property (i.e., $\bar{\theta}$ is unbounded) or lies in a lower-dimensional subspace, then no algorithm can consistently estimate the simplex. To build intuition for this fact, consider two degenerate simplices whose total variation distance is close to $1$, but whose KL divergence is small. In such a scenario, it becomes statistically impossible to distinguish between them using finite samples. This highlights the critical role of isoperimetricity in ensuring the learnability of the simplex. A notable limitation of the above lower bound is its dependence on $\sigma$. Specifically, when $\sigma$ is relatively small, the  bound approaches zero and becomes trivial. To address this issue, we also derive lower bounds for the noiseless regime. Before analyzing the estimation of a simplex in terms of TV distance, let us establish a sample complexity lower bound for estimating its vertices. Accordingly, consider the following definition:

\begin{definition}[$\ell_1$-Vertex Distance]
The $\ell_1$-vertex distance between the two simplices $\mathcal{S}_1,\mathcal{S}_2\in\mathbb{S}_K$ is defined as
\begin{equation}
\|V_{\mathcal{S}_1} - V_{\mathcal{S}_2}\|_1 = 
\min_{\pi\in\mathsf{P}(\{0,\ldots,K\})}
\sum_{i=0}^{K}{\|\boldsymbol{v}_i^{\mathcal{S}_1} - \boldsymbol{v}_{\pi(i)}^{\mathcal{S}_2}\|_1},
\end{equation}
where $\mathsf{P}(\{0,\ldots,K\})$ denotes the set of all permutations of $\{0,\ldots,K\}$.
\end{definition}
Using this vertex-based distance, we have the following impossibility result:
\begin{theorem}
\label{thm:noislessSimplexFirstLowerBound}
The minimum number of noiseless samples required (by any algorithm) to estimate the vertices of a $K$-simplex within $\ell_1$-vertex distance of at most $\epsilon>0$ satisfies $n \geq \Omega\left(\frac{K^2}{\epsilon}\right)$.
\end{theorem}
The proof is provided in Appendix \ref{sec:app:lowerbound}. In particular, the proof also establishes the following remark: for any algorithm that outputs a simplex entirely contained within the true simplex $\mathcal{S}_T$, the TV distance and the $\ell_1$-vertex distance between the estimated and true simplices are of the same order. Therefore, the lower bound in Theorem \ref{thm:noisySimplexFirstLowerBound} also applies to such estimators. Notably, the maximum likelihood (ML) estimator that selects the minimum-volume simplex enclosing all data points appears to produce such an estimate. We now extend this analysis to obtain a lower bound on sample complexity in terms of TV distance in the noiseless regime:
\begin{theorem}
\label{thm:noislessSimplexSecondLowerBound}
The minimum number of noiseless samples required to estimate a simplex within TV distance at most $\epsilon$ satisfies $n \geq \Omega\left(\frac{K}{\epsilon}\right)$.
\end{theorem}
The proof is given in Appendix \ref{sec:app:lowerbound}. Our findings suggest that geometric properties—such as isoperimetry and the specific shape of the simplex—play a significant role in determining sample complexity. Although we do not formally prove this, it appears that for simplices close to the standard simplex (with equal and axis-aligned edges in $\mathbb{R}^K$), approximately $K/\epsilon$ samples may suffice to achieve TV error less than $\epsilon$. Since such simplices are often used in constructing lower bounds, this may explain the difficulty in improving our bounds further in these cases.

Combining our results for the noisy and noiseless regimes, we obtain the following unified lower bound:
\begin{corollary}[Refining the Result of Theorem \ref{thm:noisySimplexFirstLowerBound}]
\label{cor:noisySimplexSecondLowerBound}
The minimum number of samples needed to estimate a $K$-simplex within TV distance at most $\epsilon$ from samples contaminated by Gaussian noise $\mathcal{N}(\boldsymbol{0},\sigma^2\mathbf{I})$ satisfies:
\begin{equation}
n
\geq
\Omega
\left(
\frac{K^3\sigma^2}{\epsilon^2} + 
\frac{K}{\epsilon}
\right).
\end{equation}
\end{corollary}
This corollary follows directly from Theorems \ref{thm:noisySimplexFirstLowerBound} and \ref{thm:noislessSimplexSecondLowerBound}. 

\section{Conclusions}
\label{sec:conc}

We have presented information-theoretic upper and lower bounds on the sample complexity of PAC-learning high-dimensional simplices from noisy observations. Specifically, we established that given a sufficient number of samples corrupted by additive Gaussian noise (or a broad class of similar noise models), one can estimate the underlying simplex to arbitrarily small $\ell_2$ or total variation (TV) error.

Our upper bound exhibits a presumably optimal polynomial dependence on key parameters such as the dimension $K$ and the signal-to-noise ratio (SNR), closely matching the known sample complexity in the noiseless setting. Notably, we provided a rigorous theoretical explanation for an empirical phenomenon frequently observed in practice: the performance of heuristic simplex learning algorithms undergoes a sharp phase transition with respect to the noise level. In particular, when $\mathrm{SNR} \geq \Omega(K^{1/2})$, these algorithms perform nearly as well as in the noiseless regime, whereas their accuracy degrades rapidly for lower SNR values.

In addition, we derived new lower bounds extending our prior work, using tools such as Assouad’s lemma and the local Fano method. We showed that at least $\Omega(K^3 \sigma^2 / \epsilon^2)$ samples are necessary to estimate a simplex within TV distance $\epsilon$ in the presence of noise, and we established a matching lower bound of $\Omega(K / \epsilon)$ in the noiseless case. These results offer a comprehensive characterization of the fundamental limits of simplex estimation, answering previously unresolved theoretical questions and revealing the precise dependence of sample complexity on both dimension and noise level.

Our analysis incorporates several modern techniques, including sample compression arguments adapted from recent work on Gaussian mixture models, tools from high-dimensional geometry, and a novel Fourier-based denoising method. This Fourier-based approach may have broader utility in other problems involving recovery of structured distribution families from additive noise. Looking ahead, promising directions for future work include extending our framework to accommodate more general noise and distortion models that better capture practical data acquisition settings. Another important challenge is the development of computationally efficient (i.e., polynomial-time) algorithms for simplex learning. Addressing this computational barrier could significantly increase the practical relevance of our theoretical insights.


\vskip 0.2in
\bibliography{sample}


\appendix

\section{Proofs for Statistical Learning of Noisy Simplices}
\label{sec:proof of theorems}
\begin{proof}[proof of Theorem \ref{thm:PAC:mainFinalPAC}]
Suppose that $\mathcal{S}$ is a $\left(\underline{\theta}, \bar{\theta}\right)$-isoperimetric simplex in $\mathbb{R}^K$ which is confined in a $K$-dimensional hyper-sphere $\mathrm{C}^{K}(\boldsymbol{p}, R)$. From Lemma \ref{quantization lemma}, we know how to build an $\epsilon$-representative set, for any $\epsilon>0$, for all $\left(\underline{\theta}, \bar{\theta}\right)$-isoperimetric $K$-simplices in $\mathrm{C}^{K}(\boldsymbol{p}, R)$. Let us denote this set with $\widehat{\mathbb{S}}(\mathrm{C}^{K}(\boldsymbol{p}, R))$. 

Now Consider the class of all isotropic Gaussian noise, $\mathcal{N}\left(\boldsymbol{0},\sigma^2\boldsymbol{I} \right)$, with $\sigma \leq R_n$. We denote this class of distributions with $\mathfrak{N}^K\left(R_n\right)$. Consider a set $C^{\epsilon/\sqrt{K}}_{R_n} = \left\{0, \frac{\epsilon}{\sqrt{K}}, \frac{2\epsilon}{\sqrt{K}}, \cdots, R_n\right\}$, which $\frac{\epsilon}{\sqrt{K}}$-covers the interval $\left[0, R_n\right]$. Now for all $\sigma_i \in C^{\epsilon/\sqrt{K}}_{R_n}$ we put $\mathcal{N}\left(\boldsymbol{0},\sigma_i^2\boldsymbol{I} \right)$ in a set called $\widehat{\mathfrak{N}}^K\left(R_n\right)$. From Theorem $1.1$ in \citep{devroye2018total} it can be shown that $\widehat{\mathfrak{N}}^K\left(R_n\right)$ is an $\epsilon$-representative set for $\mathfrak{N}^K\left(R_n\right)$. According to the way we build $\widehat{\mathfrak{N}}^K\left(R_n\right)$, we have $\vert\widehat{\mathfrak{N}}^K\left(R_n\right)\vert \leq \frac{R_n\sqrt{K}}{\epsilon}$. Now we build a set of noisy simplices as follows:
\begin{equation}
    \widehat{\mathbb{G}}_{K, \sigma}(R, R_n) \triangleq \left\{f_{\mathcal{S}}\left(\boldsymbol{x}\right) \ast G_{\sigma}\left(\boldsymbol{x}\right)\vert \mathcal{S} \in \widehat{\mathbb{S}}(\mathrm{C}^{K}(\boldsymbol{p}, R)), G_{\sigma} \in \widehat{\mathfrak{N}}^K\left(R_n\right)\right\}.
\end{equation}
Now for any density function $f_{\mathcal{S}}\left(\boldsymbol{x}\right) \ast G_{\sigma}\left(\boldsymbol{x}\right)$, where $\mathcal{S} \in \mathrm{C}^{K}(\boldsymbol{p}, R)$ and $\sigma \leq R_n$, we can find some $\mathcal{S}^{\star} \in \widehat{\mathbb{S}}(\mathrm{C}^{K}(\boldsymbol{p}, R))$, and $G_{\sigma^\star} \in \widehat{\mathfrak{N}}^K\left(R_n\right)$ such that
\begin{align*}
    &\|f_{\mathcal{S}}- f_{\mathcal{S}^{\star}}\|_{\mathrm{TV}} \leq \epsilon,
    \\
    &\|G_{\sigma}- G_{\sigma^{\star}}\|_{\mathrm{TV}} \leq \epsilon,
    \\
    &f_{\mathcal{S}^{\star}} \ast G_{\sigma^{\star}} \in \widehat{\mathbb{G}}_{K, \sigma}(R, R_n).
\end{align*}
Now for the distance between $f_{\mathcal{S}^{\star}}^{\sigma^{\star}}$ and $f_{\mathcal{S}}^{\sigma}$ we have:
\begin{align}
    \|f_{\mathcal{S}^{\star}} \ast G_{\sigma^{\star}}-f_{\mathcal{S}} \ast G_{\sigma}\|_{\mathrm{TV}}
    \leq & \|f_{\mathcal{S}^{\star}} \ast G_{\sigma^{\star}}- f_{\mathcal{S}^{\star}} \ast G_{\sigma}\|_{\mathrm{TV}} + \|f_{\mathcal{S}^{\star}} \ast G_{\sigma}- f_{\mathcal{S}} \ast G_{\sigma}\|_{\mathrm{TV}}
    \nonumber \\
    \leq & \|G_{\sigma^{\star}} - G_{\sigma}\|_{\mathrm{TV}} + \|f_{\mathcal{S}^{\star}} - f_{\mathcal{S}} \|_{\mathrm{TV}}
    \nonumber \\ 
    \leq & 2\epsilon.
\end{align}
From the above inequalities it can be seen that for any density function $f_{\mathcal{S}} \ast G_{\sigma}$, where $\mathcal{S} \in \mathrm{C}^{K}(\boldsymbol{p}, R)$ and $\sigma \leq R_n$, there exist some density function $f^{\star} \in \widehat{\mathbb{S}}_n(\mathrm{C}^{K}(\boldsymbol{p}, R), R_n)$
where $\|f^{\star} - f_{\mathcal{S}} \ast G_{\sigma}\|_{\mathrm{TV}} \leq 2\epsilon$. Therefore the set $ \widehat{\mathbb{G}}_{K, \sigma}(R, R_n)$ is a $2\epsilon$-representative set for the class of $K$-simplices confined in a hyper-sphere with radius $R$ which convolved with a Gaussian noise with variance $\sigma \leq R_n$. We show this class with 

\begin{equation}
    \mathbb{G}_{K,\sigma}\left(R, R_n\right)
\triangleq
\left\{
f_{\mathcal{S}}*G_{\sigma}
\vert~
\mathcal{S}\in\mathbb{S}_K, \mathcal{S} \in \mathrm{C}^{K}(\boldsymbol{p}, R), \sigma \leq R_n
\right\}
\end{equation}

Assume that we have a set of i.i.d. samples from some distribution $f_{\mathcal{S}}*G_{\sigma} \in \mathbb{G}_{K,\sigma}\left(R, R_n\right)$. From Theorem \ref{combinatoeic algorithm}, we know that there exists a deterministic algorithm $\mathscr{A}$ such that given
$$n \geq \frac{\log{(3\vert\widehat{\mathbb{G}}_{K, \sigma}(R, R_n)\vert^2/\delta)}}{2\epsilon^2}$$
i.i.d. samples from $f_{\mathcal{S}}*G_{\sigma}$, the output of the algorithm denoted by  $f_{\mathcal{S}_{\mathscr{A}}}*G_{\sigma_{\mathscr{A}}}$, with probability at least $1-\delta$, satisfies:
\begin{align}
\|f_{\mathcal{S}_{\mathscr{A}}}*G_{\sigma_{\mathscr{A}}} -  f_{\mathcal{S}}*G_{\sigma} \|_{\mathrm{TV}}  \leq &~
3 \min_{f \in \widehat{\mathbb{G}}_{K, \sigma}(R, R_n)} \|f - f_{\mathcal{S}}*G_{\sigma}\|_{\mathrm{TV}} +4\epsilon
\nonumber \\
 \leq&~ 6\epsilon +4\epsilon = 10\epsilon.
 \label{combinatorial method 2}
\end{align}
For the cardinality of $\widehat{\mathbb{G}}_{K, \sigma}(R, R_n)$ we have:
\begin{align}
    \vert\widehat{\mathbb{G}}_{K, \sigma}(R, R_n)\vert 
    =& \vert\widehat{\mathbb{S}}(\mathrm{C}^{K}(\boldsymbol{p}, R))\vert \vert\widehat{\mathfrak{N}}^K\left(R_n\right)\vert 
    \nonumber \\
    \leq& \vert\widehat{\mathbb{S}}(\mathrm{C}^{K}(\boldsymbol{p}, R))\vert \frac{R_n\sqrt{K}}{\epsilon}.
    \label{noisy representative cardinality}
\end{align}
And, according to Lemma \ref{quantization lemma} the followings hold for the cardinality of $\widehat{\mathbb{S}}(\mathrm{C}^{K}(\boldsymbol{p}, R))$: 
\begin{align}
\bigg\vert\widehat{\mathbb{S}}(\mathrm{C}^{K}(\boldsymbol{p}, R))\bigg\vert = &
\nonumber
 \binom{\left\vert \mathrm{T}_{\frac{\alpha\epsilon}{K+1}}(\mathrm{C}^{K}(\boldsymbol{p}, R))\right\vert}{K+1} 
 \\ \nonumber \leq &~ \bigg\vert \mathrm{T}_{\frac{\alpha\epsilon}{K+1}}(\mathrm{C}^{K}(\boldsymbol{p}, R))\bigg\vert^{K+1}
 \\ \nonumber \leq &~ \left(\left(1+ \frac{2(K+1) R}{\alpha\epsilon} \right)^K\right)^{K+1}
  \\   = &~ \left(1+ \frac{2(K+1)R}{\alpha\epsilon} \right)^{K(K+1)}.
  \label{quantized simplices set cardinality}
\end{align}
Now from \ref{quantized simplices set cardinality} and \ref{noisy representative cardinality} we have:
\begin{align}
    \vert\widehat{\mathbb{G}}_{K, \sigma}(R, R_n)\vert
    \leq & \frac{R_n\sqrt{K}}{\epsilon} \left(1+ \frac{2(K+1)R}{\alpha\epsilon} \right)^{K(K+1)}.
    \label{quantized noisy simplices cardinality}
\end{align}
Then using \ref{quantized noisy simplices cardinality} and \ref{combinatorial method 2}, we can say that for any $\epsilon_,\delta>0$, there exists a PAC-learning algorithm $\mathscr{A}$ for the class of noisy $\left(\underline{\theta}, \bar{\theta}\right)$-isoperimetric $K$-simplices in $\mathbb{G}_{K,\sigma}\left(R, R_n\right)$, whose sample complexity is bounded as follows:
\begin{align}
n \geq &~ 50\frac{\log{\frac{10R_n\sqrt{K}}{\epsilon}} + 2(K+1)^2\log \left(1+ \frac{20(K+1)R}{\alpha\epsilon} \right) + \log\frac{3}{\delta}}{\epsilon^2}
\nonumber \\
= &~ 50\frac{\log{\frac{30R_n\sqrt{K}}{\delta\epsilon}} + 2(K+1)^2\log \left(1+ \frac{100R\bar{\theta}(K+1)}{\epsilon\mathrm{Vol}\left(\mathcal{S}\right)^{\frac{1}{K}}}\right)}{\epsilon^2}.
\end{align}
In other words, given that the number of samples $n$ satisfies the above lower-bound, then with probability at least $1-\delta$ we have
$$
\|f_{\mathcal{S}_{\mathscr{A}}}*G_{\sigma_{\mathscr{A}}} - f_{\mathcal{S}}*G_{\sigma} \|_{\mathrm{TV}} \leq \epsilon.
$$
This completes the proof.
Using simple algebra, the bound can be further simplified into $n\ge O\left(\frac{K^2}{\epsilon^2}\log{\frac{K}{\epsilon}}\right)$ which completes the proof.
\end{proof}

\begin{proof}[proof of Theorem \ref{Theorem2}]
From Theorem \ref{thm:PAC:mainFinalPAC}, we know that the class of $\left(\underline{\theta}, \bar{\theta}\right)$-isoperimetric $K$-simplices contained in the $K$-dimensional hyper-sphere $\mathrm{C}^{K}(\boldsymbol{p}, R)$ and convolved with an isotropic Gaussian noise $\mathcal{N}\left(\boldsymbol{0}, \sigma^2 \boldsymbol{\mathrm{I}}\right)$, with $\sigma \leq \mathrm{R}_n$, is PAC-learnable, with sample complexity $O\left(\frac{K^2}{\epsilon^2}\log{\frac{K}{\epsilon}}\right)$, where $\epsilon$ is defined accordingly. And from Lemma \ref{lemma2}, we know that if we have $O\left(K^2\right)$ samples from a noisy simplex $f_{\mathcal{S}}*\mathbb{G}_\sigma$, we can find a $K$-dimensional sphere which with probability at least $1-\delta/2$ contains the true simplex as long as the radius $R$ of the sphere and the upper-bound of the noise variance $R_n$ satisfy
\begin{align}
& R \leq 4\sqrt{K+1}\left(1+\frac{K+2}{d_{\mathrm{S}}/\sigma}\right)d_{\mathrm{S}}
\label{upper bound for the radius of containing sphere}
\\
& R_n \leq \frac{K+2}{K-3}\left(1+\frac{d_{\mathrm{S}}/\sigma}{K+2}\right)\sigma.
\label{upper bound for the noise variance}
\end{align}

Therefore, it can be shown that there exists an algorithm $\mathscr{A}$ such that given $n$ i.i.d samples from $\mathbb{G}_\mathcal{S}$ with
\begin{align}
n
\geq&~ 50\frac{\log{\frac{30R_n\sqrt{K}}{\delta\epsilon}} + 2(K+1)^2\log \left(1+ \frac{100R\bar{\theta}(K+1)}{\epsilon\mathrm{Vol}\left(\mathcal{S}\right)^{\frac{1}{K}}}\right)}{\epsilon^2}
+ 2000(K+1)(K+2)\log{\frac{6}{\delta}}
\nonumber\\
=&~ 100\frac{\log{6/\delta}+(K+2)^2\log \left(1+ \frac{100\bar{\theta}(K+1)^{3/2}}{\epsilon\mathrm{Vol}\left(\mathcal{S}\right)^{\frac{1}{K}}}\left(1+\frac{K+2}{d_{\mathrm{S}}/\sigma}\right)d_{\mathrm{S}}\right)}{\epsilon^2}
\nonumber\\
=&~ O\left(\frac{K^2}{\epsilon^2}\log\frac{K}{\epsilon}\right),
\label{sample complexity of noisy simplex}
\end{align}
the output of the algorithm $f_{\mathcal{S}_{\mathscr{A}}}*G_{\sigma_{\mathscr{A}}}$ with probability at least $1-\delta$ satisfies
\begin{equation*}
\|f_{\mathcal{S}_{\mathscr{A}}}*G_{\sigma_{\mathscr{A}}} - f_{\mathcal{S}}*\mathbb{G}_\sigma\|_\mathrm{TV} \leq ~\epsilon.
\end{equation*}
It should be noted that $R$ and $R_n$ in \eqref{sample complexity of noisy simplex} is already replaced with the bound in the r.h.s of \eqref{upper bound for the radius of containing sphere} and \eqref{upper bound for the noise variance}.
This way, the proof is completed.
\end{proof}

\subsection{Proof of Lemmas}
\label{sec:proof of lemmas}
\begin{proof}[proof of Lemma \ref{lemma2}]
To find a $K$-dimensional sphere containing the true simplex, it suffices to find a point $\boldsymbol{p}$ inside the simplex and an upper-bound $R$ for its diameter, i.e., the maximum distance between two points in the simplex. Obviously, this ensures that the $K$-dimensional sphere with center point $\boldsymbol{p}$ and radius $R$ contains the true simplex. For the maximum distance between any two points inside $\mathcal{S}\in\mathbb{S}_K$, denoted by $d_{\mathcal{S}}$, we have
\begin{align}
d_{\mathcal{S}} = &\max_{\boldsymbol{x},\boldsymbol{y} \in \mathcal{S}}{\|\boldsymbol{x}-\boldsymbol{y}\|_2} 
\nonumber \\ 
=& \max_{\boldsymbol{\phi}_{\boldsymbol{x}}, \boldsymbol{\phi}_{\boldsymbol{y}} \in \mathcal{S}_\mathrm{S}^K}{\|\boldsymbol{V}_{\mathcal{S}}\boldsymbol{\phi}_{\boldsymbol{x}} - \boldsymbol{V}_{\mathcal{S}}\boldsymbol{\phi}_{\boldsymbol{y}} \|_2}
\nonumber \\ 
=& \max_{\boldsymbol{\phi}_{\boldsymbol{x}}, \boldsymbol{\phi}_{\boldsymbol{y}} \in \mathcal{S}_\mathrm{S}^K}{\|\boldsymbol{V}_{\mathcal{S}}\left(\boldsymbol{\phi}_{\boldsymbol{x}} -\boldsymbol{\phi}_{\boldsymbol{y}}\right) \|_2},
\end{align}
where $\mathcal{S}_{\mathrm{S}}^K$ represents the standard simplex in $\mathbb{R}^K$ and $\boldsymbol{V}_{\mathcal{S}}$ denotes the vertex matrix of $\mathcal{S}$ (Equation 4 in the main paper). Let $\sigma_{\max}\left(\mathcal{S}\right)$ denote the maximum singular-value for the matrix $\boldsymbol{V}_{\mathcal{S}}$, and assume $\boldsymbol{v}_{\max}$ be the unitary eigenvector that corresponds to $\sigma_{\max}\left(\mathcal{S}\right)$. To be more precise, let $\boldsymbol{V}_{\mathcal{S}}=\boldsymbol{U}\boldsymbol{\Sigma}\boldsymbol{V}^T$ be the singular value decomposition of the vertex matrix. Then, $\sigma_{\max}\left(\mathcal{S}\right)$ refers to the largest singular value on the main diagonal of $\boldsymbol{\Sigma}$, and $\boldsymbol{v}_{\max}$ represents the corresponding column in $\boldsymbol{V}$. Then, for any point $\boldsymbol{x}\in\mathbb{R}^{K+1}$, one can write
$$
\left\Vert
\boldsymbol{V}_{\mathcal{S}}\boldsymbol{x}
\right\Vert^2_2
=\sum_{i}\Sigma^2_{i,i}\left\Vert
\boldsymbol{U}_i\boldsymbol{V}_i^T\boldsymbol{x}
\right\Vert^2_2
=\sum_{i}\Sigma^2_{i,i}\left(\boldsymbol{V}_i^T\boldsymbol{x}\right)^2,
$$
which results in the following useful inequality:
$$
\left\Vert
\boldsymbol{V}_{\mathcal{S}}\boldsymbol{x}
\right\Vert_2
\ge
\sigma_{\max}\left(\mathcal{S}\right)
\left\vert
\boldsymbol{v}_{\max}^T\boldsymbol{x}
\right\vert.
$$
For $\boldsymbol{\phi}_{\boldsymbol{x}}\neq\boldsymbol{\phi}_{\boldsymbol{y}}$, let $\boldsymbol{r}\triangleq\left(\boldsymbol{\phi}_{\boldsymbol{x}}-\boldsymbol{\phi}_{\boldsymbol{y}}\right)/\left\Vert\boldsymbol{\phi}_{\boldsymbol{x}}-\boldsymbol{\phi}_{\boldsymbol{y}}\right\Vert_2$, and $\ell\triangleq\left\Vert\boldsymbol{\phi}_{\boldsymbol{x}}-\boldsymbol{\phi}_{\boldsymbol{y}}\right\Vert_2$. In other words, let $\boldsymbol{r}$ represent the unitary direction vector for the difference, and $\ell$ to denote the corresponding length. In this regard, we have
\begin{align}
\max_{\boldsymbol{\phi}_{\boldsymbol{x}},\boldsymbol{\phi}_{\boldsymbol{y}}\in\mathcal{S}^K_\mathrm{S}}
\left\Vert
\boldsymbol{V}_{\mathcal{S}}\left(
\boldsymbol{\phi}_{\boldsymbol{x}}-\boldsymbol{\phi}_{\boldsymbol{y}}
\right)
\right\Vert_2
\ge~
\sigma_{\max}\left(\mathcal{S}\right)
\max_{\boldsymbol{r},\ell}~
\ell
\left\vert
\boldsymbol{v}_{\max}^T\boldsymbol{r}
\right\vert
\ge~
\sigma_{\max}\left(\mathcal{S}\right)
\max_{\ell\vert~\boldsymbol{r}=\boldsymbol{v}_{\max}}~
\ell.
\end{align}
Now we should find a lower-bound for maximum possible $\ell$ that can be achieved for some fixed direction $\boldsymbol{r}=\boldsymbol{v}_{\max}$. To this aim, assume the difference vector $\boldsymbol{\phi}_{\boldsymbol{x}}-\boldsymbol{\phi}_{\boldsymbol{y}}$ starts from any arbitrary vertex, i.e., $\boldsymbol{\phi}_{\boldsymbol{x}}$ is a one-hot vector with $K$ components equal to zero and the remaining one equal to $1$. In this regard, the minimum possible $\ell$ that can be achieved (minimum is taken w.r.t. direction of the difference vector) occurs when the difference vector becomes perpendicular to its front facet. Since all $\boldsymbol{\phi}$ vectors belong to the standard simplex, such difference vectors can be easily shown to be of the form $\boldsymbol{\Delta}$ with
$$
\Delta_i=1\quad\mathrm{and}\quad\Delta_{j\vert j\neq i}=-1/K,
$$
for any $i\in0,\ldots,K$. The $\ell_2$-norm of all such vectors equals to $\sqrt{1+1/K}$. Thus, we have
\begin{align}
d_{\mathcal{S}}\ge~\sqrt{\frac{K+1}{K}}\sigma_{\mathrm{max}}\left(S\right)
\ge~\sigma_{\mathrm{max}}\left(S\right).
\end{align}
Assume we have $2m$ i.i.d. samples drawn from a noisy simplex $f_{\mathcal{S}}*\mathbb{G}_{\sigma}$, denoted by $\left\{\boldsymbol{y}\right\}^{2m}_{i=1}$. This way, we have $\boldsymbol{y}_i=\boldsymbol{V}_{\mathcal{S}}\boldsymbol{\phi}_i+\boldsymbol{z}_i$ for $i=1,\ldots,2m$, where $\boldsymbol{\phi}_i$ represents the weight vector for the $i$th sample (drawn from a uniform Dirichlet distribution), while $\boldsymbol{z}_i\sim\mathcal{N}\left(\boldsymbol{0},\sigma^2\boldsymbol{I}\right)$ is an independently drawn noise vector. Let us define $\mathrm{D}$ as
\begin{equation}
\mathrm{D}= \frac{1}{2m}\sum_{i =1}^{m}{\|\boldsymbol{y}_{2i} - \boldsymbol{y}_{2i-1}\|_2^2}.
\end{equation}
To find an upper-bound for $d_{\mathcal{S}}$, it is enough to find a high probability lower-bound for $\mathrm{D}$ in terms of $d_{\mathcal{S}}$. To this aim, we rewrite $\mathrm{D}$ as
\begin{align}
\mathrm{D} =~  & \frac{1}{2m}\sum_{i =1}^{m}{\|\boldsymbol{V}_{\mathcal{S}}\boldsymbol{\phi}_{2i} - \boldsymbol{V}_{\mathcal{S}}\boldsymbol{\phi}_{2i-1} + \boldsymbol{z}_{2i} - \boldsymbol{z}_{2i-1} \|_2^2}
\nonumber \\  
=~ & \frac{1}{2m}\sum_{i =1}^{m}{\|\boldsymbol{V}_{\mathcal{S}}\boldsymbol{\phi}_{2i} - \boldsymbol{V}_{\mathcal{S}}\boldsymbol{\phi}_{2i-1}\|_2^2} + \frac{1}{2m}\sum_{i=1}^{m}{\|\boldsymbol{z}_{2i} - \boldsymbol{z}_{2i-1} \|_2^2} 
\nonumber \\
& + \frac{1}{m}\sum_{i=1}^{m}{\left(\boldsymbol{V}_{\mathcal{S}}\boldsymbol{\phi}_{2i} - \boldsymbol{V}_{\mathcal{S}}\boldsymbol{\phi}_{2i-1}\right)^T\left(\boldsymbol{z}_{2i} - \boldsymbol{z}_{2i-1}\right)}
\nonumber \\  
= & ~ \frac{1}{2m}\sum_{i =1}^{m}{\|\boldsymbol{V}_{\mathcal{S}}\boldsymbol{\phi}_{2i} - \boldsymbol{V}_{\mathcal{S}}\boldsymbol{\phi}_{2i-1}\|_2^2} + \frac{2\sigma^2}{2m}\sum_{i=1}^{m}{\|\tilde{\boldsymbol{z}}_{2i} -  \tilde{\boldsymbol{z}}_{2i-1} \|_2^2}
\nonumber \\ 
& + \frac{\sigma\sqrt{2}}{m}\sum_{i=1}^{m}{\left(\boldsymbol{V}_{\mathcal{S}}\boldsymbol{\phi}_{2i} - \boldsymbol{V}_{\mathcal{S}}\boldsymbol{\phi}_{2i-1}\right)^T\left(\tilde{\boldsymbol{z}}_{2i} - \tilde{\boldsymbol{z}}_{2i-1}\right)}.
\label{Lower bound for statistic}
\end{align}
Let us denote the three terms in r.h.s. of \eqref{Lower bound for statistic} as $\mathrm{I}$, $\mathrm{II}$ and $\mathrm{III}$, respectively. Also, we have $\tilde{\boldsymbol{z}}_{i} = \boldsymbol{z}_i/\sigma$. To find a lower-bound for $\mathrm{D}$, we should find respective lower-bounds for $\mathrm{I}$, $\mathrm{II}$ and $\mathrm{III}$. First, let us discuss about $(\mathrm{I})$:
\begin{equation}
\mathrm{I} = \frac{1}{2m}\sum_{i =1}^{m}{\left\Vert
\boldsymbol{V}_{\mathcal{S}}\boldsymbol{\phi}_{2i} - \boldsymbol{V}_{\mathcal{S}}\boldsymbol{\phi}_{2i-1}
\right\Vert_2^2}  =  f_{\boldsymbol{V}_{\mathcal{S}}}\left(\boldsymbol{\phi}_1,\boldsymbol{\phi}_2,\cdots, \boldsymbol{\phi}_{2m}\right).
\end{equation}
For the term inside of the summation, we have
\begin{align}
\|\boldsymbol{V}_{\mathcal{S}}\boldsymbol{\phi}_{2i} - \boldsymbol{V}_{\mathcal{S}}\boldsymbol{\phi}_{2i-1}\|_2^2 
~\stackrel{a.s.}{\leq}~ & d_{\mathcal{S}}^2,
\\
\mathrm{Var}\left(\|\boldsymbol{V}_{\mathcal{S}}\boldsymbol{\phi}_{2i} - \boldsymbol{V}_{\mathcal{S}}\boldsymbol{\phi}_{2i-1}\|_2^2 \right)
\leq &~ \frac{12\sigma^4_{\mathrm{max}}\left(\mathcal{S}\right)}{K^3} \leq ~\frac{12d_{\mathcal{S}}^4}{K^3}
\label{Bounded difference property}
\end{align}
From the above inequalities we can see that $\|\boldsymbol{V}_{\mathcal{S}}\boldsymbol{\phi}_{2i} - \boldsymbol{V}_{\mathcal{S}}\boldsymbol{\phi}_{2i-1}\|_2^2 $ satisfies the Bernstein's condition with $b = 2d^2_{\mathcal{S}}$ (see eq. $2.15$ in \citep{wainwright2019high}). Then (from proposition $2.10$ in  \citep{wainwright2019high}) the following inequality holds with probability at least $1-\delta/6$, for any $\delta\in\left(0,1\right)$:
\begin{equation}
f_{\boldsymbol{V}_{\mathcal{S}}}\left(\boldsymbol{\phi}_1,\cdots, \boldsymbol{\phi}_{2m}\right) \geq \mathbb{E}\left[f_{\boldsymbol{V}_{\mathcal{S}}}\left(\boldsymbol{\phi}_1,\cdots, \boldsymbol{\phi}_{2m}\right)\right] - 2\sqrt{ \max\left\{\frac{4\log{\frac{6}{\delta}}}{m}, \frac{3}{K^3}\right\}\cdot \frac{d_{\mathcal{S}}^4}{m}\log{\frac{6}{\delta}}},
\label{Macdiarmid}
\end{equation}
where for the expected value of the function $f$ we have:
\begin{align}
\mathbb{E}\left[ f_{\boldsymbol{V}_{\mathcal{S}}}\left(\boldsymbol{\phi}_1,\cdots, \boldsymbol{\phi}_{2m}\right)\right] ~
= & ~ \frac{1}{2m}\sum_{i =1}^{m}{ \mathbb{E}\left[\|\boldsymbol{V}_{\mathcal{S}}\boldsymbol{\phi}_{2i} - \boldsymbol{V}_{\mathcal{S}}\boldsymbol{\phi}_{2i-1}\|_2^2\right]}
\nonumber \\ 
= & ~ \frac{1}{2} \mathbb{E}\left[\|\boldsymbol{V}_{\mathcal{S}}\boldsymbol{\phi} - \boldsymbol{V}_{\mathcal{S}}\boldsymbol{\phi}^\prime\|_2^2\right]
\nonumber \\  
= & ~ \frac{1}{2}\mathbb{E}\left[\left(\boldsymbol{V}_{\mathcal{S}}\boldsymbol{\phi} - \boldsymbol{V}_{\mathcal{S}}\boldsymbol{\phi}^\prime\right) ^T \left(\boldsymbol{V}_{\mathcal{S}}\boldsymbol{\phi} - \boldsymbol{V}_{\mathcal{S}}\boldsymbol{\phi}^\prime\right)\right]
\nonumber \\
= & ~ \frac{1}{2}\mathbb{E}\left[\left(\boldsymbol{\phi} -\boldsymbol{\phi}^\prime\right) ^T\boldsymbol{V}_{\mathcal{S}}^T\boldsymbol{V}_{\mathcal{S}} \left(\boldsymbol{\phi} - \boldsymbol{\phi}^\prime\right)\right]
\nonumber \\
= & ~ \frac{1}{2}\mathbb{E}\left[\Tr{\left(\boldsymbol{V}_{\mathcal{S}}^T\boldsymbol{V}_{\mathcal{S}} \left(\boldsymbol{\phi} - \boldsymbol{\phi}^\prime\right)\left(\boldsymbol{\phi} -\boldsymbol{\phi}^\prime\right) ^T\right)}\right]
\nonumber \\
= & ~ \frac{1}{2}\Tr{\left(\boldsymbol{V}_{\mathcal{S}}^T\boldsymbol{V}_{\mathcal{S}} \mathbb{E}\left[\left(\boldsymbol{\phi} - \boldsymbol{\phi}^\prime\right)\left(\boldsymbol{\phi} -\boldsymbol{\phi}^\prime\right)^T\right]\right)}
\nonumber \\
= & ~ \frac{1}{2}\Tr{\left(\boldsymbol{V}_{\mathcal{S}}^T\boldsymbol{V}_{\mathcal{S}}\frac{1}{(K+1)^2(K+2)}
  \begin{bmatrix}
    K & -1  & \dots  & -1 \\
    -1 & K  & \dots  & -1 \\
    \vdots  & \vdots  & \ddots & \vdots \\
    -1 & -1 & \dots  & K
\end{bmatrix}
\right)}
\nonumber \\
= & ~ \frac{1}{2(K+1)^2(K+2)}\Tr{\left(\boldsymbol{V}_{\mathcal{S}}^T\boldsymbol{V}_{\mathcal{S}}\left((K+1)\mathrm{\boldsymbol{I}} - 
  \boldsymbol{1}  \boldsymbol{1}^T\right)\right)}
\nonumber \\
= & ~ \frac{1}{2(K+1)^2(K+2)}\left((K+1)\Tr{\left(\boldsymbol{V}_{\mathcal{S}}^T\boldsymbol{V}_{\mathcal{S}}\right)} - \Tr{\left(
  \boldsymbol{1}^T\boldsymbol{V}_{\mathcal{S}}^T \boldsymbol{V}_{\mathcal{S}} \boldsymbol{1}  \right)}\right)
\nonumber \\
= & ~ \frac{1}{2(K+2)}\left(\sum_{i=1}^{K}{\frac{1}{K+1}\sum_{j=1}^{K+1}{\boldsymbol{V}_{\mathcal{S}}^{2}(i,j)}}- \sum_{i=1}^{K}{\left(\frac{1}{K+1}\sum_{j=1}^{K+1}{\boldsymbol{V}_{\mathcal{S}}(i,j)}\right)^2}\right)
\nonumber \\
= & ~ \frac{1}{2(K+2)}\sum_{i=1}^{K}{\frac{1}{K+1}\sum_{j=1}^{K+1}{\left(\boldsymbol{V}_{\mathcal{S}}(i,j) - \frac{1}{K+1}\sum_{j=1}^{K+1}{\boldsymbol{V}_{\mathcal{S}}(i,j)}\right)^2}}
\nonumber \\
= & ~ \frac{1}{2(K+2)}\sum_{i=1}^{K}{\frac{1}{K+1}\sum_{j=1}^{K+1}{\left(\widehat{\boldsymbol{V}}_{\mathcal{S}}(i,j) - \frac{1}{K+1}\sum_{j=1}^{K+1}{\widehat{\boldsymbol{V}}_{\mathcal{S}}(i,j)}\right)^2}}
\nonumber \\ 
= & ~ \frac{1}{2(K+2)}\sum_{i=1}^{K}{\frac{1}{K+1}\sum_{j=1}^{K+1}{\widetilde{\boldsymbol{V}}_{\mathcal{S}}(i,j)^2}}
\nonumber \\ 
= & ~ \frac{1}{2(K+2)(K+1)}\sum_{j=1}^{K+1}{\|\widetilde{\boldsymbol{V}}_{\mathcal{S}}^j\|_2^2}
\label{lowerboundForExpectation}
\end{align}
where $\widetilde{{\boldsymbol{V}}}_{\mathcal{S}}$ and $\widehat{{\boldsymbol{V}}}_{\mathcal{S}}$ are defined as
\begin{align}
\boldsymbol{V}_{\mathcal{S}} = &\left[\boldsymbol{v}_1,\boldsymbol{v}_2,\cdots,\boldsymbol{v}_{K+1}\right]
\\
\widehat{\boldsymbol{V}}_{\mathcal{S}} = &\left[\boldsymbol{0}, \boldsymbol{v}_2 - \boldsymbol{v}_1 ,\boldsymbol{v}_3 - \boldsymbol{v}_1 ,\cdots,\boldsymbol{v}_{K+1} - \boldsymbol{v}_1\right]
\nonumber
\\
= & \left[\hat{\boldsymbol{v}}_1,\hat{\boldsymbol{v}_2},\cdots,\hat{\boldsymbol{v}}_{K+1}\right]
\\
\widetilde{\boldsymbol{V}}_{\mathcal{S}} = &\left[\hat{\boldsymbol{v}}_1 - \frac{\sum_1^{K}{\hat{\boldsymbol{v}}_i}}{K+1} , \hat{\boldsymbol{v}}_2 - \frac{\sum_1^{K}{\hat{\boldsymbol{v}}_i}}{K+1} ,\cdots,\hat{\boldsymbol{v}}_K - \frac{\sum_1^{K}{\hat{\boldsymbol{v}}_i}}{K+1} \right]
\nonumber
\\
= & \left[\tilde{\boldsymbol{v}}_1,\tilde{\boldsymbol{v}_2},\cdots,\tilde{\boldsymbol{v}}_K,\tilde{\boldsymbol{v}}_{K+1}\right],
\end{align}
And $\widetilde{\boldsymbol{V}}_{\mathcal{S}}^j$ is the $j$th column of $\widetilde{\boldsymbol{V}}_{\mathcal{S}}$.
We know that the maximum distance between any two points inside a simplex is equal to the length of the biggest edge of the simplex. Now without loss of generality suppose that $\boldsymbol{v}_2$ and $\boldsymbol{v}_3$ are the vertices related to this edge (if more than one edge have the maximum length, suppose that, $\boldsymbol{v}_2$ and $\boldsymbol{v}_3$ are belong to one of them). Therefore the length of the vector $\boldsymbol{v}_3 - \boldsymbol{v}_2$ associated with the maximum distance satisfies the following
\begin{align}
    d_{\mathcal{S}}^2 ~= & \|\boldsymbol{v}_3 - \boldsymbol{v}_2\|_2^2
    \nonumber \\
    = & \|\tilde{\boldsymbol{v}_3} - \tilde{\boldsymbol{v}_2}\|_2^2
    \nonumber \\
    \leq & \|\tilde{\boldsymbol{v}_3}\|_2^2 + \|\tilde{\boldsymbol{v}_2}\|_2^2 + 2\|\tilde{\boldsymbol{v}_3}\|_2\|\tilde{\boldsymbol{v}_2}\|_2
    \nonumber \\
    \leq & 2\left(\|\tilde{\boldsymbol{v}}_3\|_2^2 + \|\tilde{\boldsymbol{v}}_2\|_2^2\right).
    \label{maximumLength}
\end{align}
Now the following inequality can be deduced from \ref{lowerboundForExpectation} and \ref{maximumLength}
\begin{equation}
    \mathbb{E}\left[ f_{\boldsymbol{V}_{\mathcal{S}}}\left(\boldsymbol{\phi}_1,\cdots, \boldsymbol{\phi}_{2m}\right)\right] \geq \frac{1}{4(K+1)(K+2)}d_{\mathcal{S}}^2.
    \label{lowerboundForExpectation2}
\end{equation}
From \ref{Macdiarmid} and \ref{lowerboundForExpectation2}, it is easy to show that if we have $2m \geq 2000(K+1)(K+2)\log{\frac{6}{\delta}}$ i.i.d. samples from $\mathbb{P}_{\mathcal{S}}$, then with probability at least $1-\frac{\delta}{6}$ (for any $0<\delta\leq 1$), we have
\begin{equation}
f\left(\boldsymbol{\phi}_1,\boldsymbol{\phi}_2,\cdots, \boldsymbol{\phi}_{2m}\right) \geq \frac{d_{\mathcal{S}}^2}{8(K+1)(K+2)}.
\label{Lowebound on emperical diameter}
\end{equation}
Next, we should find a lower-bound for the second term in the r.h.s. of \ref{Lower bound for statistic}, i.e., $\mathrm{II}$.

In \ref{Lower bound for statistic}, $\boldsymbol{z}_i$s are drawn from a multivariate Gaussian distribution with mean vector $\boldsymbol{0}$ and covariance matrix $\sigma^2 \boldsymbol{\mathrm{I}}$. Then, $\boldsymbol{z}^\prime_i = \tilde{\boldsymbol{z}}_{2i} - \tilde{\boldsymbol{z}}_{2i-1}$ is also a zero-mean Gaussian random vector with covariance matrix $\boldsymbol{\mathrm{I}}$, and as a result $\|\boldsymbol{z}^\prime_i\|_2^2$ is a Chi-squared random variable with $K$ degrees of freedom.  We know that a Chi-squared random variable with $K$ degrees of freedom is sub-exponential with parameters $\left(2\sqrt{K}, 4\right)$ and therefore $\frac{1}{2m}\sum_{i=1}^{m}{\|\boldsymbol{z}^\prime_i\|_2^2}$ is  a sub-exponential random variable with parameters $\left(\sqrt{{K}/{m}}, {2}/{m}\right)$. Then the following holds with probability at least $1-\delta/6$:
\begin{align}
\frac{1}{2m}\sum_{i=1}^{m}{\|\tilde{\boldsymbol{z}}_{2i}-\tilde{\boldsymbol{z}}_{2i-1}\|_2^2}~
\geq&~ \frac{1}{2m}\sum_{i=1}^{m}{\mathbb{E}\left[\|\tilde{\boldsymbol{z}}_{2i}-\tilde{\boldsymbol{z}}_{2i-1}\|_2^2\right]} - \sqrt{\frac{2K}{m}\log{\frac{6}{\delta}}}
\nonumber \\
\geq & ~ \frac{K}{2}-\sqrt{\frac{2K}{m}\log{\frac{6}{\delta}}}
\nonumber \\
\geq &~ \frac{K}{2}-\sqrt{\frac{K}{1000(K+1)(K+2)}}
\nonumber \\
\geq &~ \frac{K}{2} - 1 = \frac{K-2}{2}.
\end{align}
Based on the above inequalities, we can give a lower-bound for the second term in the r.h.s. of \ref{Lower bound for statistic}:
\begin{equation}
\mathrm{II} = \frac{2\sigma^2}{2m}\sum_{i=1}^{m}{\|\tilde{\boldsymbol{z}}_{2i} - \tilde{\boldsymbol{z}}_{2i-1} \|_2^2} \geq (K-2)\sigma^2.
\label{upper bound subexponential}
\end{equation}
Now to find a lower bound for $\mathrm{D}\left(\mathrm{S}\right)$, we only need to give a lower bound for $\mathrm{III}$, in the last inequality of \ref{Lower bound for statistic}. For this part we have:
\begin{align}
\mathrm{III} 
~= & ~\frac{\sigma\sqrt{2}}{m}\sum_{i=1}^{m}{\left(\boldsymbol{V}_{\mathcal{S}}\boldsymbol{\phi}_{2i} - \boldsymbol{V}_{\mathcal{S}}\boldsymbol{\phi}_{2i-1}\right)^T\left(\tilde{\boldsymbol{z}}_{2i} - \tilde{\boldsymbol{z}}_{2i-1}\right)}
\nonumber \\
= & ~ \frac{\sigma\sqrt{2}}{m}\sum_{i=1}^{m}{\left(\boldsymbol{V}_{\mathcal{S}}\boldsymbol{\phi}_{2i} - \boldsymbol{V}_{\mathcal{S}}\boldsymbol{\phi}_{2i-1}\right)^T\boldsymbol{z}_i^\prime}
\nonumber \\
= & ~ \frac{\sigma\sqrt{2}}{m}\sum_{i=1}^{m}{g_{\boldsymbol{\phi_{2i}}, \boldsymbol{\phi_{2i-1}}}\left(z_i^{\prime1}, z_i^{\prime2},\cdots, z_i^{\prime m}\right)},
\label{lower bound for 3rd part}
\end{align}
where in the above equations we have:
 \begin{equation}
  \mathbb{E}\left(g_{\boldsymbol{\phi_{2i}}, \boldsymbol{\phi_{2i-1}}}\left(z_i^{\prime1}, z_i^{\prime2},\cdots, z_i^{\prime m}\right)\right) 
   = \mathbb{E}\left(\boldsymbol{V}_{\mathcal{S}}\boldsymbol{\phi}_{2i} - \boldsymbol{V}_{\mathcal{S}}\boldsymbol{\phi}_{2i-1}\right)^T\mathbb{E}\left(\boldsymbol{z}_i^\prime\right) = 0.
 \end{equation}
To find a lower bound for $\mathrm{III}$ we try to bound $g_{\boldsymbol{\phi_{2i}}, \boldsymbol{\phi_{2i-1}}}\left(z_i^{\prime1}, z_i^{\prime2},\cdots, z_i^{\prime m}\right)$ from below. To do so, we write the concentration inequality for this function as follow:
\begin{align}
\mathbb{P}\left(g_{\boldsymbol{\phi_{2i}}, \boldsymbol{\phi_{2i-1}}}\left(z_i^{\prime1}, z_i^{\prime2},\cdots, z_i^{\prime m}\right) \geq \epsilon\right)
\leq & ~ \frac{\mathbb{E}_{\boldsymbol{\phi_{2i}}, \boldsymbol{\phi_{2i-1}}, \boldsymbol{z}_i^\prime}\left[e^{\lambda g_{\boldsymbol{\phi_{2i}}, \boldsymbol{\phi_{2i-1}}}\left(z_i^{\prime1}, z_i^{\prime2},\cdots, z_i^{\prime m}\right) }\right]}{e^{\lambda\epsilon}}
\nonumber \\
\leq & ~ \frac{\mathbb{E}_{\boldsymbol{\phi_{2i}}, \boldsymbol{\phi_{2i-1}}, \boldsymbol{z}_i^\prime}\left[e^{\lambda g_{\boldsymbol{\phi_{2i}}, \boldsymbol{\phi_{2i-1}}}\left(z_i^{\prime1}, z_i^{\prime2},\cdots, z_i^{\prime m}\right) }\right]}{e^{\lambda\epsilon}}
\nonumber \\
= & ~ \frac{\mathbb{E}_{\boldsymbol{\phi_{2i}}, \boldsymbol{\phi_{2i-1}}}\left[\mathbb{E}_{\boldsymbol{z}_i^\prime}\left[e^{\lambda g_{\boldsymbol{\phi_{2i}}, \boldsymbol{\phi_{2i-1}}}\left(z_i^{\prime1}, z_i^{\prime2},\cdots, z_i^{\prime m}\right) }\right]\right]}{e^{\lambda\epsilon}},
\label{concentration for g}
\end{align}
where $g_{\boldsymbol{\phi_{2i}}, \boldsymbol{\phi_{2i-1}}}\left(\boldsymbol{z}\right)$ is a lipschitz function with respect to $\boldsymbol{z}^\prime$:
\begin{align}
\vert g_{\boldsymbol{\phi_{2i}}, \boldsymbol{\phi_{2i-1}}}\left(\boldsymbol{z}_1\right) - g_{\boldsymbol{\phi_{2i}}, \boldsymbol{\phi_{2i-1}}}\left(\boldsymbol{z}_2\right)\vert
\leq& \quad \|\boldsymbol{V}_{\mathcal{S}}\boldsymbol{\phi}_{2i} - \boldsymbol{V}_{\mathcal{S}}\boldsymbol{\phi}_{2i-1}\|_2\|\boldsymbol{z}_1 - \boldsymbol{z}_2\|_2
\nonumber \\
\leq& \quad 2d_\mathcal{S}\|\boldsymbol{z}_1 - \boldsymbol{z}_2\|_2.
\end{align}
From Lemma (2.27) in \citep{wainwright2019high} we know that any $L$-lipschitz function of a Gaussian R.V. is sub-Gaussian. As a result, we have
\begin{align}
\mathbb{E}_{\boldsymbol{z}_i^\prime}\left[e^{\lambda g_{\boldsymbol{\phi_{2i}}, \boldsymbol{\phi_{2i-1}}}\left(z_i^{\prime1}, z_i^{\prime2},\cdots, z_i^{\prime m}\right) }\right]
\leq & \quad e^{\frac{4\pi^2d_{\mathcal{S}}^2\lambda^2}{8}}.
\label{upperbound for mgf of L-lipschitz}
\end{align}
From inequalities \ref{upperbound for mgf of L-lipschitz}, \ref{concentration for g} and equation \ref{lower bound for 3rd part} we have:
\begin{equation}
\mathbb{P}\left(\frac{\sigma\sqrt{2}}{2m}\sum_{i=1}^{m}{g_{\boldsymbol{\phi_{2i}}, \boldsymbol{\phi_{2i-1}}}\left(\boldsymbol{z}_i\right)} \leq -\epsilon\right)
\leq \min_{\lambda \geq 0} \frac{\mathbb{E}_{\boldsymbol{\phi_{1}},\cdots \boldsymbol{\phi_{2m}}}\left[e^{\frac{\lambda^2 \pi^2 d_{\mathcal{S}}^2\sigma^2}{2m}\cdot}\right]}{e^{\lambda\epsilon}}
\leq e^{-\frac{m\epsilon^2}{2\pi^2d_{\mathcal{S}}^2\sigma^2}}.
\end{equation}
Then if we have $2m \geq 2000(K+1)(K+2)\log{\frac{6}{\delta}}$, i.i.d samples, then with probability at least $1-\frac{\delta}{6}$ we have:
\begin{equation}
\frac{1}{2m}\sum_{i=1}^{m}{\left(\boldsymbol{V}_{\mathcal{S}}\boldsymbol{\phi}_{2i} - \boldsymbol{V}_{\mathcal{S}}\boldsymbol{\phi}_{2i-1}\right)^T\left(\boldsymbol{z}_{2i} - \boldsymbol{z}_{2i-1}\right)} \geq -\frac{\pi\sigma d_{\mathcal{S}}}{20\sqrt{(K+2)(K+1)}} .
\label{lower bound for inner product}
\end{equation}
Now from inequalities  \ref{Lowebound on emperical diameter},\ref{upper bound subexponential}, \ref{lower bound for inner product} and equation \ref{Lower bound for statistic} we have: 
\begin{align}
\mathrm{D} 
\geq &  \frac{d_{\mathcal{S}}^2}{8(K+1)(K+2)} + (K-2) \sigma^2 - \frac{\pi\sigma d_{\mathcal{S}}}{20\sqrt{(K+2)(K+1)}}
\nonumber \\
= &  \frac{d_{\mathcal{S}}^2}{8(K+1)(K+2)} + \left(\sigma\sqrt{K-2} - \frac{d_{\mathcal{S}}\sqrt{\pi}}{40\sqrt{\left(K+1\right)\left(K+2\right)\left(K-2\right)}}\right)^2 -
\nonumber \\ &
\frac{d_{\mathcal{S}}^2\pi}{1600\left(K+1\right)\left(K+2\right)\left(K-2\right)} 
\label{upper bound for noise varaince}
\\ 
=& (K-2) \sigma^2 + \left(\frac{\pi\sigma}{10\sqrt{2}} - \frac{d_{\mathcal{S}}}{2\sqrt{2}\sqrt{(K+1)(K+2)}}\right)^2 - \frac{\pi^2\sigma^2}{200} .
\label{upper bound for maximum singular value}
\end{align}
From the inequalities \ref{upper bound for maximum singular value} and \ref{upper bound for noise varaince} it can be seen that if we have $2m \geq 2000(K+1)(K+2)\log{\frac{6}{\delta}}$ then the following inequalities with probability at least $1-\delta/2$ give upper-bounds for the maximum distance between two points in the simplex $\mathcal{S}$, and the noise variance $\sigma$:
\begin{align}
d_{\mathcal{S}} \leq & 4\sqrt{(K+1)(K+2)D} = R.
\label{Radius}
\\
\sigma^2 \leq & \frac{D}{K-3} = R_n
\label{noise radius}
\end{align}

In the same way we could also give an upper bound for $\mathrm{D}$. The following inequalities hold with probability more than $1- \delta/2$:
\begin{align}
\mathrm{D}
\leq &  \frac{d_{\mathrm{S}}^2}{K+2} + \sigma^2(K+2) +\frac{\pi\sigma d_{\mathcal{S}}}{20\sqrt{(K+2)(K+1)}} 
\nonumber \\
\leq & \left(\frac{d_{\mathrm{S}}}{\sqrt{K+2}} + \sigma\sqrt{K+2}\right)^2
\nonumber \\
 = & \frac{d_{\mathrm{S}}^2}{K+2}\left(1+\frac{K+2}{d_{\mathrm{S}}/\sigma}\right)^2
\label{upperbound for empirical diameter}
\\
= & \sigma^2\left(K+2\right)\left(1+\frac{d_{\mathrm{S}}/\sigma}{K+2}\right)^2.
\label{upperbound for empirical noise variance}
\end{align}
Based on equation \ref{Radius} and inequations in \ref{upperbound for empirical diameter} and \ref{upperbound for empirical noise variance}, we can give an upper bound for the radius $R$ in \ref{Radius} and noise radius $R_n$. The following inequality holds with probability at least $1-\delta$:
\begin{align}
R \leq & 4\sqrt{K+1}\left(1+\frac{K+2}{d_{\mathrm{S}}/\sigma}\right)d_{\mathrm{S}}
\label{final radius of containing sphere}
\\
R_n \leq & \frac{K+2}{K-3}\left(1+\frac{d_{\mathrm{S}}/\sigma}{K+2}\right)\sigma.
\label{final upperbound of noise variance}
\end{align}

Now the only thing we should do is to find a point inside of the simplex. To do this we define a new statistic $\mathrm{\boldsymbol{p}}$ as follows:
\begin{align}
\mathrm{\boldsymbol{p}} 
&= \frac{1}{2m}\sum_{i=1}^{2m}{\boldsymbol{x}_i} 
= \frac{1}{2m}\sum_{i=1}^{2m}{\boldsymbol{V}_{\mathcal{S}}\boldsymbol{\phi}_{i}+\boldsymbol{z}_i} 
= \boldsymbol{V}_{\mathcal{S}}\left( \frac{1}{2m}\sum_{i=1}^{2m}{\boldsymbol{\phi}_{i}}\right)+\frac{1}{2m}\sum_{i=1}^{2m}{\boldsymbol{z}_i} 
 = \mathrm{I} + \mathrm{II}
\label{center point statistic}
\end{align}
It is clear that `'I'' is placed inside of the true simplex. So the distance between $\mathrm{\boldsymbol{p}} $ and a point inside the main simplex can be calculated as follows:
\begin{align}
\min_{\boldsymbol{x} \in \mathcal{S}}{\left\Vert\mathrm{\boldsymbol{p}} -\boldsymbol{x}\right\Vert_2}
& \leq \left\Vert \frac{1}{2m}\sum_{i=1}^{2m}{\boldsymbol{z}_i} \right\Vert_2.
\end{align}

In the above equation $\frac{1}{2m}\sum_{i=1}^{2m}{\boldsymbol{z}_i}$, is a zero mean Gaussian random vector with covariance matrix $\frac{\sigma^2}{2m}\boldsymbol{I}$, therefore $ \| \frac{1}{2m}\sum_{i=1}^{2m}{\boldsymbol{z}_i} \|_2^2$ is a chi-squared random variable with $K$ degrees of freedom, and so we can write concentration inequality for it:
\begin{align}
\mathrm{P}\left[\left\vert \left\Vert \frac{1}{2m}\sum_{i=1}^{2m}{\boldsymbol{z}_i} \right\Vert_2^2 - \frac{K\sigma^2}{2m}\right\vert \geq \epsilon \right] \leq 2e^{-\frac{m^2\epsilon^2}{2K\sigma^4}}.
\end{align}
Then with probability at least $1-\delta$ we have the following inequality:
\begin{align}
\min_{\boldsymbol{x} \in \mathcal{S}}{\|\mathrm{\boldsymbol{p}} -\boldsymbol{x}\|_2^2}
\leq \frac{K\sigma^2}{2m} + \frac{\sigma^2}{m}\sqrt{2K\log{\frac{2}{\delta}}}
\leq \frac{K\sigma^2}{m}\log{\frac{2}{\delta}} 
\end{align}
Now if we have $2m \geq 2000(K+1)(K+2)\log{\frac{6}{\delta}}$, then we can rewrite the above inequality as follows:
\begin{align}
\min_{\boldsymbol{x} \in \mathcal{S}}{\|\boldsymbol{p} -\boldsymbol{x}\|_2}
\leq \frac{\sigma}{10\sqrt{10}\sqrt{K+2}}.
\end{align}
As we mentioned earlier, the first part in the right hand side of the inequality \ref{center point statistic}  , $\frac{1}{2m}\sum_{i=1}^{2m}{\boldsymbol{V}_{\mathcal{S}}\boldsymbol{\phi}_{i}}$, belongs to the interior of the simplex. Then a $K$-dimensional sphere with radius $R$ and centered at this point, with probability at least $1 - \delta$ contains the simplex $\mathcal{S}$. However, in the process of learning we do not have access to the noiseless data and thus cannot have such a point as the center of the sphere. Let us show the distance between $\boldsymbol{p}$ and $\frac{1}{2m}\sum_{i=1}^{2m}{\boldsymbol{V}_{\mathcal{S}}}$ with $d$. Then, it is clear to see that the $K$-dimensional sphere with radius $R+d$ and center point $\boldsymbol{p}$, contains the sphere with center point at $\frac{1}{2m}\sum_{i=1}^{2m}{\boldsymbol{V}_{\mathcal{S}}\boldsymbol{\phi}_{i}}$ and radius $R$. So any simplex which is placed in sphere $\mathrm{C}^{K}(\frac{1}{2m}\sum_{i=1}^{2m}{\boldsymbol{V}_{\mathcal{S}}\boldsymbol{\phi}_{i}}, R)$ is also placed in sphere $\mathrm{C}^{K}(\boldsymbol{p}, R+d)$.

So we conclude that if we have $2m \geq 2000(K+1)(K+2)\log{\frac{6}{\delta}}$ i.i.d. samples from $\mathbb{G}_{\mathcal{S}}$ then the main simplex $\mathcal{S}$ with probability more than or equal to $1- \delta$ will be confined in a $K$-dimensional sphere with center point at $\boldsymbol{p}$ and with radius $R$:
\begin{align}
R = ~& 4\sqrt{(K+1)(K+2)D} + \frac{\sqrt{D}}{40\sqrt{K+2}} 
\nonumber\\
\leq &~ 4\sqrt{(K+1)(K+2)D}\left(1+\frac{1}{160(K+2)\sqrt{K+1}}\right)
\nonumber \\
\leq &~ 8\sqrt{(K+1)(K+2)D}.
\label{final upperbound}
\end{align}
This completes the proof.
\end{proof}

\vspace*{3mm}
\begin{proof}[proof of Lemma \ref{quantization lemma}]
Consider a $\left(\underline{\theta},\bar{\theta}\right)$-isoperimetric $K$-simplex $\mathcal{S}$, which is bounded in a sphere $\mathrm{C}^{K}(\boldsymbol{p}, R)$. It is clear that all vertices of this simplex $\left\{\boldsymbol{v}_1,\ldots,\boldsymbol{v}_{K+1}\right\}$ are placed in $\mathrm{C}^{K}(\boldsymbol{p}, R)$. From the definition of $\mathrm{T}_{\frac{\alpha\epsilon}{K+1}}(\mathrm{C}^{K}(\boldsymbol{p}, R))$, we know that for each vertex of the simplex $\boldsymbol{v}_i$, there exists some $\boldsymbol{v}_i^{\prime} \in \mathrm{T}_{\frac{\alpha\epsilon}{K+1}}(\mathrm{C}^{K}(\boldsymbol{p}, R))$ such that $\|\boldsymbol{v}_i - \boldsymbol{v}_i^{\prime}\|_2 \leq \frac{\alpha\epsilon}{K+1}$. Assume for each vertex $\boldsymbol{v}_i$ of $\mathcal{S}$, we denote its closest point in $ \mathrm{T}_{\frac{\alpha\epsilon}{K+1}}(\mathrm{C}^{K}(\boldsymbol{p}, R))$ as
\begin{equation*}
\hat{\boldsymbol{v}}_i = \argmin \left\{\|\boldsymbol{v}_i - \hat{\boldsymbol{v}}\|_2  \bigg\vert ~\hat{\boldsymbol{v}} \in  \mathrm{T}_{\frac{\alpha\epsilon}{K+1}}(\mathrm{C}^{K}(\boldsymbol{p}, R))\right\}.
\end{equation*}
Using these points we make a new simplex $\widehat{\mathcal{S}}$. It is clear that $\widehat{\mathcal{S}}$ belongs to $\widehat{\mathbb{S}}\left(\mathrm{C}^{K}(\boldsymbol{p}, R)\right)$.
Assume $f_{\mathcal{S}}$ and $f_{\hat{\mathcal{S}}}$ denote the probability density functions that correspond to $\mathcal{S}$ and $\widehat{\mathcal{S}}$, respectively. Then, the TV-distance between $f_{\mathcal{S}}$ and $f_{\widehat{\mathcal{S}}}$ can be written and bounded as
\begin{align}
\mathrm{TV}\left(\mathbb{P}_{\mathcal{S}}, \mathbb{P}_{\widehat{\mathcal{S}}}\right) 
=&~ \sup_{A}~{\mathbb{P}_{\mathcal{S}}\left(A\right)- \mathbb{P}_{\widehat{\mathcal{S}}}\left(A\right)}
\nonumber \\
=& \int_{\boldsymbol{x} \in \{x^\prime:f_{\mathcal{S}}\left(\boldsymbol{x}^\prime\right) \ge f_{\widehat{\mathcal{S}}}\left(\boldsymbol{x}^\prime\right)\}}{f_{\mathcal{S}}\left(\boldsymbol{x}\right) - f_{\widehat{\mathcal{S}}}\left(\boldsymbol{x}\right)}
    \nonumber \\ 
    =& \int_{\boldsymbol{x} \in \mathcal{S}-\widehat{\mathcal{S}}}{f_{\mathcal{S}}\left(\boldsymbol{x}\right) - f_{\widehat{\mathcal{S}}}\left(\boldsymbol{x}\right)} + \left(\int_{\boldsymbol{x} \in \mathcal{S}\cap\widehat{\mathcal{S}}}{f_{\mathcal{S}}\left(\boldsymbol{x}\right) - f_{\widehat{\mathcal{S}}}\left(\boldsymbol{x}\right)}\right)\boldsymbol{\mathrm{1}}\left(\mathrm{Vol}\left(\mathcal{S}\right) \leq \mathrm{Vol}\left(\widehat{\mathcal{S}}\right)\right)
    \nonumber \\
    \leq& \int_{\boldsymbol{x} \in \mathcal{S}-\widehat{\mathcal{S}}}{f_{\mathcal{S}}\left(\boldsymbol{x}\right) - f_{\widehat{\mathcal{S}}}\left(\boldsymbol{x}\right)} + \int_{\boldsymbol{x} \in \mathcal{S}\cap\widehat{\mathcal{S}}}{\bigg\vert f_{\mathcal{S}}\left(\boldsymbol{x}\right) - f_{\widehat{\mathcal{S}}}\left(\boldsymbol{x}\right)\bigg\vert}
    \nonumber \\
    =& \frac{\mathrm{Vol}\left(\mathcal{S-\widehat{S}}\right)}{\mathrm{Vol}\left(S\right)} + \mathrm{Vol}\left( \mathcal{S} \cap \widehat{\mathcal{S}} \right)\bigg\vert\frac{1}{\mathrm{Vol}(\mathcal{S})} - \frac{1}{\mathrm{Vol}(\widehat{\mathcal{S}})}\bigg\vert,
    \label{upper bound for tv after quantization}
\end{align}
where $\mathcal{S} - \widehat{\mathcal{S}}$ shows the set difference between $\mathcal{S}$ and $\widehat{\mathcal{S}}$. Let us denote the two terms in the r.h.s. of \ref{upper bound for tv after quantization} as $\mathrm{I}$ and $\mathrm{II}$, respectively. To find an upper-bound for the $\mathrm{TV}$-distance in \ref{upper bound for tv after quantization} we  find respective upper-bounds for $\mathrm{I}$ and $\mathrm{II}$. In order to do so, first let us discuss about $\mathrm{I}$:
\begin{align}
\mathrm{I} = \frac{\mathrm{Vol}\left(\mathcal{S-\widehat{S}}\right)}{\mathrm{Vol}\left(S\right)}. 
\end{align}
The distance between vertices of $\mathcal{S}$ and $\widehat{\mathcal{S}}$ are less than $\frac{\alpha\epsilon}{K+1}$, then the maximal difference set (in terms of volume) between $\mathcal{S}$ and $\widehat{\mathcal{S}}$ can be bounded as follows: simplex has $K+1$ facets, and the difference set that can occur from altering each of them is upper-bounded as $\leq \mathcal{A}_i\left(\mathcal{S}\right)\times \alpha\epsilon/(K+1)$, where $i=1,\ldots,K+1$ denotes the facet index. This way, for $\mathrm{I}$ we have:
\begin{align}
\mathrm{I} 
\leq~\sum_{i=1}^{K+1}{\frac{\alpha\epsilon}{K+1}\frac{\mathcal{A}_i\left(\mathcal{S}\right)}{\mathrm{Vol}\left(\mathcal{S}\right)}}
\leq~\alpha\epsilon\mathcal{A}_{max}\left(\mathcal{S}\right)
\leq\alpha\epsilon\bar{\theta}\mathrm{Vol}^{\frac{1}{K}},
\label{upper bound for tv after quantization part I}
\end{align}
where $\mathcal{A}_{max}\left(\mathcal{S}\right)$ is the volume of the largest facet of $\mathcal{S}$. In the final inequality of \ref{upper bound for tv after quantization part I}, we take advantage of the isoperimetricity property of $\mathcal{S}$. Next, we should find an upper-bound for the second term in the r.h.s. of \ref{upper bound for tv after quantization}, i.e., $\mathrm{II}$:
\begin{equation}
\mathrm{II} = \mathrm{Vol}\left( \mathcal{S} \cap \widehat{\mathcal{S}} \right)\bigg\vert\frac{1}{\mathrm{Vol}(\mathcal{S})} - \frac{1}{\mathrm{Vol}(\widehat{\mathcal{S}})}\bigg\vert.
\end{equation}
To find an upper-bound for $\mathrm{II}$ we should find an upper-bound for $\mathrm{Vol}\left(\widehat{\mathcal{S}}\right)$. To do so, we create a ``rounded" simplex $\mathcal{S}^r$ which forms by adding a $K$-dimensional sphere with radius $r = \frac{\alpha\epsilon}{K+1}$ to the original simplex $\mathcal{S}$. It can be shown that $\widehat{\mathcal{S}}$ is definitely placed inside $\mathcal{S}^r$ and therefore $\mathrm{Vol}(\hat{\mathcal{S}}) \leq \mathrm{Vol}(\mathcal{S}^r)$. In this regard, we have
\begin{align}
\mathrm{II} 
=&~ \mathrm{Vol}\left( \mathcal{S} \cap \widehat{\mathcal{S}} \right)\bigg\vert\frac{1}{\mathrm{Vol}(\mathcal{S})} - \frac{1}{\mathrm{Vol}(\hat{\mathcal{S}})}\bigg\vert
\nonumber \\
\leq&~ \bigg\vert\frac{\mathrm{Vol}(\hat{\mathcal{S}}) - \mathrm{Vol}(\mathcal{S})}{\mathrm{Vol}(\mathcal{S})}\bigg\vert
\nonumber \\
\leq&~ \bigg\vert\frac{\mathrm{Vol}(\mathcal{S}^r) - \mathrm{Vol}(\mathcal{S})}{\mathrm{Vol}(\mathcal{S})}\bigg\vert
\nonumber \\
\leq&~ \left(\sum_{i=1}^{K+1}{\frac{\alpha\epsilon}{K+1}\mathcal{A}_i\left(\mathcal{S}\right)} + (K+1)\mathcal{C}_{\alpha}\left(\frac{\alpha\epsilon}{K+1}\right)^K\right)\frac{1}{\mathrm{Vol}(\mathcal{S})}
\nonumber \\
\leq&~ \left((K+1)\frac{\alpha\epsilon}{K+1}\mathcal{A}_{\max}\left(\mathcal{S}\right) + (K+1)\mathcal{C}_{\alpha}\left(\frac{\alpha\epsilon}{K+1}\right)^K\right)\frac{1}{\mathrm{Vol}(\mathcal{S})}
\nonumber \\
\leq&~ 2\alpha\epsilon\bar{\theta}\mathrm{Vol}^{\frac{1}{K}},
\label{upper bound for tv after quantization part II}
\end{align}
where $\mathcal{C}_{\alpha}\left(\frac{\alpha\epsilon}{K+1}\right)^K$ is the volume of a $K$-dimensional sphere with radius $r = \frac{\alpha\epsilon}{K+1}$. Now, using \ref{upper bound for tv after quantization part I} and \ref{upper bound for tv after quantization part II} we have
\begin{align}
\mathrm{TV}\left(\mathbb{P}_{\mathcal{S}}, \mathbb{P}_{\widehat{\mathcal{S}}}\right)
\leq 3\alpha\epsilon\bar{\theta}\mathrm{Vol}^{\frac{1}{K}} \leq \frac{3}{5}\epsilon \leq \epsilon.
\label{upper bound for tv after quantization final}
\end{align}
From \ref{upper bound for tv after quantization final}, it can be seen that for any $\left(\underline{\theta},\bar{\theta}\right)$-isoperimetric simplex $\mathcal{S}$ which is bounded in a $K$-dimensional sphere $\mathrm{C}^{K}(\boldsymbol{p}, R)$, there exists at least one simplex in $\widehat{\mathbb{S}}(\mathrm{C}^{K}(\boldsymbol{p}, R))$ within a TV-distance of $\epsilon$ from $\mathcal{S}$. Thus the proof is complete.

In the end, it is worth mentioning one possible method to build an $\epsilon$-covering set for a $K$-dimensional sphere $\mathrm{C}^{K}(\boldsymbol{p}, R)$. Although there exists deterministic ways to build this set, an alternative approach is by uniformly sampling points from the sphere. Consider an $\epsilon/2$-packing set with size $L$ for the sphere. Based on the result of the ``coupon collector problem" we know that having $L\log(L)$ uniform samples from the sphere guarantees with high probability that there exists at least one sample within a distance of $\epsilon/2$ from each element of the packing set. Therefore, the resulting samples are an $\epsilon$-covering set for the sphere. Hence, the cardinality of an $\epsilon$-covering set built in this way is at most $\left(1+\frac{4R}{\epsilon}\right)^{2K}$.
\end{proof}


\section{Recoverability From Noise}
\label{sec:app:noise}

Assume $\mathcal{S}_1,\mathcal{S}_2\in\mathbb{S}_K$ represent two arbitrary simplices in $\mathbb{R}^K$. In this regard, let $\mathbb{P}_{\mathcal{S}_1}$ and $\mathbb{P}_{\mathcal{S}_2}$ denote the probability measures, and $f_{\mathcal{S}_1}$ and $f_{\mathcal{S}_2}$ represent the probability density functions associated to $\mathcal{S}_1$ and $\mathcal{S}_2$, respectively. Let us assume that $\mathcal{S}_1,\mathcal{S}_2$ have a minimum degree of geometric regularity in the following sense.

\begin{definition}
For a simplex $\mathcal{S}\in\mathbb{S}_K$ with vertices $\boldsymbol{\theta}_0,\ldots,\boldsymbol{\theta}_{K}\in\mathbb{R}^K$, we say $\mathcal{S}$ is $\left(\bar{\lambda},\underline{\lambda}\right)$-regular if
$$
\underline{\lambda}
\leq
\lambda_{\min}\left(\boldsymbol{\Theta}\right)
\leq
\lambda_{\max}\left(\boldsymbol{\Theta}\right)
\leq
\bar{\lambda},
$$
where $\lambda_{\max}\left(\cdot\right)$ and $\lambda_{\min}\left(\cdot\right)$ denotes the largest and smallest eigenvalues of a matrix, respectively. Here, $\boldsymbol{\Theta}$ represents the zero-centerd vertex matrix of $\mathcal{S}$, i.e.,
$$
\boldsymbol{\Theta}\triangleq
\left[
\boldsymbol{\theta}_1-\boldsymbol{\theta}_0
\vert
\cdots
\vert
\boldsymbol{\theta}_K-\boldsymbol{\theta}_0
\right].
$$
\end{definition}
This definition has tight connections to the previously used notion of $\left(\bar{\theta},\underline{\theta}\right)$-isoperimetricity which is already used in the main body of the manuscript. In fact, it can be shown that $\mathcal{L}_{\max}=\mathcal{O}\left(\Bar{\lambda}\right)$ and vice versa.


Our aim is to show that if the noisy versions of $\mathcal{S}_1$ and $\mathcal{S}_2$, which we denote by $f_{\mathcal{S}_1}*G_{\sigma}$ and $f_{\mathcal{S}_2}*G_{\sigma}$, respectively, have a maximum total variation distance of at least $\epsilon>0$, then the TV distance between $\mathbb{P}_{\mathcal{S}_1}$ and $\mathbb{P}_{\mathcal{S}_2}$ is also bounded away from zero according to a function of $\epsilon,\sigma$ and the geometric reqularity of simplices $\mathcal{S}_1$ and $\mathcal{S}_2$. The theoretical core behind our method is stated in the following general theorem.

\begin{theorem}[Recovery of Low-Frequency Objects from Additive Noise]
\label{thm:generalResultTheorem-Appendix}
For $K\in\mathbb{N}$, consider a probability density function family $\mathscr{F}\subseteq\mathcal{M}\left(\mathbb{R}^K\right)$, i.e., a subset of distributions supported over $\mathbb{R}^K$. Assume for sufficiently large $\alpha>0$, the following bound holds for all $f,g\in\mathscr{F}$:
\begin{equation*}
\frac{1}{\left(2\pi\right)^K}
\int_{\left\Vert\boldsymbol{\omega}\right\Vert_{\infty}\ge\alpha}
\left\vert
\mathcal{F}\left\{f\right\}\left(\boldsymbol{\omega}\right)
-
\mathcal{F}\left\{g\right\}\left(\boldsymbol{\omega}\right)
\right\vert^2
\leq
\zeta\left(\alpha^{-1}\right)
\int_{\mathbb{R}^K}
\left\vert
f-g
\right\vert^2,
\end{equation*}
where $\mathcal{F}\left\{\cdot\right\}$ denotes the Fourier transform, and $\zeta$ is an increasing fuction with $\zeta\left(0\right)=0$ and continuity at $0$. Also, assume the probability density function $Q\in\mathcal{M}\left(\mathbb{R}^K\right)$ again for a sufficiently large $\alpha>0$ has the following property:
\begin{equation*}
\inf_{\left\Vert\boldsymbol{\omega}\right\Vert_{\infty}\leq\alpha}
\left\Vert
\mathcal{F}\left\{Q\right\}\left(\boldsymbol{\omega}\right)
\right\Vert
\ge \eta\left(\alpha\right),
\end{equation*}
where $\eta\left(\cdot\right)$ is a non-negative decreasing function. Then, there exists a non-negative constant $C$ where for any $\sigma,\varepsilon>0$ and $f,g\in\mathscr{F}$ with
$
\left\Vert
f-g
\right\Vert_2
\geq\varepsilon,
$
we have
$$
\left\Vert
\left(f-g\right)*Q
\right\Vert_2
\geq
\frac{\varepsilon}{\left(2\pi\right)^K}
\left(
\sup_{\alpha>C}~
\eta\left(\alpha\right)
\sqrt{1-\zeta\left(\alpha^{-1}\right)}
\right),
$$
with $*$ denoting the multi-dimensional convolution operator.
\end{theorem}
\begin{proof}
For the sake of simplicity in notations, let $\mathcal{F},\mathcal{G},\mathcal{Q}:\mathbb{R}^K\rightarrow\mathbb{C}$ denote the Fourier transforms of $f,g$ and $Q$, respectively. Due to Parseval's theorem, we have
$$
\left\Vert
f-g
\right\Vert^2_2
=
\frac{1}{\left(2\pi\right)^K}
\left\Vert
\mathcal{F}-\mathcal{G}
\right\Vert^2_2.
$$
Also, due to the properties of the Fourier transform, which is the transofrmation of convolution into direct multiplication, one can write
$$
\mathcal{F}\left\{
\left(f-g\right)*Q
\right\}
=
\mathcal{Q}\left(\mathcal{F}-\mathcal{G}\right).
$$
Thus, there exists universal constant $C>0$ such that for any $\alpha>C$:
\begin{align}
\left(2\pi\right)^K
\left\Vert
\left(f-g\right)*Q
\right\Vert^2_2
&=
\int_{\mathbb{R}^K}
\left\vert
\mathcal{Q}
\left(\boldsymbol{\omega}\right)
\left(
\mathcal{F}\left(\boldsymbol{\omega}\right)
-
\mathcal{G}\left(\boldsymbol{\omega}\right)
\right)
\right\vert^2
\\
&\ge
\int_{\left\Vert\boldsymbol{\omega}\right\Vert_{\infty}\leq\alpha}
\left\vert
\mathcal{Q}
\left(\boldsymbol{\omega}\right)
\right\vert^2
\left\vert
\mathcal{F}\left(\boldsymbol{\omega}\right)
-
\mathcal{G}\left(\boldsymbol{\omega}\right)
\right\vert^2
\nonumber\\
&\ge
\eta^2\left(\alpha\right)
\int_{\left\Vert\boldsymbol{\omega}\right\Vert_{\infty}\leq\alpha}
\left\vert
\mathcal{F}\left(\boldsymbol{\omega}\right)
-
\mathcal{G}\left(\boldsymbol{\omega}\right)
\right\vert^2
\nonumber\\
&\geq
\varepsilon^2
\eta^2\left(\alpha\right)
\left[
1-\zeta\left(\alpha^{-1}\right)
\right].
\nonumber
\end{align}
The above chain of inequalities hold for all $\alpha>C$, therefore we have:
\begin{align}
\left\Vert
\left(f-g\right)*Q
\right\Vert^2_2
&\ge
\frac{\varepsilon}{\left(2\pi\right)^K}
\sup_{\alpha>C}~
\eta\left(\alpha\right)
\sqrt{1-\zeta\left(\alpha^{-1}\right)},
\end{align}
which completes the proof.
\end{proof}


Theorem \ref{thm:generalResultTheorem-Appendix} presents a general approach to prove the recoverability of latent functions (or objects, which are the main focus in this work) from a certain class of independent additive noise. This approach works as long as the function class as well as the noise distribution are mostly comprised of low-frequency components in the Forier domain. For example, the Gaussian noise hurts low-frequency parts of a geometric object far less than its high-frequency details. More specifically, we prove the following corollary for Theorem \ref{thm:generalResultTheorem-Appendix}:
\begin{corollary}[Recoverability from Additive Gaussian Noise $\mathcal{N}\left(\boldsymbol{0},\sigma^2\boldsymbol{I}\right)$]
\label{corl:GaussinNoiseMain-Appendix}
Consider the setting in Theorem \ref{thm:generalResultTheorem-Appendix}, and assume the noise distribution follows $Q\triangleq\mathcal{N}\left(\boldsymbol{0},\sigma^2\boldsymbol{I}\right)$ for $\sigma>0$. Then, as long as for $f,g\in\mathscr{F}$ we have $\left\Vert f-g\right\Vert_2\ge\varepsilon$ for some $\varepsilon\ge0$, we also have
$$
\left\Vert
\left(f-g\right)*Q
\right\Vert_2
\ge
\frac{\varepsilon}{\left(2\pi\right)^K}
\left(
\sup_{\alpha>C}~
\sqrt{1-\zeta\left(\frac{1}{\alpha}\right)}
e^{-K\left(\sigma\alpha\right)^2/2}
\right)
$$
\end{corollary}
\begin{proof}
The Fourier transform of $Q=\mathcal{N}\left(\boldsymbol{0},\sigma^2\boldsymbol{I}\right)$ can be computed as follows:
\begin{equation}
\mathcal{F}\left\{Q\right\}\left(\boldsymbol{\omega}\right)
=
\prod_{i=1}^{K}\mathcal{F}\left\{\mathcal{N}\left(0,\sigma^2\right)\right\}\left(\omega_i\right)
=
e^{-\sigma^2\left\Vert\boldsymbol{\omega}\right\Vert_2^2/2}.
\end{equation}
Also, it can be easily checked that
$$
\inf_{\left\Vert\boldsymbol{\omega}\right\Vert_{\infty}\leq\alpha}
e^{-\sigma^2\left\Vert\boldsymbol{\omega}\right\Vert_2^2/2}=
e^{-\sigma^2/2\left(\alpha^2+\ldots+\alpha^2\right)}
=
e^{-K\left(\alpha\sigma\right)^2/2}.
$$
By subsititution into the end result of Theorem \ref{thm:generalResultTheorem-Appendix}, the claimed bounds can be achieved and the proof is complete.
\end{proof}


In this regard, our main explicit theoretical contribution in this section with respect to simplices has been stated in the following theorem:

\begin{theorem}[Recoverability of Simplices from Additive Noise]
\label{thm:mainNoisySimplexBound-Appendix}
For any two $\left(\bar{\lambda},\underline{\lambda}\right)$-regular simplices $\mathcal{S}_1,\mathcal{S}_2\in\mathbb{S}_K$ with $\bar{\lambda},\underline{\lambda}>0$, given that
$$
\mathcal{D}_{\mathrm{TV}}\left(\mathbb{P}_{\mathcal{S}_1*G_{\sigma}},\mathbb{P}_{\mathcal{S}_2*G_{\sigma}}\right)
=
\frac{1}{2}\int\left\vert
\left(f_{\mathcal{S}_1}-f_{\mathcal{S}_2}\right)*G_{\sigma}
\right\vert
\leq\varepsilon
$$
for some $\varepsilon\ge0$, where $G_{\sigma}$ (for $\sigma\ge0$) represents the density function associated to a Gaussian measure with zero mean and covariance matrix of $\sigma^2\boldsymbol{I}_{K\times K}$ in $\mathbb{R}^K$, i.e., $\mathcal{N}\left(\boldsymbol{0},\sigma^2\boldsymbol{I}\right)$. Then, we have
$$
\mathcal{D}_{\mathrm{TV}}\left(\mathbb{P}_{\mathcal{S}_1},\mathbb{P}_{\mathcal{S}_2}\right)
\leq \varepsilon
e^{\Omega\left(\frac{K}{\mathrm{SNR}^2}\right)}
,
$$
where $\mathrm{SNR}\triangleq \frac{\bar{\lambda}}{K\sigma}$ denotes the effective signal-to-noise ratio, which is the ratio of the standard deviation of a scaled uniform Dirichlet distribution to that of the noise, per dimension.
\end{theorem}
\begin{proof}
Proof is based on properties of the Fourier transforms of simplices $\mathcal{S}_1$ and $\mathcal{S}_2$.
For $K\in\mathbb{N}$, and any integrable function $f:\mathbb{R}^K\rightarrow\mathbb{R}$, the Fourier transform of $f$, denoted by $\mathcal{F}\left\{f\right\}\left(\boldsymbol{\omega}\right):\mathbb{R}^K\rightarrow\mathbb{C}$ is defined as follows:
\begin{equation}
\label{eq:FourierDef}
\mathcal{F}\left\{f\right\}\left(\boldsymbol{\omega}\right)
\triangleq
\int_{\boldsymbol{x}\in\mathbb{R}^K}
f\left(\boldsymbol{x}\right)
e^{-i\boldsymbol{\omega}^T\boldsymbol{x}}
\mathrm{d}\boldsymbol{x}.
\end{equation}
Throughout this proof, for $\mathcal{S}\in\mathbb{S}_K$, let us denote by $\mathcal{F}_{\mathcal{S}}$ the Fourier transform of $f_{\mathcal{S}}$, i.e., the uniform proability density function over $\mathcal{S}$. Also, the inverse Fourier transform which recovers $f_{\mathcal{S}}$ from $\mathcal{F}_{\mathcal{S}}$ can be written as
$$
f_{\mathcal{S}}\left(\boldsymbol{x}\right)
=
\frac{1}{2\pi}\int_{\boldsymbol{\omega}\in\mathbb{R}^K}\mathcal{F}_{\mathcal{S}}\left(\boldsymbol{\omega}\right)
e^{i\boldsymbol{\omega}^T\boldsymbol{x}}
\mathrm{d}\boldsymbol{\omega}.
$$

In this regard, let $\Delta_K$ represent the standard simplex in $\mathbb{R}^K$ which means a $K$-simplex with $\boldsymbol{\theta}_0=\boldsymbol{0}$ and $\boldsymbol{\Theta}=\boldsymbol{I}$. We begin by deriving a number of useful properties for $\Delta_K$ through the following lemmas.


\begin{lemma}
\label{lemma:consistency:recursive}
Let $\mathscr{F}_{\Delta_k}\left(\omega_1,\ldots,\omega_k\right):\mathbb{R}^k\rightarrow\mathbb{C}$ for $k\in[K]$ represent the Fourier transform of $f_{\Delta_k}$ in $\mathbb{R}^k$. Also, we have $\boldsymbol{\omega}_{1:k}\triangleq\left(\omega_1,\ldots,\omega_k\right)$. Then, the following recursive relation holds for $k>1$:
$$
\mathscr{F}_{\Delta_k}\left(
\boldsymbol{\omega}_{1:k}
\right)=
\frac{k}{i\omega_k}\left[
\mathscr{F}_{\Delta_{k-1}}
\left(
\boldsymbol{\omega}_{1:k-1}
\right)
-
e^{-i\omega_k}
\mathscr{F}_{\Delta_{k-1}}
\left(\omega_1-\omega_k,\ldots,\omega_{k-1}-\omega_k\right)
\right].
$$
\end{lemma}
\begin{proof}
For any $\left(x_1,\ldots,x_k\right)\in\Delta_k$,
Let $Z_i\triangleq x_1+\ldots+x_i$ for $i\in[k]$. Then
\begin{align}
\mathscr{F}_{\Delta_k}\left(\boldsymbol{\omega}_{1:k}\right)
&=
\int_{0}^{1}\int_{0}^{1-Z_1}\cdots
\int_{0}^{1-Z_{k-1}}
k!
e^{-i\left(\omega_1 x_1+\ldots+\omega_k x_k\right)}
\mathrm{d}x_1\ldots\mathrm{d}x_k
\\
&=\int_{0}^{1}\cdots\int_{0}^{1-Z_{k-2}}
k!
e^{-i\left(\omega_1 x_1+\ldots+\omega_{k-1} x_{k-1}\right)}
\left(
\int_{0}^{1-Z_{k-1}}
e^{-i\omega_k x_k}
\mathrm{d}x_k
\right)
\mathrm{d}x_1\ldots\mathrm{d}x_{k-1},
\nonumber
\end{align}
where a uniform probability density function over $\Delta_k$ has been assumed to be $f_{\Delta_k}\left(\boldsymbol{x}\right)=(k!)\boldsymbol{1}\left(\boldsymbol{x}\in\Delta_k\right)$ for all $\boldsymbol{x}\in\mathbb{R}^k$, due to the fact the Lebesgue measure (or volume) of the standard simplex is $1/k!$.
Since we have
$$
\int_{0}^{1-Z_{k-1}}
e^{-i\omega_k x_k}
\mathrm{d}x_k
=\frac{1}{i\omega_k}\left(
1-e^{-i\omega_{k}\left(1-Z_{k-1}\right)}
\right),
$$
the Fourier transform can be rewritten as follows:
\begin{align}
\mathscr{F}_{\Delta_k}\left(\boldsymbol{\omega}_{1:k}\right)
=&
\frac{k}{i\omega_k}
\int_{0}^{1}\cdots\int_{0}^{1-Z_{k-2}}
\left(k-1\right)!
e^{-i\left(\omega_1 x_1+\ldots+\omega_{k-1} x_{k-1}\right)}
\mathrm{d}x_1\ldots\mathrm{d}x_{k-1}
\nonumber\\
&-\frac{ke^{-i\omega_k}}{i\omega_k}
\int_{0}^{1}\cdots\int_{0}^{1-Z_{k-2}}
\left(k-1\right)!
\left(\prod_{i=1}^{k-1}e^{-i\left(\omega_i-\omega_k\right)x_i}
\right)
\mathrm{d}x_1\ldots\mathrm{d}x_{k-1}
\nonumber\\
=&
\frac{k}{i\omega_k}\left[
\mathscr{F}_{\Delta_{k-1}}
\left(
\boldsymbol{\omega}_{1:k-1}
\right)
-
e^{-i\omega_k}
\mathscr{F}_{\Delta_{k-1}}
\left(\omega_1-\omega_k,\ldots,\omega_{k-1}-\omega_k\right)
\right],
\end{align}
which completes the proof.
\end{proof}
Lemma \ref{lemma:consistency:recursive} gives us a recursive procedure to produce the Fourier transform, or derive usefull properties for $K$-dimensional standard simplex through tools such as induction. The following lemma establishes a relation between the Fourier transform of the standard simplex and that of an arbitrary simplex in $\mathbb{S}_K$ with a non-zero Lebesgue measure.

\begin{lemma}
\label{lemma:consistency:standard}
For a simplex $S\in\mathbb{S}_K$ with vertices $\boldsymbol{\theta}_0,\ldots,\boldsymbol{\theta}_K$ and a reversible zero-translated vertex matrix $\boldsymbol{\Theta}$, we have
$$
\mathscr{F}_{\mathcal{S}}\left(
\boldsymbol{\omega}_{1:k}
\right)
=
e^{-i\boldsymbol{\omega}_{1:k}^T\boldsymbol{\theta}_0}
\mathscr{F}_{\Delta_K}\left(
\boldsymbol{\Theta}^T
\boldsymbol{\omega}_{1:k}
\right).
$$
\end{lemma}
\begin{proof}
Let $f_{\mathcal{S}}:\mathbb{R}^K\rightarrow\mathbb{R}_{\ge0}$ denote the probability density function associated to $\mathcal{S}$. Then, it can be seen that
$$
f_{\mathcal{S}}\left(\boldsymbol{x}\right)
=
\frac{1}{\mathrm{det}\left(\boldsymbol{\Theta}\right)}
f_{\Delta_K}\left(
\boldsymbol{\Theta}^{-1}
\left(\boldsymbol{x}-\boldsymbol{\theta}_0\right)
\right),
\quad
\forall \boldsymbol{x}\in\mathbb{R}^K.
$$
In this regard, one just needs to write down the definition of Fourier transform for simplex $\mathcal{S}$ and utilizes the change of variables technique as follows:
\begin{align}
\mathscr{F}_{\mathcal{S}}\left(
\boldsymbol{\omega}_{1:k}
\right)
&=
\int_{\mathbb{R}^K}
f_{\mathcal{S}}\left(\boldsymbol{x}_{1:k}\right)
e^{-i\boldsymbol{\omega}_{1:k}^T\boldsymbol{x}_{1:k}}
\mathrm{d}x_1\ldots\mathrm{d}x_k
\\
&=
\frac{1}{\mathrm{det}\left(\boldsymbol{\Theta}\right)}
\int_{\mathbb{R}^K}
f_{\Delta_K}\left(
\boldsymbol{\Theta}^{-1}
\left(\boldsymbol{x}-\boldsymbol{\theta}_0\right)
\right)
e^{-i\boldsymbol{\omega}_{1:k}^T\boldsymbol{x}_{1:k}}
\mathrm{d}x_1\ldots\mathrm{d}x_k
\nonumber\\
&=
e^{-i\boldsymbol{\omega}_{1:k}^T\boldsymbol{\theta}_0}
\int_{\mathbb{R}^K}
f_{\Delta_K}\left(
\boldsymbol{u}_{1:k}
\right)
e^{-i\boldsymbol{\omega}_{1:k}^T
\boldsymbol{\Theta}\boldsymbol{u}_{1:k}}
\mathrm{d}u_1\ldots\mathrm{d}u_k
\nonumber\\
&=
e^{-i\boldsymbol{\omega}_{1:k}^T\boldsymbol{\theta}_0}
\mathscr{F}_{\Delta_K}\left(
\boldsymbol{\Theta}^T
\boldsymbol{\omega}_{1:k}
\right).
\end{align}
Therefore, the proof is complete.
\end{proof}


In the following lemma, we show that the uniform measure over the standard simplex $\Delta_K$ corresponds to a low-frequency probability density function. This would be the first step toward using Corollary \ref{corl:GaussinNoiseMain-Appendix}, in order to prove Theorem \ref{thm:mainNoisySimplexBound-Appendix}.

\begin{lemma}[Low-Pass Property of $\Delta_K$]
There exists a universal constant $C>0$, such that
for $\alpha>C$ and $K\in\mathbb{N}$ the uniform probability density function over $\Delta_K$, i.e., $f_{{\Delta_K}}:\mathbb{R}^K\rightarrow\mathbb{R}_{\ge0}$, is a low-frequency function in the following sense:
\begin{equation}
\frac{\mathrm{Vol}\left(\Delta_K\right)}{\left(2\pi\right)^K}
\int_{\left\Vert\boldsymbol{\omega}\right\Vert_{\infty}\ge\alpha}
\left\vert
\mathscr{F}_{\Delta_K}
\left(\boldsymbol{\omega}\right)
\right\vert^2
\leq
\mathcal{O}\left(\frac{K}{\alpha}\right),
\label{eq:simlexLowPass}
\end{equation}
where by $\left\Vert\boldsymbol{\omega}\right\Vert_{\infty}$ we simply mean $\max_{i\in[K]}\left\vert\omega_i\right\vert$.
\label{lemma:simplexLowPass}
\end{lemma}
\begin{proof}
Proof is based on the direct analysis of the Fourier transform of $f_{\Delta_k}$, which we denoted as $\mathscr{F}_{\Delta_K}\left(\boldsymbol{\omega}\right)$. We take advantage of the fact that the $K$-dimensional unit hypercube $\left[0,1\right]^K$ can be thought as the union of $K!$ properly rotated and translated versions of $\Delta_K$.

Mathematically speaking, assume the ordered tuple of unit axis-aligned vectors $E=\left(\boldsymbol{1}_1,\ldots,\boldsymbol{1}_K\right)$, where $\boldsymbol{1}_i$ for $i\in[K]$ denotes the one-hot vector over the $i$th component. Then,
let $\boldsymbol{V}_1,\ldots,\boldsymbol{V}_{K!}\in\mathbb{R}^{K\times K}$ be the set of orthonormal matrices, where each matrix transforms $E$ into one of its $K!$ possible permutations, i.e.,
$$
\left(\boldsymbol{V}_i\boldsymbol{1}_1,\ldots,\boldsymbol{V}_i\boldsymbol{1}_K\right)
=
\left(\boldsymbol{1}_{p_{i,1}},\ldots,\boldsymbol{1}_{p_{i,K}}\right),\quad
\boldsymbol{p}_i\in\mathrm{Perm}\left([K]\right).
$$
In this regard, we already know that there exist $K!$ corresponding vectors $\boldsymbol{b}_1,\ldots,\boldsymbol{b}_{K!}\in\mathbb{R}^K$, such that the following combined probability density function
\begin{equation}
\label{eq:lemma3:simp2cube}
\frac{1}{K!}\sum_{i=1}^{K!}
f_{\left(\boldsymbol{V}_i\left[\Delta_K+\boldsymbol{b}_i\right]\right)}
=f_{\square_K},
\end{equation}
where $f_{\square_K}$ denotes the pdf over the $K$-dimensional unit hypercube. Also, it should be noted that any two distinct summands in \eqref{eq:lemma3:simp2cube} have an empty overlap. In fact, each $\boldsymbol{V}_i\left[\Delta_K+\boldsymbol{b}_i\right]$ represents a translated and rotatated version of the standard simplex $\Delta_K$.

Next, one can see that
\begin{align}
\mathcal{F}\left\{
\frac{1}{K!}\sum_{i=1}^{K!}
f_{\left(\boldsymbol{V}_i\left[\Delta_K+\boldsymbol{b}_i\right]\right)}
\right\}
&=
\frac{1}{K!}\sum_{i=1}^{K!}
\mathcal{F}\left\{
f_{\left(\boldsymbol{V}_i\left[\Delta_K+\boldsymbol{b}_i\right]\right)}
\right\}
\nonumber\\
&=
\frac{1}{K!}\sum_{i=1}^{K!}
e^{-i\boldsymbol{\omega}^T\boldsymbol{b}_i}
\mathscr{F}_{\Delta_K}\left(\boldsymbol{V}^{-1}_i\boldsymbol{\omega}\right),
\end{align}
where we have used the result of Lemma \ref{lemma:consistency:standard}. Also, note that $\boldsymbol{V}^{-1}_i=\boldsymbol{V}_i$ for all $i$, and $\boldsymbol{V_i}\boldsymbol{\omega}$ represents a permutation of the components of $\boldsymbol{\omega}$. As a result and due to the symmtery of the standard simplex $\Delta_K$ and also the $\ell_{\infty}$-norm w.r.t. the ordering of the edges, we have
\begin{equation}
\int_{\left\Vert
\boldsymbol{\omega}
\right\Vert\leq\alpha
}
\left\vert
\mathcal{F}\left\{
f_{\left(\boldsymbol{V}_i\left[\Delta_K+\boldsymbol{b}_i\right]\right)}
\right\}
\left(\boldsymbol{\omega}\right)
\right\vert^2
=
\int_{\left\Vert
\boldsymbol{\omega}
\right\Vert\leq\alpha
}
\left\vert
\mathscr{F}_{\Delta_K}
\left(\boldsymbol{V}_i\boldsymbol{\omega}\right)
\right\vert^2
=
\int_{\left\Vert
\boldsymbol{\omega}
\right\Vert\leq\alpha
}
\left\vert
\mathscr{F}_{\Delta_K}
\left(\boldsymbol{\omega}\right)
\right\vert^2,\quad\forall i\in[K].
\end{equation}
Now, using \eqref{eq:lemma3:simp2cube} we have:
\begin{align}
\int_{\left\Vert
\boldsymbol{\omega}
\right\Vert\leq\alpha
}
\left\vert
\frac{1}{K!}\sum_{i=1}^{K!}
\mathcal{F}\left\{
f_{\left(\boldsymbol{V}_i\left[\Delta_K+\boldsymbol{b}_i\right]\right)}
\right\}
\right\vert^2
&=
\frac{1}{\left(K!\right)^2}
\sum_{i,j}^{K!}
\int_{\left\Vert
\boldsymbol{\omega}
\right\Vert\leq\alpha
}
\left\vert
\mathscr{F}_{\boldsymbol{V}^{-1}_i\Delta_K}
\bar{\mathscr{F}}_{\boldsymbol{V}^{-1}_j\Delta_K}
\right\vert
\left\vert
e^{-i\boldsymbol{\omega}^T\left(\boldsymbol{b}_i-\boldsymbol{b}_j\right)}
\right\vert
\nonumber\\
&=
\frac{1}{K!}
\int_{\left\Vert
\boldsymbol{\omega}
\right\Vert\leq\alpha
}
\left\vert
\mathscr{F}_{\Delta_K}\left(\boldsymbol{\omega}\right)
\right\vert^2
+\left(1-\frac{1}{K!}\right)\mathcal{O}\left(\frac{1}{\alpha}\right).
\label{eq:lemma3:firstO}
\end{align}
The latter term in the r.h.s. of \eqref{eq:lemma3:firstO} corresponds to the $\left(K!\right)^2-K!$ summands for which we have $i\neq j$. In fact, if $i\neq j$ and based on the fact that Fourier transform preserves inner product\footnote{This is a direct result of the fact that Fourier transform is an ``orthonormal" transformation.}, we have
$$
\int_{\mathbb{R}^K}\left\vert
\mathscr{F}_{\boldsymbol{V}_i\Delta_K}
\bar{\mathscr{F}}_{\boldsymbol{V}_j\Delta_K}
e^{-i\boldsymbol{\omega}^T\left(\boldsymbol{b}_i-\boldsymbol{b}_j\right)}
\right\vert
=
\int_{\mathbb{R}^K}
f_{\boldsymbol{V}_i\left[\Delta_K+\boldsymbol{b}_i\right]}
f_{\boldsymbol{V}_j\left[\Delta_K+\boldsymbol{b}_j\right]}
=0,
$$
which holds since $\Delta_i\triangleq\boldsymbol{V}_i\left[\Delta_K+\boldsymbol{b}_i\right]$ and $\Delta_j\triangleq\boldsymbol{V}_j\left[\Delta_K+\boldsymbol{b}_j\right]$ do not overlap with each other. However, when we add the constraint 
$\left\Vert\boldsymbol{\omega}\right\Vert_{\infty}\leq\alpha$, we effectively take the integral over the hypercube $\left[-\alpha,\alpha\right]^K$ instead of the whole $\mathbb{R}^K$. This procedure is equivalent to the innner product of the two ``smoothed" versions of $f_{\Delta_i}$ and $f_{\Delta_j}$ in the spatial domain. Here, by ``smoothed" we simply mean being convolved with a $K$-dimensional sinc function with parameter $\alpha$, i.e.,
$$
\mathrm{sinc}_{\alpha,K}\left(\boldsymbol{x}\right)
\triangleq
\prod_{i=1}^{K}\frac{\sin\left(\alpha x_i\right)}{\alpha x_i}.
$$
Hence, we have
\begin{align}
\int_{
\left\Vert\boldsymbol{\omega}\right\Vert_{\infty}\leq\alpha
}
\left\vert
\mathscr{F}_{\Delta_i}
\bar{\mathscr{F}}_{\Delta_j}
\right\vert
=
\int_{\mathbb{R}^K}
\left[
f_{\Delta_i}*
\mathrm{sinc}_{\alpha,K}
\right]\left[
f_{\Delta_j}*
\mathrm{sinc}_{\alpha,K}
\right],
\end{align}
which is at most $\mathcal{O}\left(1/\alpha\right)$, since only the leakages that are due to $\mathrm{sinc}_{\alpha,K}$ overlap with each other.

On the other hand, the l.h.s. of \eqref{eq:lemma3:firstO} represents the integration of the Fourier transform of $\square_K$ within $\left[-\alpha,\alpha\right]^K$. Therefore, we have
\begin{align}
\frac{1}{\left(2\pi\right)^KK!}
\int_{\left\Vert
\boldsymbol{\omega}
\right\Vert\leq\alpha
}
\left\vert
\mathscr{F}_{\Delta_K}\left(\boldsymbol{\omega}\right)
\right\vert^2
&\ge
\frac{1}{\left(2\pi\right)^K}
\int_{\left\Vert
\boldsymbol{\omega}
\right\Vert\leq\alpha
}
\left\vert
\mathscr{F}_{\square_K}\left(\boldsymbol{\omega}\right)
\right\vert^2
-
\mathcal{O}\left(\frac{1}{\alpha}\right)
\nonumber\\
&=
\prod_{i=1}^{K}
\frac{1}{2\pi}
\left(\int_{-\alpha}^{\alpha}
\left\vert
\int_{0}^{1}
e^{-i\omega_i x_i}
\right\vert^2
\right)
-
\mathcal{O}\left(\frac{1}{\alpha}\right)
\nonumber\\
&=
\left(1-\frac{2}{\pi\alpha}+o\left(\alpha^{-1}\right)\right)^K
-\mathcal{O}\left(\frac{1}{\alpha}\right)
\nonumber\\
&=1-\mathcal{O}\left(\frac{K}{\alpha}\right),
\end{align}
where we have used Laurent series expansion for the integral of sinc function to derive the bound. Finally, noting the fact that we have $\mathrm{Vol}\left(\Delta_K\right)=1/K!$, and also the Parseval's theorem:
$$
\frac{1}{\left(2\pi\right)^K}
\int_{\mathbb{R}^K}
\left\vert
\mathscr{F}_{\Delta_K}\left(\boldsymbol{\omega}\right)
\right\vert^2
=
\int_{\mathbb{R}^K}
f^2_{\Delta_K}\left(\boldsymbol{x}\right)
=\frac{\mathrm{Vol}\left(\Delta_K\right)}{\mathrm{Vol}^2\left(\Delta_K\right)}=
\mathrm{Vol}^{-1}\left(\Delta_K\right)
$$
completes the proof.
\end{proof}


So far, we have managed to show that the standard simplex $\Delta_K$ is associated to a low-frequency PDF and thus would preserve a minimum level of information even after getting corrupted by additive Gaussian noise. Using Lemma \ref{lemma:consistency:standard}, one can simply extend this notion to any $\left(\bar{\lambda},\underline{\lambda}\right)$-regular simplex. Before that, let us also extend this notion of ``low-frequency" property to the difference function
\begin{equation}
f_{\Delta_K}-f_{\boldsymbol{A}_{\varepsilon}\left(\Delta_K+\boldsymbol{b}_{\varepsilon}\right)},
\label{eq:lemma3:diffFunctionDef}
\end{equation}
where the linear transformation matrix $\boldsymbol{A}_{\varepsilon}$ and translation vector $\boldsymbol{b}_{\varepsilon}$ (for $\varepsilon>0$) are $\varepsilon$-controlled perturbations that alter the standard simplex with the following magnitude
$$
\mathcal{D}_{\mathrm{TV}}\left(
\mathbb{P}_{\Delta_K}
,
\mathbb{P}_{\boldsymbol{A}_{\varepsilon}\left(\Delta_K+\boldsymbol{b}_{\varepsilon}\right)}
\right)
=\varepsilon.
$$
This task is more challenging since the difference PDF function in \eqref{eq:lemma3:diffFunctionDef} can be both positive and negative, and in fact, integrates to zero over $\mathbb{R}^K$. However, the following lemma states that such functions for any $\boldsymbol{A}_{\varepsilon}$ and $\boldsymbol{b}_{\varepsilon}$ are ``band-pass" in nature: even though in the Fourier domain they become exactly zero at the origin, they also become infinitesimally small in terms of magnitude as one asymotitically increases $\left\Vert\boldsymbol{\omega}\right\Vert_{\infty}$.


\begin{lemma}[Low-Frequency Property for Difference of Simplices]
For $\varepsilon\ge0$, Assume the linear transformation matrix $\boldsymbol{A}_{\varepsilon}\in\mathbb{R}^{K\times K}$ where $A_{i,j}=\delta_{ij}+\mathcal{O}\left(\varepsilon\right)$, with $\delta_{ij}$ denoting the kroecker's delta function. Also, assume the translation vector $\boldsymbol{b}_{\varepsilon}\in\mathbb{R}^K$ where $b_i=\mathcal{O}\left(\varepsilon\right)$. Let $\hat{\Delta}_{K,\varepsilon}$ denote the simplex $\boldsymbol{A}_{\varepsilon}\left(\Delta_K+\boldsymbol{b}_K\right)$, and assume the total variation distance between $\mathbb{P}_{\Delta_K}$ and $\mathbb{P}_{\hat{\Delta}_{K,\varepsilon}}$ is $\varepsilon$. Then
$$
\frac{1}{\left(2\pi\right)^K}
\int_{\left\Vert
\boldsymbol{\omega}
\right\Vert_{\infty}\geq\alpha
}
\left\Vert
\mathscr{F}_{\Delta_K} - \mathscr{F}_{\hat{\Delta}_{K,\varepsilon}}
\right\Vert^2
\leq
\mathcal{O}\left(\frac{K\varepsilon}{\alpha}\right),
$$
for sufficiently large $\alpha>0$.
\label{lemma:simpDiffLowPass}
\end{lemma}
\begin{proof}
Based on Lemmas \ref{lemma:simplexLowPass} and \ref{lemma:consistency:standard} and for sufficiently small $\varepsilon$, we already know that both $\mathscr{F}_{\Delta_K}$ and $\mathscr{F}_{\hat{\Delta}_{K,\varepsilon}}$ are low-frequency functions in the following sense: their normalized $\ell_2$ energy outside of the hypercube $\left[-\alpha,\alpha\right]^K$ is at most $\mathcal{O}\left(K/\alpha\right)$. We also know that
$$
\frac{1}{\left(2\pi\right)^K}
\int_{\left\Vert
\boldsymbol{\omega}
\right\Vert_{\infty}\geq\alpha
}
\left\Vert
\mathscr{F}_{\Delta_K} - \mathscr{F}_{\hat{\Delta}_{K,\varepsilon}}
\right\Vert^2
=
\mathcal{O}\left(\varepsilon\right),
$$
due to the assumption of the lemma. However, in order to prove the $\leq\mathcal{O}\left(K\varepsilon/\alpha\right)$ bound, we need to prove another important property: that the energy distribution of the difference function is not significantly controlled by $\varepsilon$, e.g., it does not concentrate on high-frequency regions with $\left\Vert\boldsymbol{\omega}\right\Vert_{\infty}\ge\mathcal{O}\left(1/\varepsilon\right)$.

Before going further into the mathematical details of the above statement, let us present a better view on the Fourier transform of a simplex. We start this procedure with computing $\mathscr{F}_{\Delta_K}$. In this regard, we show the following relation holds for $k\in\mathbb{N}$:
\begin{equation}
\label{eq:pichideh}
\mathcal{F}_{\Delta_k}
=
\frac{k!}{i^k}
\sum_{\ell=1}^{k}\left(1-e^{-i\omega_\ell}\right)
\left[
\omega_\ell
\prod_{j\neq\ell}^{k}\left(\omega_j-\omega_\ell\right)
\right]^{-1}.
\end{equation}
Proof of \eqref{eq:pichideh} is by induction. {\bf {Base}}: First, we show the equality holds for the simple case of $k=1$. {\bf {Step}}: Then, we prove that for $k\ge 2$, if the relation holds for $k-1$, then it also holds for $k$.

\paragraph{Base}
For $k=1$, it can be readily seen that a uniform measure over $\Delta_1$ is a {\it {unit~pulse~function}} over the interval $[0,1]$. Therefore, the Fourier transform of $f_{\Delta_1}$ is already known to be a sinc function, i.e.,
\begin{align}
\mathcal{F}_{\Delta_1}\left(\omega\right)=
\int_{0}^{1}e^{-i\omega x}
\mathrm{d}x
=
\frac{1-e^{-i\omega}}{i\omega},
\end{align}
which matches the formulation of \eqref{eq:pichideh} if one sets $k=1$.

\paragraph{Induction~step}

Assume \eqref{eq:pichideh} holds for $k-1$. By taking advantage of Lemma \ref{lemma:consistency:recursive}, we have:
\begin{align}
\mathcal{F}_{\Delta_k}
=&
\frac{k}{i\omega_k}
\left(
\mathcal{F}_{\Delta_{k-1}}\left(\omega_1,\ldots,\omega_{k-1}\right) - e^{-i\omega_k}
\mathcal{F}_{\Delta_{k-1}}\left(\omega_1-\omega_k,\ldots,\omega_{k-1}-\omega_k\right)
\right)
\\
=&
\frac{\left(k-1\right)!}{i^{k-1}}
\frac{k}{i\omega_k}
\left(
\sum_{\ell=1}^{k-1}\left(1-e^{-i\omega_\ell}\right)
\left[
\omega_\ell
\prod_{j\neq\ell}^{k-1}\left(\omega_j-\omega_\ell\right)
\right]^{-1}
\right.
\nonumber
\\
&
\hspace*{22mm}
-\left.
e^{-i\omega_k}
\sum_{\ell=1}^{k-1}
\left(1-e^{-i\left(\omega_\ell-\omega_k\right)}\right)
\left[
\left(\omega_\ell-\omega_k\right)
\prod_{j\neq\ell}^{k-1}\left(\omega_j-\omega_\ell\right)
\right]^{-1}
\right)
\nonumber
\\
=&
\frac{k!}{i^k}\sum_{\ell=1}^{k-1}
\left(
\left(\omega_k - \omega_\ell\right)
\left(
1-e^{-i\omega_\ell}
\right)
-
\omega_\ell
\left(
e^{-i\omega_\ell}-
e^{-i\omega_k}\right)
\right)
\left[
\omega_\ell
\omega_k
\prod_{j\neq\ell}^{k-1}\left(\omega_j-\omega_\ell\right)
\left(\omega_k-\omega_\ell\right)
\right]^{-1}
\nonumber
\\
=&
\frac{k!}{i^k}
\sum_{\ell=1}^{k-1}
\left(
\omega_k
\left(
1-
e^{-i\omega_\ell}\right)
-
\omega_\ell
\left(
1-e^{-i\omega_k}
\right)
\right)
\left[
\omega_\ell
\omega_k
\prod_{j\neq\ell}^{k-1}\left(\omega_j-\omega_\ell\right)
\left(\omega_k-\omega_\ell\right)
\right]^{-1}
\nonumber\\
=&
\frac{k!}{i^k}
\sum_{\ell=1}^{k-1}
\left(
1-e^{-i\omega_\ell}
\right)
\left[
\omega_\ell
\prod_{j\neq\ell}^{k}\left(\omega_j-\omega_\ell\right)
\right]^{-1}
-
\frac{k!}{i^k}
\left(
\frac{
1-e^{-i\omega_k}
}{\omega_k}
\right)
\sum_{\ell=1}^{k-1}
\prod_{j\neq\ell}^{k}\frac{1}{\omega_j-\omega_\ell}
\nonumber\\
=&
\frac{k!}{i^k}
\sum_{\ell=1}^{k}
\left(
\frac{1-e^{-i\omega_\ell}}{\omega_\ell}
\right)
\prod_{j\neq\ell}^{k}\frac{1}{\omega_j-\omega_\ell},
\nonumber
\end{align}
which again matches with the formulation of \eqref{eq:pichideh} and therefore completes the induction step. Here, for the last equality we have used the following mathematical identity:
$$
\sum_{\ell=1}^{k}\prod_{j\neq\ell}^{k}\frac{1}{\omega_j-\omega_\ell}=0.
$$
\hfill{\bf {End of induction}}
\\[2mm]
‌Based on this explicit formula, the difference function $\mathscr{F}_{\Delta_K} - \mathscr{F}_{\hat{\Delta}_{K,\varepsilon}}$ can be written as follows:
\begin{align}
\mathscr{F}_{\Delta_K} - \mathscr{F}_{\hat{\Delta}_{K,\varepsilon}}
&=
\frac{K!}{i^K}
\sum_{\ell=1}^{K}\left(1-e^{-i\omega_\ell}\right)
\left[
\omega_\ell
\prod_{j\neq\ell}^{K}\left(\omega_j-\omega_\ell\right)
\right]^{-1}
\left(1-r_{\ell}\left(\boldsymbol{\omega}\right)\right),
\end{align}
where
\begin{align}
r_{\ell}\left(\boldsymbol{\omega}\right)
&\triangleq
e^{-i\boldsymbol{\omega}^T\boldsymbol{b}_{\varepsilon}}
\left(
\frac{1-e^{-i\boldsymbol{\omega}^T\boldsymbol{a}_{\ell}}}{1-e^{-i\omega_{\ell}}}
\right)
\frac{\omega_{\ell}}{\boldsymbol{\omega}^T\boldsymbol{a}_{\ell}}
\prod_{j\neq \ell}
\frac{\omega_j-\omega_{\ell}}{\boldsymbol{\omega}^T\left(\boldsymbol{a}_j-\boldsymbol{a}_{\ell}\right)}
\nonumber\\
&=e^{-i\boldsymbol{\omega}^T
\left(\boldsymbol{b}_{\varepsilon}-\boldsymbol{\delta}_{\ell}\right)}
\frac{\sin\left(\omega_{\ell}/2 + \boldsymbol{\delta}^T_{\ell}\boldsymbol{\omega}/2\right)}
{\sin\left(\omega_{\ell}/2\right)}
\frac{1}{1+\left(\boldsymbol{\omega}/\omega_{\ell}\right)^T\boldsymbol{\delta}_{\ell}}
\prod_{j\neq \ell}
\frac{1}{1+
\frac{\left(\boldsymbol{\delta}_j-\boldsymbol{\delta}_{\ell}\right)^T\boldsymbol{\omega}
}{\left(\omega_j-\omega_\ell\right)}}
,
\end{align}
where $\boldsymbol{a}_i$ denotes the $i$th row of $\boldsymbol{A}_{\varepsilon}$. Also, we have assumed $\boldsymbol{a}_i=\boldsymbol{1}_i+\boldsymbol{\delta}_i$ with $\left\Vert\boldsymbol{\delta}_i\right\Vert_{\infty},~i\in[K]$ being upper bounded by $\mathcal{O}\left(\varepsilon\right)$ according to the definition of $\boldsymbol{A}_{\varepsilon}$. It can be readily checked that $r_{\ell}\left(\boldsymbol{\omega}\right)=1+\mathcal{O}\left(\varepsilon\right)$. Also, the multiplicative factors that form $r_{\ell}\left(\boldsymbol{\omega}\right)$ are either not dependent on $\left\Vert\boldsymbol{\omega}\right\Vert_2$ (dependence is only on the direction, and not the magnitude) or they oscillate as a function of magnitude. This implies that $r_{\ell}\left(\boldsymbol{\omega}\right)$ which appears solely due to the difference between the two simplices does not behave differently in regions which are far from the origin or the areas in its vicinity.

Due to the above derivations, the difference function has a similar behaviour to that of $\mathscr{F}_{\Delta_K}\left(\boldsymbol{\omega}\right)$ and thus we have
$$
\frac{1}{\left(2\pi\right)^K}
\int_{\left\Vert
\boldsymbol{\omega}
\right\Vert_{\infty}\geq\alpha
}
\left\Vert
\mathscr{F}_{\Delta_K} - \mathscr{F}_{\hat{\Delta}_{K,\varepsilon}}
\right\Vert^2
\leq
\mathcal{O}\left(\frac{K\varepsilon}{\alpha}\right),
$$
which completes the proof.
\end{proof}
With the help of Corollary \ref{corl:GaussinNoiseMain-Appendix} and Lemmas \ref{lemma:consistency:standard}, \ref{lemma:simplexLowPass} and \ref{lemma:simpDiffLowPass}, we can finally prove the claimed bound. Let simplex $\mathcal{S}_1$ and $\mathcal{S}_2$ to have the vertex matrices $\boldsymbol{\Theta}_1$ and $\boldsymbol{\Theta}_2$, respectively. Also, we have already assumed
$$
\underline{\lambda}\leq\lambda_{\min}\left(\boldsymbol{\Theta}_i\right)\quad,\quad
\lambda_{\max}\left(\boldsymbol{\Theta}_i\right)\leq\bar{\lambda}\quad
\mathrm{for}~i=1,2.
$$
Then, based on Lemma \ref{lemma:consistency:standard} and for a sufficiently large $\alpha>0$, we have
\begin{align}
\int_{\left\Vert\boldsymbol{\omega}\right\Vert_{\infty}\geq\alpha}
\left\vert
\mathscr{F}_{\mathcal{S}_i}\left(\boldsymbol{\omega}\right)
\right\vert^2
&=
\int_{\left\Vert\boldsymbol{\omega}\right\Vert_{\infty}\geq\alpha}
\left\vert
\mathscr{F}_{\Delta_K}\left(\boldsymbol{\Theta}_i^T\boldsymbol{\omega}\right)
\right\vert^2
\nonumber\\
&\leq
\frac{1}{\mathrm{det}\left(\boldsymbol{\Theta}\right)}
\int_{\left\Vert\boldsymbol{\omega}\right\Vert_{\infty}\geq\bar{\lambda}\alpha}
\left\vert
\mathscr{F}_{\Delta_K}\left(\boldsymbol{\omega}\right)
\right\vert^2
\nonumber\\
&\leq
\frac{\left(2\pi\right)^K}{\mathrm{Vol}\left(\mathcal{S}_i\right)}
\mathcal{O}\left(\frac{K}{\bar{\lambda}\alpha}\right),
\end{align}
for $i=1,2$. Therefore, Lemma \ref{lemma:simpDiffLowPass} would consequently imply the following relation for the difference function of the two simplices:
\begin{align}
\frac{1}{\left(2\pi\right)^K}
\int_{\left\Vert\boldsymbol{\omega}\right\Vert_{\infty}\geq\alpha}
\left\vert
\mathscr{F}_{\mathcal{S}_1}-\mathscr{F}_{\mathcal{S}_2}
\right\vert^2
\leq
\mathcal{O}\left(\frac{K}{\bar{\lambda}\alpha}\right)
\int_{\mathbb{R}^K}\left(
f_{\mathcal{S}_1}-f_{\mathcal{S}_2}
\right)^2,
\end{align}
which facilitates the usage of our main general result in Theorem \ref{thm:generalResultTheorem-Appendix}, and more specifically Corollary \ref{corl:GaussinNoiseMain-Appendix} which leads to the following relation:
\begin{align}
\left\Vert
\left(f_{\mathcal{S}_1}-f_{\mathcal{S}_2}\right)*G_{\sigma}
\right\Vert_2
&\ge 
\left\Vert f_{\mathcal{S}_1}-f_{\mathcal{S}_2} \right\Vert_2
\sup_{\alpha}
\left(
\sqrt{1-\mathcal{O}\left(\frac{K}{\bar{\lambda}\alpha}\right)}
e^{-K\left(\sigma\alpha\right)^2}
\right)
\nonumber\\
&\geq
\left\Vert f_{\mathcal{S}_1}-f_{\mathcal{S}_2} \right\Vert_2
e^{-\mathcal{O}\left(\frac{K}{\mathrm{SNR}^2}\right)},
\label{eq:theoremNoisy:mainfinal}
\end{align}
since $\alpha\geq\Omega\left(K/\Bar{\lambda}\right)$ guarantees that the l.h.s. of the bound remains positive, while gives the minimum possible exponent (at least order-wise) to the exponential term. Combining the fact $\ell_2$-norm $\leq$ $\ell_1$-norm together with theorem's assumptions, we have
$$
\left\Vert
\left(f_{\mathcal{S}_1}-f_{\mathcal{S}_2}\right)*G_{\sigma}
\right\Vert_2
\leq
\left\Vert
\left(f_{\mathcal{S}_1}-f_{\mathcal{S}_2}\right)*G_{\sigma}
\right\Vert_1
\leq 2\varepsilon.
$$
Consequently, \eqref{eq:theoremNoisy:mainfinal} implies the following upper-bound on the $\ell_2$-norm of the difference PDF function $f_{\mathcal{S}_1}-f_{\mathcal{S}_2}$ as follows:
$$
\left\Vert f_{\mathcal{S}_1}-f_{\mathcal{S}_2}\right\Vert_2\leq
\varepsilon e^{\mathcal{O}\left(\frac{K}{\mathrm{SNR}^2}\right)},
$$
and completes the proof.
\end{proof}

\section{Proofs for Non-asymptotic Impossibility Results}
\label{sec:app:lowerbound}
\begin{proof}[Proof of Theorem \ref{thm:noisySimplexFirstLowerBound}]
We use the total variation (TV) distance as a metric on the space of simplices, and define the loss accordingly:
\[
  \rho(S_{1},S_{2})
    \;=\;
  \bigl\lVert f_{\mathcal{S}_{1}}-f_{\mathcal{S}_{2}}\bigr\rVert_{\mathrm{TV}},
  \qquad
  \mathcal{L}(S,\hat S)
    \;=\;
  \rho(S,\hat S).
\]

To obtain the lower bound, we invoke the \emph{local Fano method}~\citep{Tsybakov2009}.  
Start with the \emph{standard simplex}
\begin{equation*}
    \mathcal{S}^{\text{std}}
    \;=\;
    \operatorname{conv}\{\mathbf{0},e_{1},\dots,e_{d}\},
\end{equation*}
whose vertices are the all-zero vector and the canonical basis vectors
$e_{i}=(0,\dots,0,1,0,\dots,0)^{\!\top}$. Note that this simplex lies in $d$-dimensional space.

Focus on the vertex $\mathbf{0}$ and draw a hypersphere of
radius~$\zeta/d$ centered at it.
Classical volumetric arguments show that this hypersphere contains at
least $M=2^{d}$ points that can be packed such that every pair is
separated by a Euclidean distance of at least~$\zeta/2d$.
Denote these points by
\(
  v_{01},v_{02},\ldots,v_{0M}
\)
and define the displacement vectors
\begin{equation}
    t_{j}=v_{0j}- \mathbf{0}=v_{0j}\quad(j=1,\dots,M).
\end{equation}
For every $i\in[d+1]$ and $j\in[M]$, define $v_{ij}\;=\;e_{i}+t_{j}$.
Now, we form $M$ simplices, where the $j$-th simplex is
\begin{equation}
    \mathcal{S}_{j}
    \;=\;
    \operatorname{conv}\{v_{0j},v_{1j},\dots,v_{dj}\}.
\end{equation}
Thus, $\mathcal{S}_{j}$ is simply the standard simplex translated by the
vector~$t_{j}$.
Because
\begin{equation}
    \lVert f_{\mathcal{S}_{i}}-f_{\mathcal{S}_{k}}\rVert_{\mathrm{TV}}
     \ge \zeta - \mathcal{O}\left(\zeta^2\right) \geq \frac{\zeta}{2} \quad \text{for } i\neq k,
     \label{eq:tvpacking}
\end{equation}
the family $\{\mathcal{S}_{i}\}_{i=1}^{M}$ forms a
$\zeta$-packing in TV distance. The above inequality holds for $\zeta \leq \frac{1}{2}$.
In addition, the maximum distance between the simplices $\mathcal{S}_{i}$ is $\epsilon$ because their vertices are at most at a distance of $\zeta/d$ from each other. 

Now, we convolve these simplices with Gaussian noise with covariance matrix $\sigma^2\mathbf{I}$. We denote their densities by $\tilde{f}_{\mathcal{S}_i}$. Based on the local Fano method, we have:
\begin{align}
\label{eq:fanoFirstInequality}
    \underset{\mathscr{A}}{\min}\;\underset{\mathcal{S}}{\max}\;\text{TV}\left(\mathcal{S}, \mathscr{A}\left(\mathcal{S}, n, \sigma\right)\right) 
    & \geq 
    \frac{\zeta}{2} \left(1 - \frac{1}{\log{M}}\cdot
    {\underset{i,j}{\max}}~{\text{D}_{\text{KL}}\left(\tilde{f}_{\mathcal{S}_i}^{\otimes n}, \tilde{f}_{\mathcal{S}_j}^{\otimes n}\right)}
    \right)
    \nonumber \\ 
    & = \frac{\zeta}{2}\left(1 - \frac{n}{\log{M}}\cdot
    {\underset{i,j}{\max}}~{\text{D}_{\text{KL}}\left(\tilde{f}_{\mathcal{S}_i}, \tilde{f}_{\mathcal{S}_j}\right)}
    \right),
\end{align}
where $\mathscr{A}$ is an algorithm that receives $n$ samples from a simplex $\mathcal{S}$ contaminated with Gaussian noise of covariance matrix $\sigma^2\mathbf{I}$. We denote the distribution of $n$ i.i.d.\ samples from the noisy simplex by $\tilde{f}_{\mathcal{S}_i}^{\otimes n}$. The minimum is taken over all possible algorithms, and the maximum is over all simplices.

Now we aim to upper bound the KL divergence between noisy simplices. 
Without loss of generality, assume the two simplices whose noisy versions attain the maximum TV distance are $\mathcal{S}_1$ and $\mathcal{S}_2$. We have:
\begin{equation*}
    \tilde{f}_{\mathcal{S}_1}(\boldsymbol{x}) = \frac{1}{\left(2\pi\sigma^2\right)^{K/2}}\int_{\boldsymbol{y} \in \mathcal{S}_1}\frac{1}{\text{Vol}\left(\mathcal{S}_1\right)}\exp\left(-\frac{\|\boldsymbol{x} - \boldsymbol{y}\|^2_2}{\sigma^2}\right)d\boldsymbol{y}.
\end{equation*}
Also, since $\mathcal{S}_1$ and $\mathcal{S}_2$ are shifted versions of each other, let $\boldsymbol{b}$ be the shift vector from $\mathcal{S}_1$ to $\mathcal{S}_2$. Then:
\begin{equation}
    \tilde{f}_{\mathcal{S}_2}(\boldsymbol{x}) = \frac{1}{\left(2\pi\sigma^2\right)^{K/2}}\int_{\boldsymbol{y} \in \mathcal{S}_1}\frac{1}{\text{Vol}\left(\mathcal{S}_1\right)}\exp\left(-\frac{\|\boldsymbol{x} - \boldsymbol{y} - \boldsymbol{b}\|^2_2}{\sigma^2}\right)d\boldsymbol{y}.
\end{equation}
Using Jensen's inequality and the convexity of KL divergence, we obtain:
\begin{align}
    \text{D}_{\text{KL}}\left(\tilde{f}_{\mathcal{S}_1}, \tilde{f}_{\mathcal{S}_2}\right)  
    & = \; \text{D}_{\text{KL}}\left(\frac{1}{\text{Vol}\left(\mathcal{S}_1\right)}\int_{\boldsymbol{y}\in \mathcal{S}_1}f_\sigma\left(\boldsymbol{x}\vert\boldsymbol{y}\right)d\boldsymbol{y}, \frac{1}{\text{Vol}\left(\mathcal{S}_1\right)}\int_{\boldsymbol{y}\in \mathcal{S}_1}f_\sigma\left(\boldsymbol{x}\vert\boldsymbol{y} + \boldsymbol{b}\right)d\boldsymbol{y}\right)
    \nonumber \\
    & \leq \frac{1}{\text{Vol}\left(\mathcal{S}_1\right)}\int_{\boldsymbol{y}\in \mathcal{S}_1}\text{D}_{\text{KL}}\left(f_\sigma\left(\boldsymbol{x}\vert\boldsymbol{y}\right), f_\sigma\left(\boldsymbol{x}\vert\boldsymbol{y} + \boldsymbol{b}\right)\right)d\boldsymbol{y}
    \nonumber \\ 
    & \leq \frac{1}{\text{Vol}\left(\mathcal{S}_1\right)}\int_{\boldsymbol{y}\in \mathcal{S}_1}\frac{\|\boldsymbol{b}\|_2^2}{2\sigma^2}d\boldsymbol{y}
    \nonumber \\
    &=\frac{\|\boldsymbol{b}\|_2^2}{2\sigma^2}
    \nonumber \\ 
    & \leq \frac{\zeta^2}{2K^2\sigma^2},
\end{align}
where $f_\sigma\left(\boldsymbol{x}\vert\boldsymbol{y}\right)$ denotes the density of a normal distribution with mean $\boldsymbol{y}$ and covariance matrix $\sigma^2\mathbf{I}$. Therefore, we can continue inequality~\eqref{eq:fanoFirstInequality} as follows:
\begin{align}
    \underset{\mathscr{A}}{\min}\;\underset{\mathcal{S}}{\max}\;\text{TV}\left(\mathcal{S}, \mathscr{A}\left(\mathcal{S}, n, \sigma\right)\right) 
    & \geq \frac{\zeta}{2} \left(1 - \frac{n}{\log{M}}\cdot
    {\underset{i,j}{\max}}~{\text{D}_{\text{KL}}\left(\tilde{f}_{\mathcal{S}_i}, \tilde{f}_{\mathcal{S}_j}\right)}
    \right)
    \nonumber \\ 
    & \geq \frac{\zeta}{2}\left(1-\frac{n\zeta^2}{2K^3\sigma^2\log 2}\right).
\end{align}

This inequality holds for any $\zeta$, so we set $\zeta = K\sigma\sqrt{\frac{K}{n}}$ and obtain:
\begin{align}
    \underset{\mathscr{A}}{\min}\;\underset{\mathcal{S}}{\max}\;\text{TV}\left(\mathcal{S}, \mathscr{A}\left(\mathcal{S}, n, \sigma\right)\right) 
    & \geq \frac{K\sigma\sqrt{K}}{2\sqrt{n}}.
\end{align}
Therefore, to estimate the original simplex within TV distance less than $\epsilon$ using $n$ noisy samples, we must have:
\[
n \geq \frac{K^3\sigma^2}{\epsilon^2}.
\]
We also require that $\zeta = \frac{\sigma K\sqrt{K}}{\sqrt{n}} \leq \frac{1}{2}$ for \eqref{eq:tvpacking} to hold, which is always satisfied when $\epsilon \leq \frac{1}{2}$. 

To prove the second part of the theorem, consider a simplex that satisfies the properties described in Definition~\ref{def:isoperimetric}, such that $\mathcal{A}_{\max}\left(\mathcal{S}\right) = \bar{\theta} \mathrm{Vol}\left(\mathcal{S}\right)^{\frac{K-1}{K}}$. Denote this simplex by $\mathcal{S}_1$. Let us consider the altitude perpendicular to the largest facet (i.e., the facet with the largest $(K-1)$-dimensional volume), and denote the length of this altitude by $h$. We then displace the simplex in the direction of this altitude by a distance of $\zeta$ to construct a new simplex $\mathcal{S}_2$. For the total variation distance between these two simplices, we have:
\begin{equation}
    \|f_{\mathcal{S}_1} - f_{\mathcal{S}_2}\|_{\text{TV}} \geq \frac{K\zeta}{h} = \frac{K\zeta\mathcal{A}_{\text{max}}}{K\text{Vol}\left(\mathcal{S}_1\right)} = \frac{\zeta\bar{\theta}}{\text{Vol}\left(\mathcal{S}_1\right)^\frac{1}{k}}.
    \label{eq:lecamlowerbound}
\end{equation}
In the above inequality, we use the fact that $\frac{h\mathcal{A}_{\text{max}}}{K} = \text{Vol}\left(\mathcal{S}_1\right)$. Now, from the two-point Le Cam method, we have:
\begin{align}
    \underset{\mathscr{A}}{\min}\;\underset{\mathcal{S}}{\max}\;\text{TV}\left(\mathcal{S}, \mathscr{A}\left(\mathcal{S}, n, \sigma\right)\right) 
    & \geq  \frac{\zeta\bar{\theta}}{\text{Vol}\left(\mathcal{S}_1\right)^\frac{1}{k}}\left(1 - \text{TV}\left(\tilde{f}_{\mathcal{S}_1}^{\otimes n}, \tilde{f}_{\mathcal{S}_2}^{\otimes n}\right)\right)
    \nonumber \\
    & \geq \frac{\zeta\bar{\theta}}{\text{Vol}\left(\mathcal{S}_1\right)^\frac{1}{k}}\left(1 - \sqrt{D_\text{KL}\left(\tilde{f}_{\mathcal{S}_1}^{\otimes n}, \tilde{f}_{\mathcal{S}_2}^{\otimes n}\right)}\right)
    \nonumber \\
    & \geq \frac{\zeta\bar{\theta}}{\text{Vol}\left(\mathcal{S}_1\right)^\frac{1}{k}}\left(1 - \sqrt{\frac{n\zeta^2}{2\sigma^2}}\right)
    \nonumber \\
    & \geq \frac{\sigma\bar{\theta}}{\text{Vol}\left(\mathcal{S}_1\right)^\frac{1}{K}\sqrt{n}}.
\end{align}
Therefore, we have $n \geq \frac{\sigma^2\bar{\theta}^2}{\epsilon^2\text{Vol}\left(\mathcal{S}\right)^{\frac{2}{K}}}$. This completes the proof, and we conclude with:
\begin{equation}
    n \geq \Omega\left(\frac{\sigma^2\bar{\theta}^2}{\epsilon^2\text{Vol}\left(\mathcal{S}\right)^{\frac{2}{K}}} + \frac{K^3\sigma^2}{\epsilon^2}\right).
\end{equation}
Now we can see that if the isoperimetricity properties are not satisfied and $\bar{\theta}$ can grow arbitrarily large, then the lower bound for the minimax error approaches $1$. Another way to interpret this is by using the first inequality in \eqref{eq:lecamlowerbound}, which gives $n \geq \Omega\left(\frac{K^2\sigma^2}{\epsilon^2h^2} + \frac{K^3\sigma^2}{\epsilon^2}\right)$. This implies that as $h/\sigma$ tends to zero, the lower bound increases. In other words, if the simplex does not satisfy the isoperimetricity properties, the sample complexity cannot remain bounded and may diverge as the geometry degenerates.

\end{proof}

\begin{proof}[Proof of Theorem \ref{thm:noislessSimplexFirstLowerBound}]
We prove this theorem using the \textbf{Assouad method}, which provides a lower bound on the estimation loss between the true simplex and its estimate. Specifically, we aim to lower bound the quantity:
\begin{align*}
    \underset{\mathscr{A}}{\min}\;\underset{\mathcal{S} \in \mathbb{S}_K}{\max}\;\|V_{\mathcal{S}} - V_{\mathcal{S}_{\mathscr{A}}}\|_1,
\end{align*}
where $\mathscr{A}$ is an arbitrary estimator (algorithm). To apply the Assouad method, we construct a finite subset $\mathbb{S}_K^c$ of $\mathbb{S}_K$ (the set of all $K$-simplices) that serves as a covering set. We begin with the \textbf{standard simplex}:
\begin{equation*}
    \mathcal{S}^{\text{std}} = \operatorname{conv}\{\mathbf{0}, e_1, \dots, e_K\},
\end{equation*}
and construct perturbations of its vertices. For each $i$, define the candidate set for the $i$-th vertex as:
\begin{equation}
    \boldsymbol{v}_i \in \mathcal{V}_i = \left\{ \left( \zeta(1 - 2b_i^1), \zeta(1 - 2b_i^2), \cdots, 1, \cdots, \zeta(1 - 2b_i^K) \right)^T : b_i^j \in \{0,1\},\; j \neq i \right\},
\end{equation}
which can be compactly expressed as:
\begin{equation}
\label{eq:coveringForAssouad}
    \boldsymbol{v}_i = e_i + \zeta (1 - 2\boldsymbol{b}_i), \quad \boldsymbol{b}_i \in \{0,1\}^K,\; b_i^i = 1.
\end{equation}

Using these vertex constructions, we define the covering set:
\begin{equation}
    \mathbb{S}_K^c = \left\{ \mathcal{S}\left(\mathbf{0}, \boldsymbol{v}_1, \dots, \boldsymbol{v}_K\right) : \boldsymbol{v}_i \in \mathcal{V}_i \right\}.
\end{equation}
Each simplex in $\mathbb{S}_K^c$ can be encoded by a $K^2$-bit binary string. Denote the code of a simplex $\mathcal{S} \in \mathbb{S}_K^c$ as $B^{\mathcal{S}} = (b_1^\mathcal{S}, \dots, b_K^\mathcal{S})$, and let $\mathcal{S}_B$ be the simplex generated from code $B$.

For a general $K$-simplex $\mathcal{S} = \operatorname{conv}(\boldsymbol{x}_0, \dots, \boldsymbol{x}_K)$, define the mapping $\psi: \mathbb{S}_K \to \{0,1\}^{K \times (K-1)}$ as:
\begin{equation}
    \psi(\mathcal{S})_{(i-1)K + j} = \frac{1 - \text{sign}(x_i^j)\cdot \mathbf{1}(i \neq j)}{2}, \quad i, j \in \{1, \dots, K\}.
\end{equation}

We now lower bound the minimax loss in the regime where the error $\epsilon$ is sufficiently small:
\begin{align}
\label{eq:assouadPrimary1}
    \epsilon \triangleq \underset{\mathscr{A}}{\min}\;\underset{\mathcal{S} \in \mathbb{S}_K}{\max}\;\|V_{\mathcal{S}} - V_{\mathcal{S}_{\mathscr{A}}}\|_1 
    &\geq \underset{\mathscr{A}}{\min}\;\underset{\mathcal{S} \in \mathbb{S}_K^c}{\max}\;\|V_{\mathcal{S}} - V_{\mathcal{S}_{\mathscr{A}}}\|_1 \nonumber \\
    &\geq \underset{\mathscr{A}}{\min}\;\underset{\mathcal{S} \in \mathbb{S}_K^c}{\max}\;\sum_{i=1}^K \delta_i,
\end{align}
where $\delta_i$ is the $\ell_1$ distance between the $i$-th vertex $v_i^\mathcal{S}$ of the true simplex and its nearest estimated vertex $v_i^{\mathscr{A}}$ from $\mathcal{S}_{\mathscr{A}}$. Define the vertex-wise loss as:
\begin{equation}
    \ell(\mathcal{S}, \mathscr{A}(\mathcal{S}, n, \sigma)) = \sum_{i=1}^K \|v_i^{\mathcal{S}} - v_i^{\mathscr{A}}\|_1.
\end{equation}
Then this loss satisfies:
\begin{align}
    \ell(\mathcal{S}, \mathscr{A}(\mathcal{S}, n, \sigma)) \geq \zeta \cdot d_H(B_{\mathcal{S}}, \psi(\mathscr{A}(\mathcal{S}, n, \sigma))),
\end{align}
where $d_H$ denotes the Hamming distance. Applying \textbf{Assouad's Lemma}, we obtain:
\begin{align}
\label{eq:AssouadMian1}
    \underset{\mathscr{A}}{\min}\;\underset{\mathcal{S} \in \mathbb{S}_K^c}{\max}\;\ell(\mathcal{S}, \mathscr{A}(\mathcal{S}, n, \sigma)) 
    &\geq \zeta \sum_{i=1}^{K^2} \left( 1 - \max_{B, B': d_H(B, B') = 1} \text{TV}(\mathcal{S}_B^{\otimes n}, \mathcal{S}_{B'}^{\otimes n}) \right) \nonumber \\
    &\geq \zeta \sum_{i=1}^{K^2} \left( 1 - \sqrt{n \cdot \max_{B, B': d_H(B, B') = 1} \text{TV}(\mathcal{S}_B, \mathcal{S}_{B'})} \right).
\end{align}

For two simplices $\mathcal{S}_B$ and $\mathcal{S}_{B'}$ differing in exactly one bit, we have:
\begin{equation}
\label{eq:AssouadAux1}
    \max_{B, B': d_H(B, B') = 1} \text{TV}(\mathcal{S}_B, \mathcal{S}_{B'}) \leq 2\zeta.
\end{equation}

Substituting this into \eqref{eq:AssouadMian1} and choosing $\zeta = \frac{1}{8n}$ yields:
\begin{equation}
\label{eq:AssouadRes1}
    \underset{\mathscr{A}}{\min}\;\underset{\mathcal{S} \in \mathbb{S}_K^c}{\max}\;\ell(\mathcal{S}, \mathscr{A}(\mathcal{S}, n, \sigma)) \geq \frac{K^2}{2n}.
\end{equation}

Combining \eqref{eq:assouadPrimary1} and \eqref{eq:AssouadRes1}, we obtain:
\begin{equation}
    \epsilon = \underset{\mathscr{A}}{\min}\;\underset{\mathcal{S} \in \mathbb{S}_K}{\max}\;\|V_{\mathcal{S}} - V_{\mathcal{S}_{\mathscr{A}}}\|_1 \geq \frac{K}{\sqrt{n}} \quad \Longrightarrow \quad n \geq \frac{K^2}{\epsilon}.
\end{equation}

This completes the proof. One important point worth mentioning is that for any algorithm $\mathscr{A}$ whose estimated simplex $\mathcal{S}_{\mathscr{A}}$ lies entirely within the true simplex $\mathcal{S}$, we can express the following:
\begin{align}
    \underset{\mathscr{A}}{\min}\;\underset{\mathcal{S} \in \mathbb{S}_K}{\max}\text{TV}\left(\mathcal{S}, \mathcal{S}_{\mathscr{A}}\right) 
    & \geq \underset{\mathscr{A}}{\min}\;\underset{\mathcal{S} \in \mathbb{S}^c_K}{\max}\text{TV}\left(\mathcal{S}, \mathcal{S}_{\mathscr{A}}\right)
    \nonumber \\
    & \geq\underset{\mathscr{A}}{\min}\;\underset{\mathcal{S} \in \mathbb{S}_K^c}{\max}\;\|V_{\mathcal{S}} - V_{\mathcal{S}_{\mathscr{A}}}\|_1, 
    \label{eq:tvNorm1Relation}
\end{align}
And from this, we can proceed exactly as in \eqref{eq:assouadPrimary1} to obtain the same result for the TV distance. To prove the inequality above, let us consider a single facet $\mathcal{F}$ of the true simplex and its corresponding facet $\widehat{\mathcal{F}}$ in the estimated simplex. Without loss of generality, assume that $\mathcal{F}$ is the facet opposite to the vertex $\boldsymbol{v}_1$.
We rotate our coordinate system such that $\widehat{\mathcal{F}}$ lies within the subspace spanned by the last $K-1$ dimensions of the new coordinate system; denote this subspace by $W$. In this coordinate system, we define $D_1$ as the region of the true simplex that projects outside $\widehat{\mathcal{F}}$ along the direction orthogonal to $W$.
To find a lower bound on the TV distance, we evaluate the probability measure of the region $D_1$, since it is not contained in the estimated simplex. Summing over analogous regions $D_i$ for each facet gives a lower bound on the TV distance. Importantly, the intersection of these regions is of second-order in the distance between the corresponding vertices, and hence negligible at first order.
Now, let $\delta^1_2, \dots, \delta^1_{K}$ denote the distances from each of the $K-1$ vertices of $\widehat{\mathcal{F}}$ to the corresponding facet $\mathcal{F}$ of the true simplex. Then, the probability measure of the region $D_1$ is:
\begin{align}
    \text{TV}\left(\mathcal{S}, \mathcal{S}_{\mathscr{A}}\right) 
    &\geq \sum_i\int_{D_i} f_{\mathcal{S}}(\boldsymbol{x})\,d\boldsymbol{x} 
    \nonumber \\ 
    & \geq \sum_{j=1}^K\mathbb{P}\left(\sum_{i=2}^{K}{\delta_i^j\phi_i} \geq \phi_1\right)
    \nonumber \\ 
    & \geq \Omega\left(\sum_{i,j} \delta_i^j\right) 
    \nonumber \\ 
    &= \Omega\left(\|V_{\mathcal{S}} - V_{\mathcal{S}_{\mathscr{A}}}\|_1\right)
    \end{align}
where $(\phi_0, \cdots, \phi_K)$ is drawn from the uniform Dirichlet distribution. 
This concludes the proof of \eqref{eq:tvNorm1Relation}.
\end{proof}

\begin{proof}[proof of Theorem \ref{thm:noislessSimplexSecondLowerBound}]
To prove this theorem, we use the Assouad method. In this method, we aim to provide a lower bound for the loss function between our estimate and the true simplex. Mathematically, we want to lower bound:
\begin{align*}
    \underset{\mathscr{A}}{\min}\;\underset{\mathcal{S} \in \mathbb{S}_K}{\max}\;\mathbb{E}\left[\text{TV}\left(\mathcal{S}, \mathscr{A}\left(\mathcal{S}, n, \sigma\right)\right)\right],
\end{align*}
where $\mathscr{A}$ is the algorithm and the expectation is taken over the samples. To achieve this, we first construct a covering for the set of all possible $K$-simplices $\mathbb{S}_K$, which we denote by $\mathbb{S}_K^c$. To build this covering, we consider the standard simplex
\begin{equation*}
    \mathcal{S}^{\text{std}} = \operatorname{conv}\{\mathbf{0}, e_1, \dots, e_K\}.
\end{equation*}
Next, we define the candidate positions for the $i$-th vertex of the simplices in $\mathbb{S}_K^c$ as
\begin{equation}
    \boldsymbol{v}_i \in \mathcal{V}_i = \left\{\left(\zeta b_1, \zeta b_2, \cdots, 1, \cdots, \zeta b_K\right)^T : b_j \in \{0,1\},\; j \neq i \right\},
\end{equation}
which can equivalently be written as:
\begin{equation}
\label{eq:coveringForAssouadTV}
    \boldsymbol{v}_i = e_i - \zeta \boldsymbol{b}, \quad \boldsymbol{b} \in \left\{0,1\right\}^K,\; b_i = 1.
\end{equation}
We now construct the covering set $\mathbb{S}_K^c$ as:
\begin{equation}
    \mathbb{S}_K^c = \left\{\mathcal{S}\left(\mathbf{0}, \boldsymbol{v}_1, \boldsymbol{v}_2, \cdots, \boldsymbol{v}_K \;\big\vert\; \boldsymbol{v}_i \in \mathcal{V}_i\right)\right\}.
\end{equation}

From this construction, it is clear that the simplices in $\mathbb{S}_K^c$ can be encoded using $K$ bits. Let us denote the code for a simplex $\mathcal{S} \in \mathbb{S}_K^c$ by $B^{\mathcal{S}} = \left(b_1^\mathcal{S}, \cdots, b_K^\mathcal{S}\right)$, and the simplex associated with a code $B$ by $\mathcal{S}_B$. We denote the set of all such encodings by $\mathcal{B}$. Now, we define a set of regions in the space as follows:
\begin{align}
    R_{i} &= \left\{\boldsymbol{x} : \boldsymbol{x} \in \mathcal{S}_{B_{i}},\; \boldsymbol{x} \notin \mathcal{S}^{\text{std}}\right\}, \nonumber \\
    B_{i} &\in \{0,1\}^{K}, \quad (B_{i})_j =
\begin{cases}
1 & \text{if } j = i, \\
0 & \text{otherwise}.
\end{cases}
\end{align}
Suppose we denote the density associated with the simplex in $\mathbb{S}_K^c$ encoded by $B$ as $f_B$. Then we proceed as follows:
\begin{align}
    \underset{\mathcal{S} \in \mathbb{S}_K}{\max}\;\mathbb{E}&\left[\text{TV}\left(\mathcal{S}, \mathscr{A}\left(\mathcal{S}, n, \sigma\right)\right)\right] \geq 
    \underset{\mathcal{S} \in \mathbb{S}_K^c}{\max}\;\left[\text{TV}\left(\mathcal{S}, \mathscr{A}\left(\mathcal{S}, n, \sigma\right)\right)\right]
    \nonumber \\
    & \geq \frac{1}{2^{K}}\sum_{B \in \mathcal{B}}{\mathbb{E}\left[\int \vert f_{\mathscr{A}}\left(\boldsymbol{x}\right)  - f_B\left(\boldsymbol{x}\right)\vert d\boldsymbol{x} \right]}
    \nonumber \\
    & = \frac{1}{2^{K}}\sum_{B \in \mathcal{B}}{\int\left(\int \vert f_{\mathscr{A}}\left(\boldsymbol{x}\right)  - f_B\left(\boldsymbol{x}\right)\vert d\boldsymbol{x}\right)\prod_{k=1}^nf_B\left(\boldsymbol{x}_k\right)\prod_{k=1}^nd\boldsymbol{x}_k}
    \nonumber \\
    & \geq \frac{1}{2^{K}}\sum_{B \in \mathcal{B}}{\int\left(\sum_{i}\int_{R_{i}} \vert f_{\mathscr{A}}\left(\boldsymbol{x}\right)  - f_B\left(\boldsymbol{x}\right)\vert d\boldsymbol{x}\right)\prod_{k=1}^nf_B\left(\boldsymbol{x}_k\right)\prod_{k=1}^nd\boldsymbol{x}_k} - \mathcal{O}\left(K^4\zeta^2\right)
    \label{eq:assouadbreakIntoregions0} \\
    & \geq \frac{1}{2^{K}}\sum_{B \in \mathcal{B}}{\int\left(\sum_{i}\int_{R_{i}} \vert f_{\mathscr{A}}\left(\boldsymbol{x}\right)  - f_B\left(\boldsymbol{x}\right)\vert d\boldsymbol{x}\right)\prod_{k=1}^nf_B\left(\boldsymbol{x}_k\right)\prod_{k=1}^nd\boldsymbol{x}_k} - \mathcal{O}\left(K^4\zeta^2\right)
    \nonumber \\
    & \geq \frac{1}{2^{K + 1}}\sum_{B \in \mathcal{B}}\left(\sum_{i}\int\left(\int_{R_{i}} \vert f_{\mathscr{A}}\left(\boldsymbol{x}\right)  - f_{B_{ij0}}\left(\boldsymbol{x}\right)\vert d\boldsymbol{x}\right)\prod_{k=1}^nf_{B_{i0}}\left(\boldsymbol{x}_k\right)\prod_{k=1}^nd\boldsymbol{x}_k + \right.
    \nonumber \\
     & \quad\quad\quad\quad\quad  \left. + \int\sum_{i}\left(\int_{R_{i}} \vert f_{\mathscr{A}}\left(\boldsymbol{x}\right)  - f_{B_{i1}}\left(\boldsymbol{x}\right)\vert d\boldsymbol{x}\right)\prod_{k=1}^nf_{B_{i1}}\left(\boldsymbol{x}_k\right)\prod_{k=1}^nd\boldsymbol{x}_k  \right)- 
     \nonumber \\ 
     & \quad\quad\quad\quad\quad \mathcal{O}\left(K^4\zeta^2\right)
     \label{eq:assouadbreakIntoregions1}
     \\
     & \geq \frac{1}{2^{K + 1}}\sum_{B \in \mathcal{B}}\int\sum_{i}\left(\int_{R_{i}} \vert f_{\mathscr{A}}\left(\boldsymbol{x}\right)  - f_{B_{i0}}\left(\boldsymbol{x}\right)\vert d\boldsymbol{x} \right. \nonumber\\
     & \quad\quad\quad\quad\quad\quad\quad \left. \int_{R_{i}} \vert f_{\mathscr{A}}\left(\boldsymbol{x}\right)  - f_{B_{i1}}\left(\boldsymbol{x}\right)\vert d\boldsymbol{x}\right) \nonumber \\ 
     & \quad\quad\quad\quad\quad\quad\quad \min\left\{\prod_{k=1}^nf_{B_{i0}}\left(\boldsymbol{x}_k\right), \prod_{k=1}^nf_{B_{i1}}\left(\boldsymbol{x}_k\right)\right\}\prod_{k=1}^nd\boldsymbol{x}_k  - \mathcal{O}\left(K^4\zeta^2\right) 
     \label{eq:assouadbreakIntoregions2}
     \\
     & \geq \frac{1}{2^{K + 1}}\sum_{B \in \mathcal{B}}\int\sum_{i}\int_{R_{i}} \vert f_{B_{i1}}\left(\boldsymbol{x}\right)  - f_{B_{i0}}\left(\boldsymbol{x}\right)\vert d\boldsymbol{x} \nonumber \\
     & \quad\quad\quad\quad\quad\quad\quad 
     \min\left\{\prod_{k=1}^nf_{B_{i0}}\left(\boldsymbol{x}_k\right), \prod_{k=1}^nf_{B_{i1}}\left(\boldsymbol{x}_k\right)\right\}\prod_{k=1}^nd\boldsymbol{x}_k  - \mathcal{O}\left(K^4\zeta^2\right) 
     \label{eq:assouadbreakIntoregions3}
     \\
     & \geq \frac{\zeta K^2}{2^{K + 2}}\sum_{B \in \mathcal{B}}\int\min_{i}\min\left\{\prod_{k=1}^nf_{B_{i0}}\left(\boldsymbol{x}_k\right), \prod_{k=1}^nf_{B_{ij1}}\left(\boldsymbol{x}_k\right)\right\}\prod_{k=1}^nd\boldsymbol{x}_k  - \mathcal{O}\left(K^4\zeta^2\right) 
     \label{eq:assouadbreakIntoregions4}
     \\
     & \geq \frac{\zeta K^2}{2}\min_{B,i}\int\min\left\{\prod_{k=1}^nf_{B_{i0}}\left(\boldsymbol{x}_k\right), \prod_{k=1}^nf_{B_{i1}}\left(\boldsymbol{x}_k\right)\right\}\prod_{k=1}^nd\boldsymbol{x}_k  - \mathcal{O}\left(K^4\zeta^2\right) 
     \label{eq:assouadbreakIntoregions5}
     \\
     & \geq \frac{\zeta K^2}{8}\min_{B,i}\left(\int\sqrt{\prod_{k=1}^nf_{B_{i0}}\left(\boldsymbol{x}_k\right)f_{B_{i1}}\left(\boldsymbol{x}_k\right)}\prod_{k=1}^nd\boldsymbol{x}_k\right)^2  - \mathcal{O}\left(K^4\zeta^2\right) 
     \label{eq:assouadbreakIntoregions6}
     \\
     & \geq \frac{\zeta K^2}{8}\left(\min_{B,i}\int\sqrt{f_{B_{i0}}\left(\boldsymbol{x}\right)f_{B_{i1}}\left(\boldsymbol{x}\right)}d\boldsymbol{x}\right)^{2n}  - \mathcal{O}\left(K^4\zeta^2\right) 
     \label{eq:assouadbreakIntoregions7}
     \\
     & \geq \frac{\zeta K^2}{8}\min_{B,i}\left(1 - \text{TV}\left(f_{B_{i0}}, f_{B_{i1}}\right)\right)^{2n}  - \mathcal{O}\left(K^4\zeta^2\right) 
     \label{eq:assouadbreakIntoregions8}
     \\
     & \geq \frac{\zeta K^2}{8}\min_{B,i}\left(1 - 2n\text{TV}\left(f_{B_{i0}}, f_{B_{i1}}\right)\right)  - \mathcal{O}\left(K^4\zeta^2\right) 
     \\
     & \geq \frac{\zeta K^2}{8}\left(1 - \sqrt{2nK\zeta}\right)  - \mathcal{O}\left(K^4\zeta^2\right)
     \nonumber \\ 
     &\geq \Omega\left(\frac{K}{n}\right)  - \mathcal{O}\left(\frac{K^2}{n^2}\right) = \Omega\left(\frac{K}{n}\right)
     \label{eq:assouadbreakIntoregions9}
\end{align}
About the above inequalities we should mention the followings:
\begin{enumerate}
    \item For the regions $R_{ij}$, from the way they are constructed, we have $\text{Vol}\left(R_{i} \cap R_{j}\right) \leq \mathcal{O}\left(K^2\zeta^2\right) \text{Vol}(\mathcal{S}^{\text{std}})$. Due to this fact, we subtract the $\mathcal{O}\left(K^4 \zeta^2\right)$ term in \eqref{eq:assouadbreakIntoregions0}, since each region $R_{i}$ intersects with at most $K$ other regions, and we are summing over $K$ regions in total.
    
    \item By $B_{i0}$ we mean a $0$–$1$ vector equal to $B$ except its $i$-th bit is $0$, and $B_{i1}$ is equal to $B$ with its $i$-th bit set to $1$.
    
    \item In inequality \eqref{eq:assouadbreakIntoregions1}, we use the fact that both $B_{i0}$ and $B_{i1}$ are in $\mathcal{B}$. Suppose we write the summands for these two options as $S_0$ and $S_1$, respectively. We then use $\frac{S_0 + S_1}{2}$ in place of the individual terms.
    
    \item In inequality \eqref{eq:assouadbreakIntoregions3}, we use the fact that the algorithm will output the same simplex given the same samples $\left\{\boldsymbol{x}_i\right\}_{i=1}^n$, along with the inequality $|a| + |b| \geq |a - b|$.
    
    \item In inequality \eqref{eq:assouadbreakIntoregions3}, we use the fact that $\int_{R_{i}} \vert f_{B_{i1}}\left(\boldsymbol{x}\right)  - f_{B_{i0}}\left(\boldsymbol{x}\right)\vert d\boldsymbol{x} \geq K\zeta/2$.

    \item In inequality \eqref{eq:assouadbreakIntoregions6}, we use Le Cam's inequality: $2 \int \min\left\{f, g\right\} \geq \left(\int \sqrt{f g}\right)^2$.
    
    \item In inequality \eqref{eq:assouadbreakIntoregions8}, we use the inequality $\text{BC} \geq 1 - \text{TV}$, where $\text{BC}$ is the Bhattacharyya coefficient.
    
    \item In inequality \eqref{eq:assouadbreakIntoregions9}, we set $\zeta = \frac{1}{n}$.
    
    \item The equality in \eqref{eq:assouadbreakIntoregions9} follows from the fact that $\frac{K^2}{n^2} \leq \frac{K}{n}$.
\end{enumerate}
This concludes the proof. 
\end{proof}

\end{document}